\documentclass[11pt]{article}
\usepackage[left=1in, right=1in, top=1in]{geometry}
\usepackage{xargs}
\usepackage[numbers]{natbib}
\usepackage{fancyhdr}
\usepackage{setspace}
\usepackage{lastpage}
\usepackage{upgreek}
\usepackage[american]{babel}
\usepackage[utf8]{inputenc}
\usepackage[T1]{fontenc}
\usepackage{amsmath,mathtools,amsthm,amsfonts,amssymb}
\usepackage{dsfont,bbm}
\usepackage{graphicx}
\usepackage{subcaption}
\usepackage{booktabs}
\usepackage{nicefrac}
\usepackage{microtype}
\usepackage{xcolor}
\usepackage{longtable}
\usepackage{comment}
\usepackage{enumitem}
\usepackage{multirow}
\usepackage[disable]{todonotes}
\usepackage{bm}
\usepackage{mathrsfs}

\setcitestyle{number}

\usepackage[colorlinks=true,breaklinks=true,bookmarks=true,urlcolor=blue,
  citecolor=blue,linkcolor=blue,bookmarksopen=false,draft=false]{hyperref}
\usepackage{aliascnt}
\usepackage[nameinlink,capitalize,noabbrev]{cleveref}

\newcommand{\abs}[1]{\left\vert #1 \right\vert}

\def\funcAw{\mathbf{A}}
\def\funcbw{\mathbf{b}}

\newcommandx{\zmfuncA}[2][1=]{\tilde{\funcAw}^{#1}(#2)}
\newcommandx{\zmfuncAw}[1][1=]{\tilde{\funcAw}_{#1}}
\newcommandx{\zmfuncb}[2][1=]{\tilde{\funcbw}^{#1}(#2)}

\newcommand{\supnorm}[1]{\norm{ #1 }[\infty]}
\newcommandx{\norm}[2][2=]{\Vert#1 \Vert_{{#2}}}

\newcommandx{\vartconstwas}[1][1=V]{c_{#1}}
\newcommandx{\ratewas}[1][1=]{\varrho_{#1}}

\newcommandx{\excess}[1][1=]{\mathcal{E}_{#1}}
\newcommandx{\approxerror}[1][1=\tau]{\mathcal{A}_{#1}}

\newcommand{\quant}[2]{\operatorname{Q}(#1;#2)}
\newcommand{\quantq}{\operatorname{Q}}

\newcommandx{\normop}[2][2=]{\Vert{#1}\Vert_{{#2}}}

\newcommand{\PP}{\mathbb{P}}

\newcommand{\RR}{\mathbb{R}}

\def\rmd{\mathrm{d}}

\def\rset{\mathbb{R}}

\theoremstyle{definition}

\theoremstyle{plain}
\newcommand{\continuation}{??}

\newcommand{\coint}[1]{\left[#1\right)}
\newcommand{\ocint}[1]{\left(#1\right]}
\newcommand{\ooint}[1]{\left(#1\right)}
\newcommand{\ccint}[1]{\left[#1\right]}

\newcommand{\indi}[1]{\mathbbm{1}_{#1}}
\newcommand{\indiacc}[1]{\mathbbm{1}_{\{#1\}}}

\newcommandx{\CPE}[3][1=]{\mathsf{E}_{#1}^{#3}\bigl[  #2 \bigr]}

\def\nsets{\mathbb{N}^{\star}}

\def\N{\mathbb{N}}

\DeclareMathAlphabet\mathbfcal{OMS}{cmsy}{b}{n} 
\usepackage{aliascnt}
\usepackage{cleveref}

\theoremstyle{plain}
\newtheorem{theorem}{Theorem}

\newaliascnt{lemma}{theorem}
\newtheorem{lemma}[lemma]{Lemma}
\aliascntresetthe{lemma}

\newaliascnt{proposition}{theorem}
\newtheorem{proposition}[proposition]{Proposition}
\aliascntresetthe{proposition}

\newaliascnt{corollary}{theorem}

\aliascntresetthe{corollary}

\theoremstyle{definition}
\newaliascnt{definition}{theorem}
\newtheorem{definition}[definition]{Definition}
\aliascntresetthe{definition}

\newaliascnt{remark}{theorem}
\newtheorem{remark}[remark]{Remark}
\aliascntresetthe{remark}
\theoremstyle{plain}

\crefname{theorem}{Theorem}{Theorems}
\crefname{lemma}{Lemma}{Lemmas}
\crefname{proposition}{Proposition}{Propositions}
\crefname{corollary}{Corollary}{Corollaries}
\crefname{definition}{Definition}{Definitions}
\crefname{remark}{Remark}{Remarks}

\def\vartconstwas{c_{\MKK}}

\def\ratewas{\varrho}

\def\MKK{{\rm K}}

\def\Xset{\XC}
\def\Yset{\mathcal{Y}}

\DeclareMathOperator*{\essinf}{ess\,inf}
\DeclareMathOperator*{\esssup}{ess\,sup}

\newcommand{\E}{\mathbb{E}}
\newcommand{\R}{\mathbb{R}}

\newcommandx{\low}[1][1=x]{\mu_{\operatorname{low}}(#1)}
\def\up{\mu_{\operatorname{up}}}
\newcommandx\lows[3][1=S, 3=\boldsymbol{\alpha}]{\kappa^{(#1)}_{#3,#2,h}}
\def\ups{\sigma_{\operatorname{up}}}

\newcommandx{\DC}[1][1=]{P_{#1}}
\newcommandx{\EC}[2][1=n,2=\delta]{\mathcal{E}_{#1,#2}}

\def\OC{\mathcal{O}}

\def\XC{\mathcal{X}}

\def\iid{i.i.d.}

\def\eps{\varepsilon}

\newcommand{\balpha}{{\boldsymbol{\alpha}}}

\newcommand{\ahi}{\alpha_{\operatorname{hi}}}
\newcommand{\alo}{\alpha_{\operatorname{lo}}}
\newcommand{\pmin}{p_{\operatorname{min}}}

\newcommandx{\lipdens}[1][1=X]{\mathrm{L}_{#1}}

\newcommand{\KDE}{\mathcal{K}}

\newcommandx{\CCstar}[1][1=\alpha]{\mathcal{C}_{#1}^{\star}}
\newcommandx{\CChat}[1][1=\alpha]{\hat{\mathcal{C}}_{#1}}
\newcommandx{\kCChat}[2][1=\alpha,2=h]{\hat{\mathcal{C}}_{#1}^{(#2)}}

\theoremstyle{definition}
\newtheorem{assum}{\textbf{A}\hspace{-2pt}}
\crefname{assum}{\textbf{A}\hspace{-2pt}}{\textbf{A}\hspace{-2pt}}
\Crefname{assum}{\textbf{A}\hspace{-2pt}}{\textbf{A}\hspace{-2pt}}
\newcommand{\Aref}[2]{\Cref{#1}\ensuremath{(#2)}}

\newtheorem{assumP}{\textbf{P}\hspace{-2pt}}
\crefname{assumP}{\textbf{P}\hspace{-2pt}}{\textbf{P}\hspace{-2pt}}
\Crefname{assumP}{\textbf{P}\hspace{-2pt}}{\textbf{P}\hspace{-2pt}}

\crefname{assumB}{\textbf{B}\hspace{-2pt}}{\textbf{B}\hspace{-2pt}}
\Crefname{assumB}{\textbf{B}\hspace{-2pt}}{\textbf{B}\hspace{-2pt}}

\Crefname{assumL}{\textbf{L}\hspace{-2pt}}{\textbf{L}\hspace{-2pt}}
\crefname{assumL}{\textbf{L}}{\textbf{L}}

\newtheorem{assumK}{\textbf{K}\hspace{-2pt}}
\Crefname{assumK}{\textbf{K}\hspace{-2pt}}{\textbf{K}\hspace{-2pt}}
\crefname{assumK}{\textbf{K}}{\textbf{K}}
\theoremstyle{plain}

\newcommandx{\fhat}[2][1=n]{\hat{f}_{#1,#2}}

\newcommandx{\bestfinclass}[1][1=\tau]{\tilde{f}_{#1}}
\newcommandx{\bestparam}[1][1=\tau]{\bar{\param}_{#1}}

\newcommand{\Sstar}{S^\star}
\newcommand{\Shat}{\widehat{S}}
\newcommandx{\Risk}[2][1=\tau]{R_{#1}(#2)}            
\newcommandx{\EmpRisk}[2][1=\tau]{R_{n,\tau}(#2)}      

\def\param{\theta}

\newcommandx{\hparam}[2][1=n,2=\tau]{{\widehat{\param}_{#1,#2}}} 
\newcommandx{\tparam}[1][1=\tau]{{\widetilde{\param}_{#1}}} 

\newcommandx{\Dtrain}[1][1=n]{\mathcal{D}_{#1}^{\operatorname{tr}}} 
\newcommandx{\Dcal}[1][1=m]{\mathcal{D}_{#1}^{\operatorname{cal}}} 
\newcommandx{\Itrain}[1][1=n]{\mathcal{I}_{#1}^{\operatorname{tr}}}
\newcommandx{\Ical}[1][1=m]{\mathcal{I}_{#1}^{\operatorname{cal}}}                 
\newcommand{\ntrain}{m}        
\newcommandx{\gz}[2][1=h]{\gamma_{#2,#1}}

\newcommandx{\locweight}[2][1=x]{w_{#1,h}(#2)}        
\newcommandx{\locnorm}[1][1=x]{Z_{#1,h}}                 
\newcommandx{\locmeasure}[1][1={x,h}]{\Lambda_{#1}}        

\newcommandx{\locjoint}[1][1={x,h}]{\bar{\pi}_{#1}}

\newcommandx{\coverrloc}[2][1={m,h}]{\mathcal{E}_{#1}(#2)}  

\newcommandx{\Deltaf}[1][1=n]{\Delta_{#1,f}}                  
\newcommand{\DeltaS}{\Delta_{\Sstar}}                 

\newcommandx{\Lone}[2][2=P_X]{\|#1\|_{L^1(#2)}}
\newcommandx{\Ltwo}[2][2=P_X]{\|#1\|_{L^2(#2)}} 

\newcommandx{\Lp}[2][1=\Lambda]{\|#2\|_{L^p(#1)}}
\newcommandx{\Lq}[2][1=\Lambda]{\|#2\|_{L^q(#1)}}

\newcommand{\Loneloc}[2]{\|#1\|_{L^1(\locmeasure[#2])}}

\def\CQR{\texttt{CQR}}

\newcommandx{\SetOmega}[1][1={m,x,h,\delta}]{\Omega^{\mathrm{cal}}_{#1}}

\newcommand{\CondCDF}[1][S]{F^{(#1\mid X)}} 
\newcommandx{\oraclequant}[3][1=S,3=1-\alpha]{q^{(#1)}_{#3}(#2)}        
\newcommandx{\locoraclequant}[3][1=S,3=1-\alpha]{q^{(#1)}_{#3,h}(#2)}      
 
\newcommandx{\nlocweightnorm}[3][1=x]{\widetilde{w}^{#1}_{#2,h,#3}}     
\newcommandx{\nEmpCDFloc}[2][1=x]{\widehat{F}^{#1,S}_{n,#2,h}}        
\newcommandx{\nBarCDFloc}[2][1=S]{\overline{F}^{(#1)}_{n,#2,h}}  
\newcommandx{\CDFloc}[2][1=S]{{F}^{(#1)}_{#2,h}}
\newcommandx{\PDFloc}[2][1=S]{{p}^{(#1)}_{#2,h}}
\newcommandx{\I}[3][1=x,2=\boldsymbol{\alpha},3=S]{{I}_{#1,h}^{(#3)}(#2)}
\newcommandx{\rS}[2][1={1-\alpha}]{r^{(S)}_{#1}(#2)}
\def\alphalow{\alpha^{-}}                  
\def\alphaup{\alpha^{+}}                    
\newcommandx{\bwalpha}[1][1=u]{\mathrm{h}_{\alpha}}
\def\J{J_{\alpha, h}}
\newcommandx{\auxX}[1][1=n+1]{\widetilde{X}_{#1}}
\newcommandx{\rlcpq}[3][1=x,2=\tilde{x},3=1-\alpha]{\widehat{q}^{\mathrm{R}}_{#3,h}(#1;#2)}  
\newcommandx{\rlcpC}[2][1=x,2=\tilde{x}]{\widehat{\mathcal{C}}^{\mathrm{R}}_{\alpha}(#1;#2)} 
\newcommand{\CondPDF}[1][S]{p_{#1\mid X}}
\newcommand{\rhom}{\rho_{n, h}}

\newcommand{\creg}{c_{\mathrm{reg}}}

\newcommand{\covdif}[1]{\mathcal{E}_{#1}}

\newcommand{\conformalset}[3][S]{\mathcal{C}^{(#1)}(#2, #3)} 
\newcommandx{\SetAdapt}[1][1={n,\ntrain,\tilde x,h,\delta}]{\Omega^{\mathrm{adapt}}_{#1}}
\newcommandx{\SetTrain}[1][1={\ntrain,h,\delta/2}]{\Omega^{\mathrm{train}}_{#1}} 

\newcommand{\Llen}[1][x]{L_{\mathrm{len}}(#1)} 
\newcommandx{\taustar}[1][1=\alpha]{\tau^{\star}_\alpha}

\newcommand{\cqrquant}[1]{q_{#1}}

\newcommand{\shapefun}{\psi}  
\newcommand{\shapefunset}{\Psi}
\newcommandx{\etah}[1][1=h]{\eta_{#1}}                      
\def\uniflow{\mu_{\operatorname{low}}} 

\newcommand{\rhofloor}{\rho_0}
\newcommand{\rstar}{r_\star}

\newcommand{\pfloor}{\underline{p}}
\newcommand{\Xzero}{\mathcal{X}_0}

\makeatletter
\newtheorem*{rep@theorem}{\rep@title}
\newcommand{\newreptheorem}[2]{%
  \newenvironment{rep#1}[2][]{%
    \def\rep@title{#2~\ref{##2}%
      \if\relax\detokenize{##1}\relax\else\space(##1)\fi}%
    \begin{rep@theorem}}
  {\end{rep@theorem}}}
\makeatother

\newreptheorem{theorem}{Theorem}
\newreptheorem{lemma}{Lemma}
\newreptheorem{proposition}{Proposition}
\newreptheorem{corollary}{Corollary}

\newcommand{\runinhead}[1]{\par\medskip\noindent\textbf{#1}\hspace{0.6em}\ignorespaces}
\newenvironment{acks}[1][Acknowledgments]{\section*{#1}}{}

\hypersetup{
  pdftitle={Beyond Marginal Validity: Finite-Sample Guarantees for Localized Conformal Prediction},
  pdfauthor={Anton Conrad, Rustam Isaev, Denis Belomestny, Eric Moulines, Sergey Samsonov},
}

\title{Beyond Marginal Validity: Finite-Sample Guarantees for\\
Localized Conformal Prediction}

\author{%
  Anton Conrad\textsuperscript{1}\thanks{Corresponding author: \texttt{anton.conrad@epita.fr}}
  \and Rustam Isaev\textsuperscript{2,5}
  \and Denis Belomestny\textsuperscript{2,3}
  \and Eric Moulines\textsuperscript{1,4}
  \and Sergey Samsonov\textsuperscript{2}}
\date{}

\begin{document}

\maketitle

\begingroup
\footnotesize\noindent
\textsuperscript{1}Laboratoire de Recherche d'EPITA,
14--16 rue Voltaire, 94276 Le Kremlin-Bic\^etre CEDEX, France\\
\textsuperscript{2}Faculty of Computer Science, HSE University,
11 Pokrovsky Boulevard, Moscow 109028, Russia\\
\textsuperscript{3}Faculty of Mathematics, University of Duisburg--Essen,
Thea-Leymann-Str.\ 9, 45127 Essen, Germany\\
\textsuperscript{4}Computing and Mathematical Sciences Division,
Mohamed bin Zayed University of Artificial Intelligence,
P.O. Box 7909, Masdar City, Abu Dhabi, United Arab Emirates\\
\textsuperscript{5}Faculty of Computational Mathematics and Cybernetics,
Lomonosov Moscow State University,
1 Leninskie Gory, Building 52, Moscow 119991, Russia
\endgroup

\begin{abstract}
Conformal prediction endows arbitrary black-box predictors with finite-sample, distribution-free \emph{marginal} coverage, yet marginal validity can hide severe covariate-specific miscalibration, while exact distribution-free \emph{conditional} coverage is finite-sample unattainable. Randomly localized conformal prediction (RLCP) mitigates this gap by calibrating near the test point while preserving marginal coverage. Existing theory, however, lacks finite-sample guarantees for the realized localized set that jointly control conditional validity and oracle efficiency. We provide such guarantees. For any fixed score, under H\"older regularity of the conditional score CDF and standard density and kernel assumptions, we prove high-probability bounds, uniform over a realized localization neighbourhood, for the conditional-coverage gap and the length error relative to the oracle. The bounds decompose into an $O(h^{\beta})$ localization bias and a calibration term decreasing with calibration size, clarifying the bandwidth bias--variance tradeoff and when RLCP tracks the oracle. We also analyze data-split learned scores: when the score targets a pivotal score, as in conformalized quantile regression, uniform local guarantees decompose into fixed-score calibration and uniform score-estimation errors, showing that improved learning sharpens localized guarantees.
\end{abstract}


\smallskip
\noindent\textbf{MSC2020 subject classifications:} Primary 62G08;
secondary 62G15, 62G20.

\section{Introduction}
\label{sec:intro}

Conformal prediction constructs prediction sets with finite-sample,
distribution-free marginal validity under exchangeability
\citep{vovk2005algorithmic,shafer2008tutorial,lei2018distributionfree}.
This guarantee is attractive because it is insensitive to the complexity of the
underlying prediction rule.  It is, however, an average statement over the
covariate distribution: a prediction set may attain the nominal coverage while
systematically undercovering in some regions of the covariate space and
overcovering in others.  The corresponding pointwise requirement,
\[
  \PP\{Y\in C(x)\mid X=x\}\ge 1-\alpha,
  \qquad x\in\Xset,
\]
cannot be achieved by nontrivial distribution-free procedures in finite
samples without additional assumptions
\citep{lei2014distributionfree,barber2021limits}.  It is therefore natural to
ask how closely a marginally valid procedure approximates conditional validity
under explicit regularity conditions, and whether this approximation is
obtained without an excessive increase in the size of the prediction set.

Localized conformal methods address this question by assigning larger
calibration weights to observations whose covariates are close to the test
point \citep{guan2019conformal,guan2023localized}.  Direct localization alters
the exchangeability argument and generally requires a correction of the
nominal level.  Randomly localized conformal prediction (RLCP), introduced by
\citet{hore2025conformal}, instead draws an auxiliary localization centre and
weights both the calibration and test observations relative to that centre.
The resulting reverse-law representation is a weighted conformal construction
and yields finite-sample marginal validity without a data-dependent level
correction.  Existing analyses establish marginal validity, relaxed local
guarantees, and robustness properties for RLCP
\citep{hore2025conformal,barber2025unifying}.  More recently,
\citet{min2026unified} derived a non-asymptotic decomposition of conditional
miscoverage and obtained a pointwise conditional-coverage rate for RLCP after
averaging over the calibration sample and the auxiliary randomization.

The aim of this paper is to describe the finite-sample behaviour of the
\emph{realized} localized prediction set.  We consider two quantities at a
covariate value $x$: its conditional-coverage error and the difference between
its length and that of the conditional score-oracle region
$\conformalset{x}{\oraclequant{x}}$.  The analysis is conditional on the
realized auxiliary centre $\tilde x$ and therefore retains the randomness that
determines the RLCP set.  This formulation makes explicit the two
sources of error introduced by localization.  A smaller bandwidth reduces the
discrepancy between the conditional score distribution at $x$ and its local
mixture, but it also reduces the effective calibration sample size.

We first treat a deterministic score $S$; equivalently, the results may be
read conditionally on an independent sample used to construct the score.  We
assume that the conditional score distribution is H\"older continuous in the
covariate, together with local lower density and standard kernel and design
regularity conditions.  With probability at least $1-\delta$ over the
calibration sample, conditional on $\tilde x$, and for bandwidths and
calibration sizes meeting the explicit admissibility thresholds of
\Cref{sec:fixed_score}, the conditional-coverage error
and the oracle-relative length error are both bounded, simultaneously for
$x\in B(\tilde x,h)\cap\Xset$, by a constant multiple of
\begin{equation}
  \label{eq:intro-fixed-rate}
  h^\beta
  +\left\{\frac{\log(4/\delta)}
  {n\pmin h^d}\right\}^{1/2}
  +\frac{\log(4/\delta)}{n\pmin h^d}.
\end{equation}
The multiplying constant depends only on the structural constants of the
assumptions; for the length error, it additionally involves the localized score-density floor
 at the realized centre $\tilde x$ and the local length-regularity constant. 
The first term is the localization bias and the other two terms are governed
by the uniform lower bound $n\pmin h^d$ on the local effective sample size.
Balancing the leading terms gives the usual
nonparametric rate $n^{-\beta/(2\beta+d)}$, up to logarithmic factors.
The event underlying this statement is common to every
$x\in B(\tilde x,h)\cap\Xset$, so the conclusion is uniform over the
neighbourhood selected by the realized centre.

We next consider a score $\Shat$ learned on an independent training sample.
The benchmark is a population score $\Sstar$ whose $(1-\alpha)$-quantile is
the same at every covariate value.  This pivotality condition is satisfied by
the population conformalized quantile regression score
\citep{romano2019conformalized} and by distributional scores based on the
conditional probability integral transform
\citep{chernozhukov2021distributional}.  Since the localized oracle quantile
then coincides with the same pivot at every centre and bandwidth for which
the localized density minorization holds, the
$h^\beta$ term in \eqref{eq:intro-fixed-rate} disappears.  We obtain uniform
local bounds in which the error is the sum of the calibration term and the
score-estimation error, and the latter enters linearly: the threshold
deviation involves the training error only through its average over the
neighbourhood selected by the realized centre, while the set comparison adds
a pointwise term, both controlled by a high-probability uniform bound on the
score-estimation error.

The principal conclusions are therefore twofold.  For a fixed score, the
paper gives uniform-local, high-probability control of
both validity and length for the realized RLCP set.  For a learned pivotal score, it gives a
localized oracle comparison that separates calibration from score estimation.
The assumptions used to invert localized score distributions are verified for
the residual, conformalized quantile regression, and distributional scores.
Numerical experiments illustrate the bandwidth trade-off in
\eqref{eq:intro-fixed-rate} and the linear propagation of score-estimation
error.  A detailed comparison with related conditional conformal and
efficiency results is given in \Cref{sec:rw}.

\section{Related Work}
\label{sec:rw}
Standard conformal prediction gives exact finite-sample marginal validity
under exchangeability
\citep{vovk2005algorithmic,shafer2008tutorial,lei2018distributionfree}, whereas
exact distribution-free conditional coverage at a continuously distributed
covariate is impossible without essentially uninformative sets
\citep{lei2014distributionfree,barber2021limits}.  Approximate conditional
validity therefore requires weaker targets or structural assumptions.  Here,
local smoothness and density conditions permit comparison of RLCP with the
conditional score oracle at a fixed covariate.

Localized conformal methods replace the empirical score law by a
covariate-weighted law near the test point
\citep{guan2019conformal,guan2023localized}.  RLCP randomizes an auxiliary
centre and admits a weighted-conformal representation under the reverse
conditional law \citep{hore2025conformal}.  This yields marginal validity,
local or subpopulation guarantees, and robustness to covariate shift.  The
partial-information framework of \citet{barber2025unifying} unifies RLCP with
related weighted procedures.  These distribution-free results principally
describe the validity mechanism and guarantees for local populations.  They
do not, however, bound the realized RLCP threshold relative to the pointwise
conditional oracle or the resulting length error under smoothness.

The closest comparison is \citet{min2026unified}, who express a broad class of
conditional conformal methods through weighted quantiles and decompose
conditional miscoverage into estimation, calibration, and intrinsic mismatch.
For fixed-score RLCP, their Corollary~S3 assumes compact covariate support, a
globally bounded and bounded-away-from-zero design density, a regular radial
kernel, bounded scores, a uniform lower score-density bound, and Lipschitz
conditional score quantiles.  For fixed $t$, it establishes
\begin{equation}
  \label{eq:rw-min-rlcp-rate}
  \left|
    \PP\{Y_{n+1}\in\widehat C^{\rm RLCP}(X_{n+1})
      \mid X_{n+1}=t\}-(1-\alpha)
  \right|
  \le C\{(nh^d)^{-1/2}\log^{1/2}n+h\}.
\end{equation}
where the probability averages over calibration and localization randomness.
By contrast, our fixed-score bounds hold, with calibration probability at
least $1-\delta$, simultaneously for all
$x\in B(\tilde x,h)\cap\Xset$ at the realized centre $\tilde x$.  Their leading
terms are
\[
  h^\beta
  +\left\{\frac{\log(4/\delta)}
  {n\pmin h^d}\right\}^{1/2}
  +\frac{\log(4/\delta)}{n\pmin h^d}.
\]
Relative to Corollary~S3, our regularity conditions permit H\"older, rather
than Lipschitz, variation of the conditional score law, and require only a
uniform lower bound on the design density, without a global upper bound.  At
fixed confidence, the calibration term lacks the $\log^{1/2}n$ factor in
\eqref{eq:rw-min-rlcp-rate}; taking $\delta=n^{-1}$ restores it.  The
distinction is therefore the mode of
control: our event is sample-specific and uniform over a realized
neighbourhood, and it additionally controls oracle-relative length.
This conclusion does not constitute an unconditional rate improvement; it
provides stronger control conditional on the realized calibration sample and
localization centre.  Conversely, \citet{min2026unified} cover substantially
broader classes of procedures, kernels, conditioning events, and structured
data.  The two analyses are accordingly complementary: theirs gives a general
conditional-miscoverage theory, whereas ours gives a finite-sample oracle
analysis specialized to RLCP.

Adaptivity may instead enter through the score.  CQR uses estimated
conditional quantiles \citep{romano2019conformalized}, while distributional
methods estimate the conditional law
\citep{chernozhukov2021distributional,izbicki2022cdsplit}.  Their population
scores are pivotal, so a common target threshold creates no intrinsic
conditional mismatch.  Corollary~S4 of \citet{min2026unified} exploits this
property for globally calibrated CQR and obtains a uniform estimation term
plus $n^{-1/2}\log^{1/2}n$.  Our learned-score analysis combines RLCP with an
independently fitted score under a high-probability uniform bound on the
score-estimation error.  It controls conditional miscoverage and
oracle-relative length uniformly over the realized neighbourhood, separating
calibration error from the linear contribution of the score-estimation
error, which enters the threshold comparison only through its average over
$\Lambda_{\tilde x,h}$.  Pivotality removes localization
bias; the bandwidth still determines the effective calibration size and the
region over which the score error is averaged.

Related efficiency criteria include expected split-conformal set size
\citep{dhillon2024expected}, length-optimized conditionally valid sets
\citep{kiyani2024length}, and non-asymptotic excess length for conformalized
regression \citep{yao2025non}.  These criteria concern expected size,
group-conditional objectives, or specific learned regression procedures.  In
contrast, the present criterion is local and pointwise: the conditional
score-oracle region is the common benchmark for coverage and realized-set
length.

\section{RLCP with a Fixed Score}
\label{sec:fixed_score}
\subsection{Setting and notations}
\label{sec:setting}

\runinhead{Data and conditional laws.}
Let $(X,Y)$ take values in $\Xset\times\Yset$, where
$\Xset\subseteq\R^d$ and $\Yset=[-M,M]$, with joint law
$P_{XY}=P_X\otimes P_{Y\mid X}$.  We observe an i.i.d. calibration sample
$\Dcal[n]:=\{(X_i,Y_i)\}_{i\in[n]}$, $[n]:=\{1,\ldots,n\}$,
and the covariate $X_{n+1}$ of an independent test pair
$(X_{n+1},Y_{n+1})\sim P_{XY}$.  The objective is to construct a prediction
set for $Y_{n+1}$ at miscoverage level $\alpha\in(0,1)$.

All conditional laws below refer to fixed regular conditional versions.  In
particular, statements involving $P_{Y\mid X=x}$ or a conditional score law
are interpreted pointwise whenever the stated assumptions make the chosen
versions well defined.  We write $B(x,r)$ for the closed Euclidean ball and
$|A|$ for the Lebesgue measure of a measurable set $A$.

\begin{assumP}
\label{assumP:marginal-density}
The marginal law $P_X$ has a density $p_X$ on $\Xset$ satisfying
$p_X(x)\ge\pmin>0$ for every $x\in\Xset$.  Moreover, $\Xset$ is
$(\creg,r_0)$-regular: there exist $\creg,r_0>0$ such that
\[
  |\Xset\cap B(x,r)|
  \ge \creg |B(x,r)|,
  \qquad x\in\Xset,\quad 0<r\le r_0.
\]
The radius $r_0$ is fixed once and for all; the same constant is used in
\Cref{assumK:kernel_assum} below.
\end{assumP}
This support condition is standard in local nonparametric analysis; see, for
example, \citet{audibert2007fast}. Let $S:\Xset\times\Yset\to\R$ be a fixed non-conformity score and
$S_i=S(X_i,Y_i)$.  The results also hold conditionally on an independent
training sample used to construct $S$.  For $x\in\Xset$, let
$\CondCDF(t\mid x):=\PP\{S(X,Y)\le t\mid X=x\}$ and, for any distribution
function $F$, let $\quantq(p;F):=\inf\{t:F(t)\ge p\}$.  Define
\begin{equation}
\label{eq:conf_set}
  \oraclequant{x}:=\quant{1-\alpha}{\CondCDF(\cdot\mid x)},
  \qquad
  \conformalset{x}{q}:=\{y\in\Yset:S(x,y)\le q\}.
\end{equation}
The set $\conformalset{x}{\oraclequant{x}}$ is the conditional score-oracle
region.  If $\CondCDF(\cdot\mid x)$ is continuous at $\oraclequant{x}$, then
$\CondCDF(\oraclequant{x}\mid x)=1-\alpha$, so that
$\PP\{Y\in\conformalset{x}{\oraclequant{x}}\mid X=x\}=1-\alpha$.
This oracle is infeasible because it depends on the unknown conditional score
law; it will serve as the benchmark for the data-driven construction below.

\noindent To approximate the score law near a target covariate, we introduce a
localization kernel $\KDE:\R^d\to\R_+$ satisfying the following assumption.
\begin{assumK}[Kernel regularity]
\label{assumK:kernel_assum}
The kernel $\KDE$ is bounded, integrates to one, and is symmetric.  Its
support is contained in $B(0,1)$.  In addition, there
exist $\kappa>0$ and $\eta\in\ocint{0,r_0}$ such that $\KDE(v)\ge\kappa\,\indiacc{\|v\|\le\eta}$,
 $v\in\R^d$.
Here $r_0$ is the same support-regularity radius as in
\Cref{assumP:marginal-density}.
\end{assumK}

For a centre $x\in\R^d$ and bandwidth $h>0$, set, for $u\in\Xset$,
$\locweight{u}:=\KDE((u-x)/h)$ and
\begin{equation}
\label{eq:normalizing-constant}
  \locnorm:=\E_{U\sim P_X}[\locweight{U}]
  =\int_{\Xset}\KDE\!\left((u-x)/h\right)p_X(u)\,\rmd u.
\end{equation}
Introduce the structural constant $c_0:=\kappa\creg\eta^d|B(0,1)|$.
By \Cref{lem:gamma_bound_exact}, for every $\tilde x\in\Xset$ and
$0<h\le1$, $\locnorm[\tilde x] \ge c_0\pmin h^d$.
This inequality identifies $n\pmin h^d$ as a uniform lower bound on the local effective
calibration size.  Whenever $\locnorm>0$, define the localized covariate
law
\begin{equation}
\label{eq:lambda}
  \locmeasure[x,h](\rmd u)
  :=(\locweight{u}/\locnorm)P_X(\rmd u).
\end{equation}
The corresponding localized score distribution and quantile are
\begin{equation}
\label{eq:localized-cdf}
  \CDFloc{x}(t)
  :=\int_{\Xset}\CondCDF(t\mid u)\locmeasure[x,h](\rmd u),
  \qquad
  \locoraclequant{x}:=\quant{1-\alpha}{\CDFloc{x}}.
\end{equation}
Thus, $\locmeasure[x,h]$ concentrates the population law around $x$, while
$\CDFloc{x}$ averages the conditional score laws in that neighborhood.  A
smaller bandwidth gives a closer approximation to the conditional oracle but
uses fewer calibration observations; a larger bandwidth has the opposite
effect. RLCP introduces an auxiliary centre $\auxX$.  Following
\citet{hore2025conformal}, draw it, conditionally on $X_{n+1}$ and
independently of $(\Dcal[n],Y_{n+1})$, from the convolution kernel
$H_h(X_{n+1},\cdot)$, where
\begin{equation}
\label{eq:hb-forward-kernel}
  H_h(x,\rmd\tilde x)
  =h^{-d}\KDE((\tilde x-x)/h)\,\rmd\tilde x,
  \qquad x\in\Xset,\quad \tilde x\in\R^d.
\end{equation}
This construction is valid for an arbitrary support $\Xset$.  The marginal
law of $\auxX$ has Lebesgue density
$\tilde x\mapsto h^{-d}Z_{\tilde x,h}$, and symmetry of $\KDE$ gives, for
$P_{\auxX}$-almost every $\tilde x$ with $Z_{\tilde x,h}>0$,
\begin{equation}
\label{eq:bayes-compatible-auxiliary-kernel}
  \PP\{X_{n+1}\in A\mid\auxX=\tilde x\}
  =\Lambda_{\tilde x,h}(A)
  =Z^{-1}_{\tilde x,h} \int_{A\cap\Xset}\locweight[\tilde x]{v}P_X(\rmd v).
\end{equation}
This reverse-law identity is the key probabilistic device: after conditioning
on the auxiliary centre, the test covariate has exactly the distribution
represented by the localization weights.
Thus, conditionally on $\auxX=\tilde x$, the test pair has law
$\Lambda_{\tilde x,h}\otimes P_{Y\mid X}$, whereas the calibration sample
retains law $P_{XY}^{\otimes n}$.  This is covariate shift with density ratio
$\rmd\Lambda_{\tilde x,h}/\rmd P_X=\locweight[\tilde x]{\cdot}/Z_{\tilde x,h}$.
Unlike compatibility, membership of the auxiliary centre in $\Xset$ is not
automatic near the boundary.  This distinction does not affect the weighted
conformal marginal guarantee below, but the fixed-centre theorems require
$\tilde x\in\Xset$.

Fix a realized centre $\tilde x$.  For a candidate test covariate $x$, define
the normalized density-ratio weights $  D_{x,\tilde x}:=\sum_{j\in[n]}\locweight[\tilde x]{X_j}
  +\locweight[\tilde x]{x}$ and
\begin{equation}
\label{eq:rlcp-weights}
  \nlocweightnorm[x]{\tilde x}{i}
  :=(\locweight[\tilde x]{X_i}/D_{x,\tilde x}) \, (i\in[n]),
  \qquad
  \nlocweightnorm[x]{\tilde x}{n+1}
  := (\locweight[\tilde x]{x})/D_{x,\tilde x}.
\end{equation}
If the denominator vanishes, the conformal threshold is set to $+\infty$.
As in split conformal prediction, the unknown test score is represented by a
point mass at $+\infty$.  This is the usual finite-sample correction: when the
test point carries too much weight, the procedure returns a conservative set
rather than relying on insufficient local calibration mass.  Define
\begin{equation}
\label{eq:rlcp-prediction-set}
\begin{aligned}
  \nEmpCDFloc{\tilde x}(t)
  &:=\sum_{i\in[n]}\nlocweightnorm{\tilde x}{i}
  \indi{\ocint{-\infty,t}}(S_i)
  +\nlocweightnorm{\tilde x}{n+1}\indiacc{+\infty\le t},
  \\
  \rlcpq[x][\tilde x]
  &:=\quant{1-\alpha}{\nEmpCDFloc{\tilde x}},
  \qquad
  \rlcpC[x][\tilde x]
  :=\conformalset{x}{\rlcpq[x][\tilde x]}.
\end{aligned}
\end{equation}
By the reverse-law representation, this is weighted conformal prediction
under the covariate shift induced by $\auxX$.  Hence, by
\citet[Theorem~1]{hore2025conformal},
$\PP\{Y_{n+1}\in \rlcpC[X_{n+1}][\tilde X_{n+1}]\}\ge1-\alpha$.
Our objective is to complement this marginal guarantee with conditional
coverage and oracle-relative length bounds for the realized set.

The next four assumptions have distinct roles.  The upper density bound turns a
threshold error into a coverage error; the lower localized density permits
quantile inversion -- it can be relaxed, but at the expense of many technicalities; H\"older continuity controls localization bias; and Lipschitz continuity of the set length turns a threshold error into a length
error.
\begin{assum}[$S$]
\label{assum:S_up}
For every $x\in\Xset$, the conditional score law admits a jointly measurable
density $\CondPDF(\cdot\mid x)$ satisfying
$\CondPDF(s\mid x)\le\ups$ for some $\ups>0$, all $x\in\Xset$, and all
$s\in\R$.
\end{assum}

Fix $0<\alphalow<\alpha<\alphaup<1$ and write
$\boldsymbol\alpha=(\alpha,\alphalow,\alphaup)$.  Define
\begin{equation}
\label{eq:definition-I}
  \I[u]
  :=\ccint{\locoraclequant{u}[1-\alphaup],
  \locoraclequant{u}[1-\alphalow]},
  \quad
  m_{\boldsymbol\alpha}
  :=(\alpha-\alphalow)\wedge(\alphaup-\alpha).
\end{equation}
Under \Aref{assum:S_up}{S}, the localized score density is
\begin{equation}
\label{eq:localized-pdf}
  \PDFloc{u}(t)
  :=\int_{\Xset}\CondPDF(t\mid v)\locmeasure[u,h](\rmd v),
  \qquad t\in\R.
\end{equation}
For $u\in\Xset$ and $0<h\le1$, define the localized score-density
minorization constant
\begin{equation}
\label{eq:pdfloc_bound}
  \lows{u}
  :=\essinf_{t\in\I[u]}\PDFloc{u}(t).
\end{equation}

\begin{assum}[$S$]
\label{assum:lip_condcdf}
There exist $\beta\in(0,1]$ and $L_S>0$ such that
\[
  \sup_{t\in\R}|\CondCDF(t\mid u)-\CondCDF(t\mid u')|
  \le L_S\|u-u'\|^\beta,
  \qquad u,u'\in\Xset.
\]
\end{assum}

For the levels fixed above, set
\begin{equation}
\label{eq:definition-bias-bandwidth}
  \mathrm h^{\mathrm{bias}}_{\boldsymbol\alpha}
  :=(1/2) (m_{\boldsymbol\alpha}/L_S)^{1/\beta}.
\end{equation}
Thus $h\le\mathrm h^{\mathrm{bias}}_{\boldsymbol\alpha}$ implies
$L_S(2h)^\beta\le m_{\boldsymbol\alpha}$, which keeps the localization
perturbation inside the quantile window $\I[u]$.

\begin{assum}[$S$]
\label{assum:S_len}
For every $x\in\Xset$, the map
$q\mapsto|\conformalset{x}{q}|$ is $\Llen[x]$-Lipschitz on $\R$.
\end{assum}

The following lemma gives a nontrivial lower bound for $ \lows{u}$ in \eqref{eq:pdfloc_bound}.  
Let $\shapefunset$ be the class of functions $\shapefun:[0,1]\to[0,\infty)$ that are positive and lower
semicontinuous on $(0,1)$ and vanish at the endpoints.  For this sufficient
condition only, use the convenient symmetric choice
\begin{align}
\label{eq:definition-margin_hmax}
  \alphalow:=\alpha-\tfrac12\alpha(1-\alpha),
  \quad
  \alphaup:=\alpha+\tfrac12\alpha(1-\alpha),
  \quad
  m_\alpha:=\tfrac12\alpha(1-\alpha),
  \\
\label{eq:definition-hmax_window}
  \bwalpha
  :=\frac12\left(m_\alpha/L_S\right)^{1/\beta},
  \qquad
  \J
  :=\ccint{1-\alphaup-L_S(2h)^\beta,
  1-\alphalow+L_S(2h)^\beta}.
\end{align}

\begin{lemma}
\label{lem:check-assum:pdfloc_bound}
Let $\shapefun\in\shapefunset$.  Suppose
\Cref{assumP:marginal-density,assumK:kernel_assum},
\Aref{assum:S_up}{S}, and \Aref{assum:lip_condcdf}{S} hold.  Assume
that, for every $u\in\Xset$ and $0<h\le1$, there exists
$\sigma_{\shapefun}(u,h)>0$ such that
\begin{equation}
\label{eq:lowerbound_sym}
  \CondPDF(s\mid z)
  \ge\sigma_{\shapefun}(u,h)
  \shapefun(\CondCDF(s\mid z))
\end{equation}
for every $z\in B(u,h)\cap\Xset$ and $s\in\R$.  Then, for every
$u\in\Xset$ and $0<h\le1\wedge\bwalpha$, the minorization condition
\eqref{eq:pdfloc_bound} holds:
$\lows{u}\ge\sigma_{\shapefun}(u,h)\inf_{r\in\J}\shapefun(r)>0$.
\end{lemma}

We adopt the convention $1/0:=+\infty$. While \Cref{thm:rlcp_length} below remains valid for $\lows{u} = 0$, the resulting upper bound is vacuous. Consequently, the length analysis is informative only when $\lows{u} > 0$. The preceding lemma ensures this strict positivity under a more intuitive condition, formulated in terms of the conditional score density itself rather than its localized counterpart.

\subsection{Main results}
Write $A\lesssim B$ if $A\le CB$, where $C$ depends only on the structural
constants in the assumptions and not on $n,h,\delta$ or the realized
covariates.  Define
\begin{equation}
\label{eq:th-bound-definition}
  A_{n,h}^{\mathrm{cal}}(\delta)
  :=\left\{\frac{\log(4/\delta)}
  {n\pmin h^d}\right\}^{1/2}
  +\frac{\log(4/\delta)}{n\pmin h^d}.
\end{equation}
This is the calibration error at the local effective sample size.  For a fixed
score, the bounds additionally contain the localization bias $h^\beta$ and
therefore display the usual local bias--variance trade-off:
increasing $h$ improves the effective sample size but makes the localized law
less representative of the target covariate.

\begin{theorem}[Uniform local length control]
\label{thm:rlcp_length}
Let $\alpha,\delta\in\ooint{0,1}$,
$\alphalow\in\ooint{0,\alpha}$, $\alphaup\in\ooint{\alpha,1}$, and
$\boldsymbol\alpha=(\alpha,\alphalow,\alphaup)$.  Let
  $\tilde x\in\Xset$ and
  $0<h\le1\wedge\mathrm h^{\mathrm{bias}}_{\boldsymbol\alpha}$.  Suppose
\Cref{assumP:marginal-density,assumK:kernel_assum},
\Aref{assum:S_up}{S}, \Aref{assum:lip_condcdf}{S}, and
\Aref{assum:S_len}{S} hold.  If
$n\ge n_1(h,\delta,\boldsymbol\alpha)$, with $n_1$ defined in
\eqref{eq:definition-n1}, then, on the event
$\SetOmega[n,\tilde x,h,\delta]$ defined in \eqref{eq:definition-Omega}, which
has probability at least $1-\delta$ over the calibration sample, the following
holds simultaneously for every $x\in B(\tilde x,h)\cap\Xset$:
\[
  \left||\rlcpC|-|\conformalset{x}{\oraclequant{x}}|\right|
  \lesssim
  \Llen[x]\bigl\{A_{n,h}^{\mathrm{cal}}(\delta)+h^\beta\bigr\}
  / \lows{\tilde x}.
\]
\end{theorem}

\begin{theorem}[Uniform local conditional coverage control]
\label{thm:rlcp_coverage}
Let $\alpha,\delta\in\ooint{0,1}$, $\tilde x\in\Xset$ and $0<h\le1$.
Suppose \Cref{assumP:marginal-density,assumK:kernel_assum},
\Aref{assum:S_up}{S}, and \Aref{assum:lip_condcdf}{S} hold.  If
$n\ge n_2(h,\delta,\alpha)$, with $n_2$ defined in
\eqref{eq:definition-n2}, then, on the event
$\SetOmega[n,\tilde x,h,\delta]$ defined in \eqref{eq:definition-Omega},
which has probability at least $1-\delta$ over the calibration sample,
the following holds simultaneously for every
$x\in B(\tilde x,h)\cap\Xset$:
\[
  \left|
  \int_{\Yset}\indi{\rlcpC}(y)P_{Y\mid X}(\rmd y\mid x)
  -(1-\alpha)
  \right|
  \lesssim
  A_{n,h}^{\mathrm{cal}}(\delta)+h^\beta.
\]
\end{theorem}

Setting $\delta=n^{-1}$ in \eqref{eq:th-bound-definition}, the leading term of
$A_{n,h}^{\mathrm{cal}}(n^{-1})$ decreases in $h$, whereas the localization
bias $h^\beta$ increases.  The bandwidth equating them,
\begin{equation}
\label{eq:h-optimal}
  \left\{\frac{\log(4n)}{n\pmin h^d}\right\}^{1/2}=h^\beta
  \quad\Longleftrightarrow\quad
  h=h_\star:=\left\{\frac{\log(4n)}{n\pmin}\right\}^{1/(2\beta+d)},
\end{equation}
minimizes the sum of these two terms to within a factor of two---the
exact minimizer is $(d/2\beta)^{2/(2\beta+d)}h_\star$, which leaves the
rate unchanged.  At $h=h_\star$ the quadratic calibration term equals
$h_\star^{2\beta}$ exactly, so
\begin{equation}
\label{eq:optimal-rate}
  A_{n,h_\star}^{\mathrm{cal}}(n^{-1})+h_\star^\beta
  = 2h_\star^{\beta}+h_\star^{2\beta}
  = 2\left\{\frac{\log(4n)}{n\pmin}\right\}^{\beta/(2\beta+d)}
   +\left\{\frac{\log(4n)}{n\pmin}\right\}^{2\beta/(2\beta+d)}
  \le 3\left\{\frac{\log(4n)}{n\pmin}\right\}^{\beta/(2\beta+d)},
\end{equation}
the final inequality holding as soon as $n\pmin\ge\log(4n)$.  This
threshold is exactly the condition $h_\star\le1$: the balanced bandwidth
is admissible, the local effective sample size satisfies
$n\pmin h_\star^d\ge\log(4n)$, and the quadratic term is of lower order,
all simultaneously. \Cref{thm:rlcp_length,thm:rlcp_coverage} at $h=h_\star$ therefore give the
classical pointwise nonparametric rate $n^{-\beta/(2\beta+d)}$ under
$\beta$-H\"older smoothness and the logarithmic factor enters only
through the confidence choice $\delta=n^{-1}$. Second, at $h_\star$ the effective
local calibration size is
$n\pmin h_\star^{d}
=(n\pmin)^{2\beta/(2\beta+d)}\{\log(4n)\}^{d/(2\beta+d)}$: the smoother
the conditional score law, the wider the admissible neighbourhood and the
larger the fraction of the calibration sample that localization retains.

\runinhead{Proof outline.}
On $\SetOmega[n,\tilde x,h,\delta]$,
\Cref{prop:rlcp_quantile_oracle} decomposes the localized population quantile as follows:
\begin{equation}
\label{eq:split-threshold}
  |\rlcpq-\oraclequant{x}|
  \le |\rlcpq-\locoraclequant{\tilde x}|
     +|\locoraclequant{\tilde x}-\oraclequant{x}|.
\end{equation}
For the first term, concentration of the kernel-weighted empirical CDF and
the mass placed at $+\infty$ produce the level error
$\varepsilon_{n,h}(\delta)+2(1-\alpha)\rhom$, where  $\rhom$ and $\varepsilon_{n,h}(\delta)$  are defined in \eqref{eq:definition-rhom} and \eqref{eq:definition-varepsilon}, respectively.
The definition of $n_1$ bounds this error by
$m_{\boldsymbol\alpha}$, defined in \eqref{eq:definition-I}, so all perturbed levels remain between
$1-\alphaup$ and $1-\alphalow$.  The density lower bound then converts the
level error into the quantile bound \eqref{eq:rev-T1}.

For the second term, the localized law is supported on $B(\tilde x,h)$;
hence its covariates lie within $2h$ of $x$.
\Aref{assum:lip_condcdf}{S} gives the Kolmogorov error $L_S(2h)^\beta$, leading to
\begin{equation*}
  |\rlcpq-\oraclequant{x}|
  \le \Delta^{\mathrm R}_{n,h}(\tilde x;\delta)
  \lesssim
  \bigl\{A_{n,h}^{\mathrm{cal}}(\delta)+h^\beta\bigr\}
  /\lows{\tilde x}.
\end{equation*}
It remains to translate this threshold bound.  For the length result,
\Aref{assum:S_len}{S} and $\rlcpC=\conformalset{x}{\rlcpq}$ give
\[
  \bigl||\rlcpC|-|\conformalset{x}{\oraclequant{x}}|\bigr|
  \le \Llen[x]|\rlcpq-\oraclequant{x}|.
\]
For \Cref{thm:rlcp_coverage}, conditional coverage equals
$\CondCDF(\rlcpq\mid x)$, and no quantile inversion is required: the
error is controlled directly on the CDF scale.  With
$p_{n,h}(x;\tilde x)$ the effective calibration level
\eqref{eq:effective-level}, the coverage error splits into three parts:
the localization bias $L_S(2h)^\beta$, from comparing
$\CondCDF(\cdot\mid x)$ with $\CDFloc{\tilde x}$; the calibration error
$|\CDFloc{\tilde x}(\rlcpq)-p_{n,h}(x;\tilde x)|
\le\varepsilon_{n,h}(\delta)$, from evaluating $\CDFloc{\tilde x}$ at
the empirical quantile; and the level shift
$0\le p_{n,h}(x;\tilde x)-(1-\alpha)\le2(1-\alpha)\rhom$, due to the
mass placed at $+\infty$. 

The event $\SetOmega[n,\tilde x,h,\delta]$ does not depend on $x$, so the
comparison holds simultaneously throughout the localization neighbourhood.
The rate separates calibration at the uniformly controlled effective sample size
$n\pmin h^d$ from the localization bias $h^\beta$, while
$1/\lows{\tilde x}$ is the cost of quantile inversion in the length
bound.

\subsection{Example: residual score}
\label{subsec:ex-residual-score}

We conclude by verifying the assumptions for the residual score.  Suppose the
conditional law of $Y$ has a density $p_{Y\mid X}$ on $[-M,M]$, and define
\begin{equation}
\label{eq:ex-mu-def}
  \mu(x):=\int_{\Yset}y p_{Y\mid X}(y\mid x)\,\rmd y,
  \qquad
  S_{\mathrm{res}}(x,y):=|y-\mu(x)|.
\end{equation}
This score yields prediction sets centred at the conditional mean.  The
proposition below shows that the preceding abstract conditions follow from
standard upper, lower, and continuity bounds on the conditional density.

\begin{assumP}[Conditional density]
\label{assumP:density}
For every $x\in\Xset$, the conditional law of $Y$ given $X=x$ is
absolutely continuous with common support $[-M,M]$.  There exist
$\mu_{\operatorname{up}}>0$ and
$\mu_{\operatorname{low}}:\Xset\to(0,\infty)$ such that
\[
\begin{gathered}
  0<\low
  \le p_{Y\mid X}(y\mid x)
  \le\up<\infty,
  \qquad x\in\Xset,\quad y\in[-M,M],
  \\
  \low[\tilde x,h]
  :=\inf_{u\in B(\tilde x,h)\cap\Xset}\low[u]>0,
  \qquad \tilde x\in\Xset,\quad h>0.
\end{gathered}
\]
\end{assumP}

We impose the following H\"older condition on the conditional density.
\begin{assumP}[H\"older conditional density]
\label{assumP:density-holder}
There exist $L_p\ge0$ and $\beta\in(0,1]$ such that
\[
  \int_{\Yset}|p_{Y\mid X}(y\mid u)-p_{Y\mid X}(y\mid u')|\,\rmd y
  \le L_p\|u-u'\|^\beta,
  \qquad u,u'\in\Xset.
\]
\end{assumP}

\begin{proposition}
\label{prop:ex-residual-score}
Suppose \Cref{assumP:marginal-density,assumP:density,assumP:density-holder}
and \Cref{assumK:kernel_assum} hold.  Let
$S=S_{\mathrm{res}}$ be defined by \eqref{eq:ex-mu-def}, and set
$\shapefun_{\mathrm{res}}(r)=\indi{(0,1)}(r)$ for $r\in[0,1]$.  Then
$\shapefun_{\mathrm{res}}\in\shapefunset$ and the following statements hold.
\begin{enumerate}[label=(\roman*),leftmargin=0pt]
\item\label{item:assum:S_up}
\Aref{assum:S_up}{S_{\mathrm{res}}} holds with $\ups=2\up$.

\item\label{item:assum:lowerbound_sym}
the localized density minorization \eqref{eq:pdfloc_bound} holds for the
levels in \eqref{eq:definition-margin_hmax} and every
$0<h\le1\wedge\bwalpha$, with
\[
  \sigma_{\shapefun_{\mathrm{res}}}(u,h)=\low[u,h],
  \qquad
  \lows{u}
  \ge\sigma_{\shapefun_{\mathrm{res}}}(u,h)
  \inf_{r\in\J}\shapefun_{\mathrm{res}}(r)
  =\low[u,h].
\]

\item\label{item:assum:lip_condcdf}
\Aref{assum:lip_condcdf}{S_{\mathrm{res}}} holds with the exponent
$\beta$ of \Cref{assumP:density-holder} and $L_S=L_p(1+2M\up)$.

\item\label{ex-item-len}
\Aref{assum:S_len}{S_{\mathrm{res}}} holds with $\Llen[x]=2$ for every
$x\in\Xset$.
\end{enumerate}
\end{proposition}

\section{RLCP with a Learned Score}
\label{sec:adaptive}
\subsection{Setting and notation}
\label{sec:adaptive-setting}

\Cref{sec:fixed_score} isolates the calibration problem by treating the score as fixed.
Here the score is learned from data, which introduces a second source of
randomness and a second statistical objective.  Conditional on an independent
training fold, the fixed-score theory still applies to the learned score and
therefore still gives finite-sample marginal validity.  Efficiency is more
subtle: the conditional oracle associated with the learned score is itself
random and need not be close to a meaningful population target.

This section separates these two issues.  The RLCP algorithm---the kernel,
bandwidth, auxiliary draw, calibration weights, and conformal quantile---is
unchanged from \Cref{sec:setting}.  What changes is the efficiency analysis:
we compare the learned score with a deterministic population score, impose a
pivotality condition that removes localization bias, and propagate the
uniform training error into local length and coverage bounds.  Thus the
additional assumptions below are needed for oracle comparison, not for
finite-sample conformal validity.

Let $\Dtrain[\ntrain]\sim P_{XY}^{\otimes\ntrain}$ be independent of
$\Dcal[n]$, $(X_{n+1},Y_{n+1})$, and $\auxX$, and let a learning algorithm
output $\Shat=\Shat_{\Dtrain[\ntrain]}:\Xset\times\Yset\to\R$.
Conditionally on $\Dtrain[\ntrain]$, the score $\Shat$ is deterministic and
independent of the calibration and test data.  Hence the construction of
\Cref{sec:setting} applies verbatim with $S=\Shat$: the calibration scores are
$\Shat(X_i,Y_i)$, and the RLCP set is
$\rlcpC=\conformalset[\Shat]{x}{\rlcpq}$, as in
\eqref{eq:rlcp-prediction-set}.  Sample splitting is important here.  It lets
us condition on the training fold without altering the weighted conformal
argument; all effects of learning can therefore be confined to the efficiency
bound.

We measure the quality of $\Shat$ against a population target score
$\Sstar:\Xset\times\Yset\to\R$ that the learning procedure aims to estimate.
The useful case is when one covariate-independent threshold for $\Sstar$
already gives the desired conditional coverage.

\begin{definition}[Oracle pivotality]
\label{def:pivotality}
Let $\alpha\in\ooint{0,1}$.  The score $\Sstar$ is $(1-\alpha)$-pivotal if
there is a number $\taustar\in\R$ such that
\begin{equation}
\label{eq:pivotality}
\CondCDF[\Sstar](\taustar\mid x)=1-\alpha
\qquad \text{for }\DC[X]\text{-almost every }x\in\Xset.
\end{equation}
\end{definition}

\begin{assum}[$\alpha$]
\label{assum:S_pivot_quantile}
The target score $\Sstar$ is $(1-\alpha)$-pivotal, with pivot $\taustar$.
\end{assum}

Pivotality only requires a common threshold at which every conditional CDF of
$\Sstar(X,Y)$ attains the level $1-\alpha$; it does not require the entire
conditional distribution of $\Sstar(X,Y)$ to be invariant in $x$. This distinction allows the oracle set
to adapt its shape and location to the covariate even though its score
threshold is constant.

Pivotality is also what removes the localization bias.  Under
\Cref{assumP:marginal-density,assumK:kernel_assum},
\Aref{assum:S_up}{\Sstar},
\Aref{assum:S_pivot_quantile}{\alpha}, \Cref{lem:quant_pivo} shows that
$\conformalset[\Sstar]{x}{\taustar}$ has exact conditional coverage
$1-\alpha$ for $\DC[X]$-almost every $x$, and that the pivot is also the
localized oracle threshold at every centre--bandwidth pair satisfying the
density minorization: for every $x\in\Xset$ and $0<h\le1$ such that
$\lows[\Sstar]{x}>0$,
\begin{equation}
\label{eq:pivot-localized-quantile-main}
\locoraclequant[\Sstar]{x}=\taustar.
\end{equation}
Indeed, averaging
conditional score distributions that all cross level $1-\alpha$ at the same
threshold leaves that threshold unchanged.  The lower-density condition makes
the crossing locally invertible.

This is the first major difference from \Cref{sec:fixed_score}.  There, localization
replaces the conditional oracle threshold at $x$ by an average over nearby
covariates, and H\"older continuity controls the resulting $h^\beta$ bias.  In
the pivotal setting the target threshold does not move with either the centre
or the bandwidth, so no such bias occurs.  The adaptive benchmark is
$\conformalset[\Sstar]{x}{\taustar}$; it remains covariate-adaptive through
$\Sstar(x,\cdot)$ even though $\taustar$ is constant.  Without pivotality, one
could still condition on the training fold and invoke \Cref{sec:fixed_score} for $\Shat$,
but the benchmark would be a random learned-score oracle and the localization
bias would remain.
Accordingly, the learned-score theorems do not require the H\"older condition
of \Cref{sec:fixed_score} to control target localization; smoothness needed to construct
$\Shat$ is instead reflected in the rate $\epsilon^{\Delta}_{\infty}$ of \Cref{assum:score_estimation_rate} below.
The training fold enters through the pointwise score error
\begin{equation}
\label{eq:score-error-pointwise}
\DeltaS(x):=\sup_{y\in\Yset}\abs{\Shat(x,y)-\Sstar(x,y)}\in[0,+\infty].
\end{equation}
Throughout, $\DeltaS$ is assumed measurable; this holds in particular
whenever $\Shat$ and $\Sstar$ are continuous in $y$---as in all examples
below---since the supremum may then be restricted to a countable dense
subset of $\Yset$. For a probability measure $\Lambda$ on
$\Xset$, define
\begin{equation}
\label{eq:score-error-localized}
\|\DeltaS\|_{L^1(\Lambda)}
=
\left\{\int_{\Xset} \DeltaS(u) \,\Lambda(\rmd u) \right\}.
\end{equation}

We encode the learning step through a high-probability uniform rate; all
randomness in the following assumption is due to $\Dtrain[\ntrain]$.

\begin{assum}[$S$]
\label{assum:score_estimation_rate}
There exists a rate function
$\epsilon^{\Delta}_{\infty}:\N\times\ooint{0,1}\to[0,\infty)$ such that, for
every $\delta\in\ooint{0,1}$:
(i) $\ntrain\mapsto\epsilon^{\Delta}_{\infty}(\ntrain,\delta)$ is
non-increasing and tends to zero as $\ntrain\to\infty$;
(ii) with probability at least $1-\delta$ over $\Dtrain[\ntrain]$,
\begin{equation}
\label{eq:score-error-uniform}
\esssup_{x\in\Xset}\DeltaS(x)
\le
\epsilon^{\Delta}_{\infty}(\ntrain,\delta),
\end{equation}
the essential supremum being taken with respect to $P_X$.
\end{assum}

The examples in
\Cref{subsec:ex-cqr-score,subsec:ex-pit-score} show how standard rates for
quantile and conditional-distribution estimators imply the required bound.

\subsection{Main results}
\label{sec:adaptive-main-results}

The two theorems below are the learned-score counterparts of
\Cref{thm:rlcp_length,thm:rlcp_coverage}: on a single event of probability at
least $1-\delta$, the RLCP set built on $\Shat$ matches the pivotal oracle
$\conformalset[\Sstar]{\cdot}{\taustar}$ in length and in conditional
coverage, uniformly over the localization ball.  The bounds keep the
contributions of the two folds separate: calibration enters through
the common rate $A_{n,h}^{\mathrm{cal}}(\delta)$ defined in
\eqref{eq:th-bound-definition}, while training enters through the rate
$\epsilon^{\Delta}_{\infty}$ of
\Aref{assum:score_estimation_rate}{S}.  No $h^\beta$ term appears here because
pivotality removes the localization bias.

\begin{theorem}[Uniform local length control]
\label{thm:adaptive_rlcp_uniform_length}
Let $\alpha, \delta \in\ooint{0,1}$, $\alphalow\in\ooint{0,\alpha}$,
$\alphaup\in\ooint{\alpha,1}$, and $\boldsymbol{\alpha}=(\alpha,\alphalow,\alphaup)$. Let $\tilde{x}\in\Xset$, and $0<h\le1$.
Suppose \Cref{assumP:marginal-density,assumK:kernel_assum},
\Aref{assum:S_up}{\Sstar},
\Aref{assum:S_len}{\Sstar}, \Aref{assum:S_pivot_quantile}{\alpha} and
\Aref{assum:score_estimation_rate}{S} hold.
Let $n\ge n_1^\star(h,\delta/2,\boldsymbol{\alpha})$ and
$\ntrain\ge\ntrain^{\Delta}(h,\delta,\boldsymbol{\alpha})$, where $n_1^\star$
and $\ntrain^{\Delta}$ are defined in \eqref{eq:definition-n1star} and
\eqref{eq:definition-ntrain-delta}, respectively.  Then on the
event $\SetAdapt$, defined in \eqref{eq:omega_adapt-definition}, which has probability at least $1 - \delta$
over the calibration and training samples, the following
holds simultaneously for every $x\in B(\tilde x,h)\cap\Xset$:
\begin{equation}
\label{eq:adaptive_rlcp_length_bound_random_main}
\left||\rlcpC|-|\conformalset[\Sstar]{x}{\taustar}|\right|
\lesssim
\Llen[x] \left( \left(1/\lows[\Sstar]{\tilde x}\right)\,
\left( A_{n,h}^{\mathrm{cal}}(\delta/2)
+ \epsilon^{\Delta}_{\infty}(\ntrain,\delta/2) \right)
 + \DeltaS(x) \right).
\end{equation}
\end{theorem}

\begin{theorem}[Uniform local conditional coverage control]
\label{thm:adaptive_rlcp_uniform_coverage}
Under the assumptions and notation of
\Cref{thm:adaptive_rlcp_uniform_length}, except that
\Aref{assum:S_len}{\Sstar} is not required, on the event $\SetAdapt$ it holds
that for $\DC[X]$-almost every $x\in B(\tilde x,h)\cap\Xset$
\begin{equation}
\label{eq:adaptive_rlcp_uniform_coverage_simple}
\left|
  \int_{\Yset}\indi{\rlcpC}(y)P_{Y\mid X}(\rmd y\mid x)
  -(1-\alpha)
\right|
\lesssim
 \left(1/\lows[\Sstar]{\tilde x}\right)\,
\left( A_{n,h}^{\mathrm{cal}}(\delta/2)
+ \epsilon^{\Delta}_{\infty}(\ntrain,\delta/2)\right)
 + \DeltaS(x).
\end{equation}
\end{theorem}

Under \Aref{assum:score_estimation_rate}{S}, the pointwise term $\DeltaS(x)$
in \eqref{eq:adaptive_rlcp_length_bound_random_main} and
\eqref{eq:adaptive_rlcp_uniform_coverage_simple} admits the same deterministic
rate: on $\SetAdapt$, the uniform bound \eqref{eq:score-error-uniform} gives
$\DeltaS(x)\le\epsilon^{\Delta}_{\infty}(\ntrain,\delta/2)$ for $P_X$-almost
every $x$.  Hence, for $P_X$-almost every $x\in B(\tilde x,h)\cap\Xset$, both
bounds hold with right-hand side
\begin{equation}
\label{eq:adaptive-ae-bound}
\left(1/\lows[\Sstar]{\tilde x}\right)\,
A_{n,h}^{\mathrm{cal}}(\delta/2)
+ \left(1 + 1/\lows[\Sstar]{\tilde x}\right)
\epsilon^{\Delta}_{\infty}(\ntrain,\delta/2),
\end{equation}
up to the factor $\Llen[x]$ in the length bound.  

\runinhead{Proof structure and pointwise oracle inequalities.}
The event $\SetAdapt$ intersects the calibration event for the learned score
with the training event on which the uniform bound
\eqref{eq:score-error-uniform} holds at level $\delta/2$.
Conditioning on the independent training fold permits the fixed-score
calibration bound to be applied to $\Shat$, and a union bound gives
$\PP(\SetAdapt)\ge1-\delta$.  On this event, pivotality gives
$\locoraclequant[\Sstar]{\tilde x}=\taustar$, so no localization-bias term
appears.  Calibration controls the empirical learned-score CDF, while
\Cref{lem:kolmo_s_sstar} and
$\Loneloc{\DeltaS}{\tilde x,h}\le\esssup_{x\in\Xset}\DeltaS(x)$---valid
because $\locmeasure[\tilde x,h]\ll\DC[X]$---control its
distance from the localized target-score CDF.  Quantile stability then gives,
simultaneously for every $x\in\Xset$,
\begin{equation*}
  |\rlcpq-\taustar|
  \le \Delta^{\mathrm{adapt}}_{n,\ntrain,h}(\tilde x;\delta)
  \lesssim
  1/\lows[\Sstar]{\tilde x} \left( A_{n,h}^{\mathrm{cal}}(\delta/2)
  +\epsilon^{\Delta}_{\infty}(\ntrain,\delta/2) \right),
\end{equation*}
where the exact nonasymptotic expression for
$\Delta^{\mathrm{adapt}}_{n,\ntrain,h}$ is given in \eqref{eq:length_delta};
this is formalized in \Cref{prop:adaptive_rlcp_quantile}.

This quantile control uses the training error only through its localized
average $\Loneloc{\DeltaS}{\tilde x,h}$, which compares the two localized
score distributions.  A genuinely pointwise training error appears when the
score-defined sets themselves are compared.  Indeed, applying the two
sandwich bounds of
\Cref{lem:sandwich_length} with $q=\rlcpq$ and $r=\taustar$ gives
\[
\begin{aligned}
  \left|\,|\rlcpC[x][\tilde x]|-
  |\conformalset[\Sstar]{x}{\taustar}|\,\right|
  &\le \Llen[x]\bigl[
  \Delta^{\mathrm{adapt}}_{n,\ntrain,h}(\tilde x;\delta)+\DeltaS(x)
  \bigr],
  \\
  \left|\PP\{Y\in\rlcpC[x][\tilde x]\mid X=x\}-(1-\alpha)\right|
  &\le \ups\bigl[
  \Delta^{\mathrm{adapt}}_{n,\ntrain,h}(\tilde x;\delta)+\DeltaS(x)
  \bigr],
\end{aligned}
\]
the first inequality holding for every $x$, and the second for
$\DC[X]$-almost every $x$ because pivotality is assumed only
$\DC[X]$-almost everywhere.  These statements are proved as
\eqref{eq:adaptive-length-pointwise} and
\eqref{eq:adaptive-coverage-pointwise} in the appendix.

These pointwise oracle inequalities are more primitive than
\Cref{thm:adaptive_rlcp_uniform_length,thm:adaptive_rlcp_uniform_coverage}, which keep
the direct perturbation $\DeltaS(x)$ explicit and, in the length case,
therefore hold for every $x$.  Bounding this last term through
\eqref{eq:score-error-uniform} costs only the $P_X$-almost-everywhere
qualifier and yields the fully deterministic rate
\eqref{eq:adaptive-ae-bound}.

The two theorem bounds have the same decomposition.  The first term is the
calibration-fluctuation term inherited from \Cref{sec:fixed_score}, with the
localization bias removed by pivotality, and is amplified by the inverse
localized score density, as expected for quantile inversion.  The second term
is the training error.  Up to structural constants, it enters linearly:
improving the score estimator immediately improves the oracle comparison at
the same rate.  

Pivotality changes the role of the bandwidth.  By
\eqref{eq:pivot-localized-quantile-main}, enlarging the localization
neighbourhood does not bias the target quantile, and the uniform training
rate \eqref{eq:score-error-uniform} does not depend on $h$.  The bandwidth
therefore enters only through the effective calibration size $n\pmin h^d$ in
$A_{n,h}^{\mathrm{cal}}(\delta/2)$.  This is different from the
calibration--localization trade-off of \Cref{sec:fixed_score}: under a pivotal target,
localization itself adds little.  Since the bandwidth enters the bounds only
through the effective calibration size $n\pmin h^d$, the right-hand sides
are minimized by taking $h$ as large as the localized density minorization
$\lows[\Sstar]{\tilde x}>0$ permits, and large-bandwidth RLCP calibration of
$\Shat$ approaches marginal split-conformal calibration, whose threshold
estimates the same pivot $\taustar$ from the full calibration sample.
Covariate adaptivity is then supplied by the learned score
$\Shat(x,\cdot)$ rather than by local calibration.  Localization regains its
value precisely when the target is not pivotal, or when $\Shat$ remains far
from $\Sstar$ at moderate $\ntrain$: the learned score is then only
approximately pivotal, its conditional quantiles retain covariate
dependence, and the fixed-score analysis of \Cref{sec:fixed_score}, with its
calibration--localization trade-off, remains relevant.

\subsection{Example: conformalized quantile regression}
\label{subsec:ex-cqr-score}

Conformalized quantile regression (CQR) is the canonical example of a learned
pivotal score.   As in the residual-score example of
\Cref{subsec:ex-residual-score}, verification rests on regularity of the
conditional response law.  Here we use a uniform strengthening of
\Cref{assumP:density}, with a common lower density bound $\uniflow$.
\begingroup
\setcounter{assumP}{1}
\renewcommand{\theassumP}{2\ensuremath{'}}
\def\theHassumP{2prime}
\begin{assumP}[Conditional density]
\label{assumP:density-unif}
For every $x\in\Xset$, the conditional law $\DC[Y\mid X=x]$ is absolutely
continuous with common support $\Yset=[-M,M]$.  There exist
$\up>0$ and $\uniflow>0$ such that
\[
    0<\uniflow
    \le
    p_{Y\mid X}(y\mid x)
    \le
    \up <\infty,
    \qquad
    x\in\Xset,\ y\in[-M,M].
\]
\end{assumP}
\setcounter{assumP}{3}
\endgroup

Let $0<\alo<\ahi<1$ satisfy $\ahi-\alo=1-\alpha$, and define the CQR score \citet{romano2019conformalized}:
\begin{equation}
\label{eq:ex-cqr-def}
\cqrquant{\tau}(x)
:=
\inf\{y\in\R:F_{Y\mid X}(y\mid x)\ge\tau\},
\qquad
\Sstar_{\mathrm{cqr}}(x,y)
:=
\max\{\cqrquant{\alo}(x)-y,\;y-\cqrquant{\ahi}(x)\}.
\end{equation}
At threshold zero, the score-defined set is
$[\cqrquant{\alo}(x),\cqrquant{\ahi}(x)]$, which has conditional coverage
$1-\alpha$.  

\begin{proposition}
\label{prop:ex-cqr}
Assume \Cref{assumP:marginal-density}, \Cref{assumP:density-unif},
\Cref{assumP:density-holder}, and \Cref{assumK:kernel_assum}.  Let
$\Yset=[-M,M]$, $\alpha\in\ooint{0,1}$, and let
$\Sstar=\Sstar_{\mathrm{cqr}}$ be defined in \eqref{eq:ex-cqr-def}.  Then:
\begin{enumerate}[label=(\roman*),leftmargin=0pt]
\item\label{item:ex2-item-a1}
\Aref{assum:S_up}{\Sstar_{\mathrm{cqr}}} holds with
$\ups=2\,\up$.

\item\label{item:ex2-item-a2}
\Aref{assum:lip_condcdf}{\Sstar_{\mathrm{cqr}}} holds with the exponent
$\beta$ of \Cref{assumP:density-holder} and
$L_S=L_p\bigl(1+2\up/\uniflow\bigr)$.  Moreover, for the levels
$\alphalow,\alphaup$ in \eqref{eq:definition-margin_hmax} and every
$0<h\le1\wedge\bwalpha$, with $\bwalpha$ given by
\eqref{eq:definition-hmax_window} for this $L_S$, the localized density
minorization \eqref{eq:pdfloc_bound} holds with
$\lows[\Sstar_{\mathrm{cqr}}]{u}\ge\uniflow$ for every $u\in\Xset$.

\item\label{item:ex2-item-a4}
\Aref{assum:S_len}{\Sstar_{\mathrm{cqr}}} holds with $\Llen[x]=2$ for every
$x\in\Xset$.

\item\label{item:ex2-item-a5}
\Aref{assum:S_pivot_quantile}{\alpha} holds with pivot $\taustar=0$.
\end{enumerate}
\end{proposition}
If the learned CQR score is
\begin{equation}
\label{eq:learned-cqr-def}
\Shat(x,y)=\max\{\fhat[\ntrain]{\alo}(x)-y,
                 y-\fhat[\ntrain]{\ahi}(x)\},
\end{equation}
then $\DeltaS(u)
\le
\max_{\tau\in\{\alo,\ahi\}}
\abs{\fhat[\ntrain]{\tau}(u)-\cqrquant{\tau}(u)}$. The maximum operation is Lipschitz in its two fitted endpoints, which explains
why the score error is no larger than the worse of the two quantile-estimation
errors.  Taking the essential supremum over $\Xset$ shows that any
high-probability sup-norm guarantee for those estimators implies
\Aref{assum:score_estimation_rate}{S}.

For example, suppose $\fhat[\ntrain]{\tau}$ is a local-polynomial quantile
estimator of order $r\ge\lfloor\beta\rfloor$ with bandwidth $b_{\ntrain}$.
When $\Xset$ is compact and the standard smoothness and density conditions
of the cited works hold uniformly over $\Xset$---the essential supremum in
\eqref{eq:score-error-uniform} charges every neighbourhood in $\Xset$, since
$p_X\ge\pmin$---a uniform Bahadur expansion gives
\citep{chaudhuri1991nonparametric,guerre2012uniform,sabbah2014uniform}
\[
\epsilon^{\Delta}_{\infty}(\ntrain,\delta)
=
C_{\operatorname{LP}}
\left(
 b_{\ntrain}^{\beta}
 +
 \sqrt{\frac{\log(c\ntrain/\delta)}{\ntrain b_{\ntrain}^d}}
\right),
\]
up to taking the non-increasing envelope in $\ntrain$.  The usual choice
$b_{\ntrain}\asymp(\log(c\ntrain/\delta)/\ntrain)^{1/(2\beta+d)}$ yields
\[
\epsilon^{\Delta}_{\infty}(\ntrain,\delta)
\lesssim
\left(\frac{\log(c\ntrain/\delta)}{\ntrain}\right)^{\beta/(2\beta+d)}.
\]
Substitution into the almost-everywhere bound \eqref{eq:adaptive-ae-bound}
yields the sum of this training rate and the calibration rate
$A_{n,h}^{\mathrm{cal}}(\delta/2)$ of \eqref{eq:th-bound-definition}.  The two
bandwidths have
different roles: $b_{\ntrain}$ smooths the quantile estimator and creates its
usual estimation bias, whereas $h$ localizes conformal calibration and, under
pivotality, creates no target-quantile bias.

\runinhead{A relaxed variant.}
The common support imposed by \Cref{assumP:density-unif} can be relaxed:
the appendix allows $\Yset=\R$
and requires the density floor $\uniflow$ only on fixed
$\rhofloor$-neighbourhoods of the two target quantiles, admitting Gaussian
and, more generally, location--scale conditional laws with uniform
constants.  Under this relaxation, items \ref{item:ex2-item-a1},
\ref{item:ex2-item-a4}, and \ref{item:ex2-item-a5} persist unchanged, and
the minorization \eqref{eq:pdfloc_bound} holds with
$\lows[\Sstar_{\mathrm{cqr}}]{u}\ge2\uniflow$ for every $u\in\Xset$ and
every $0<h\le1$---without the H\"older input of item
\ref{item:ex2-item-a2} or the bandwidth cap $\bwalpha$---whenever
$(\alpha-\alphalow)\vee(\alphaup-\alpha)<2\uniflow\,\rstar$ with
$\rstar:=\rhofloor\wedge\frac{1-\alpha}{2\up}$.

\subsection{Example: the distributional score}
\label{subsec:ex-pit-score}

A second pivotal target is the probability-integral-transform score of
\citet{chernozhukov2021distributional}:
\begin{equation}
\label{eq:ex-pit-def}
\Sstar(x,y):=\big|F_{Y\mid X}(y\mid x)- 1/2\big|.
\end{equation}
Under continuity, $F_{Y\mid X}(Y\mid x)$ is uniform conditionally on $X=x$.
Therefore its absolute deviation from $1/2$ has the same distribution for
every covariate, making pivotality immediate.  The oracle band
\[
\conformalset[\Sstar]{x}{(1 - \alpha)/2}=\bigl[\,q_{\alpha/2}(x),\;q_{1-\alpha/2}(x)\,\bigr]
\]
is the central inter-quantile interval.  Moreover, the reverse triangle
inequality bounds $\DeltaS(x)$ by the uniform-in-$y$ error of the fitted
$\widehat F_{Y\mid X}(\cdot\mid x)$.  Thus any sup-norm rate for the
conditional distribution estimator transfers directly to
\Aref{assum:score_estimation_rate}{S}.

\begin{proposition}
\label{prop:ex-pit-standard}
Assume \Cref{assumP:density,assumP:marginal-density} and
\Cref{assumK:kernel_assum}.  Let $\Yset=[-M,M]$,
$\alpha\in\ooint{0,1}$, and
let $\Sstar$ be the distributional score defined in \eqref{eq:ex-pit-def}.
Then:
\begin{enumerate}[label=(\roman*),leftmargin=0pt]
\item\label{item:ex3-item-a1}
\Aref{assum:S_up}{\Sstar} holds with
$\ups=2$.

\item\label{item:ex3-item-a2}
For every $0<\alphalow<\alpha<\alphaup<1$, the localized density
minorization holds with $\lows[\Sstar]{u}=2$ for every
$u\in\Xset$ and $0<h\le1$.

\item\label{item:ex3-item-a4}
\Aref{assum:S_len}{\Sstar} holds with $\Llen[x]=2/\low$ for every
$x\in\Xset$.

\item\label{item:ex3-item-a5}
\Aref{assum:S_pivot_quantile}{\alpha} holds with pivot
$\taustar=(1-\alpha)/2$.
\end{enumerate}
\end{proposition}

The CQR and distributional scores offer complementary routes to adaptivity.
CQR estimates only the two quantiles needed for a chosen miscoverage level,
whereas the distributional score estimates the full conditional distribution
and can therefore be reused across levels.  The latter target is exactly
pivotal with a particularly simple score density, but its learning assumption
requires uniform control in the response variable.
If the PIT learned score is
\[
\Shat(x,y):=\left|\widehat F_{Y\mid X}(y\mid x)-\frac12\right|,
\]
then the reverse triangle inequality applied to \eqref{eq:ex-pit-def} gives
\[
\DeltaS(x)
\le
\sup_{y \in \Yset}
\big|F_{Y\mid X}(y\mid x)-\widehat F_{Y\mid X}(y\mid x)\big|.
\]
For example, let $\widehat F_{Y\mid X}$ be the local-constant kernel
conditional CDF estimator
\[
\widehat F_{Y\mid X}(y\mid x)
=
\frac{\sum_{(X',Y')\in\Dtrain[\ntrain]}K((X'-x)/b_{\ntrain})\,
\indiacc{Y'\le y}}
{\sum_{(X',Y')\in\Dtrain[\ntrain]}K((X'-x)/b_{\ntrain})} ,
\]
where $K$ satisfies the conditions of \Cref{assumK:kernel_assum} and
generates a pointwise measurable VC-type class, with the convention
$\widehat F_{Y\mid X}(\cdot\mid x)\equiv1/2$ when the denominator vanishes.
Suppose, in addition to \Cref{assumP:marginal-density}, that
$p_X\le p_{\max}$, and take $b_{\ntrain}\le1$.  Then the denominator has
expectation of order $\ntrain b_{\ntrain}^d$ uniformly over $\Xset$,
including at the boundary, and is uniformly positive with high probability
whenever $\ntrain b_{\ntrain}^d\gtrsim \ell_{b_{\ntrain}}(\delta)$, where
\[
\ell_b(\delta):=V\log(A/b)+\log(1/\delta),
\]
with $(A,V)$ the VC characteristics of the induced kernel--threshold class.
If moreover $x\mapsto F_{Y\mid X}(y\mid x)$ is uniformly $\beta$-H\"older,
$0<\beta\le1$, then the population smoothed target is a convex combination
of conditional CDFs at covariates within distance $b_{\ntrain}$ of $x$, so
the bias is $O(b_{\ntrain}^{\beta})$ uniformly over $\Xset$.  Combining
this bias bound with finite-sample VC/Rademacher moment bounds
\citep[Prop.~1 and Cor.~4]{einmahl2005uniform} and Bousquet's inequality
\citep[Thm.~12.5]{boucheron2013concentration} yields, for a constant $C$
depending only on the structural constants, with probability at least
$1-\delta$,
\[
\sup_{x\in\Xset}\sup_{y\in\Yset}
\big|\widehat F_{Y\mid X}(y\mid x)-F_{Y\mid X}(y\mid x)\big|
\le
\epsilon^{\Delta}_{\infty}(\ntrain,\delta),
\]
where
\[
\epsilon^{\Delta}_{\infty}(\ntrain,\delta)
=
1\wedge
C\left(
b_{\ntrain}^\beta+
\sqrt{\frac{\ell_{b_{\ntrain}}(\delta)}{\ntrain b_{\ntrain}^d}}
+\frac{\ell_{b_{\ntrain}}(\delta)}{\ntrain b_{\ntrain}^d}
\right),
\]
up to taking the non-increasing envelope in $\ntrain$.  For the usual
bandwidths, $\ell_{b_{\ntrain}}(\delta)\lesssim \log(c\ntrain/\delta)$;
hence, whenever $\ntrain b_{\ntrain}^d\gtrsim \log(c\ntrain/\delta)$, the
last term is lower order.  The choice
$b_{\ntrain}\asymp(\log(c\ntrain/\delta)/\ntrain)^{1/(2\beta+d)}$, clipped
if necessary to satisfy $b_{\ntrain}\le1$, gives
\[
\epsilon^{\Delta}_{\infty}(\ntrain,\delta)
\lesssim
\left(\frac{\log(c\ntrain/\delta)}{\ntrain}\right)^{\beta/(2\beta+d)} .
\]
The supremum over $y$ costs no additional nonparametric dimension: the
threshold class $\{\indiacc{\cdot\le y}:y\in\R\}$ only changes the VC
constants.  The almost-sure uniform-in-bandwidth theorem of
\citet[Thm.~3]{einmahl2005uniform} gives the corresponding asymptotic
stochastic rate for this conditional empirical distribution estimator on
compact interior sets, that is, when $p_X$ is continuous and positive on a
neighbourhood of the uniformity set.

\section{Numerical Experiments}
\label{sec:experiments}
We illustrate the fixed-score theory with two synthetic diagnostics.  They are
not comparative benchmarks, and they do not attempt a numerical validation of
the constants in \cref{thm:rlcp_length,thm:rlcp_coverage}: each is an
illustration of a predicted rate.  They assess whether the calibration rate in
\eqref{eq:th-bound-definition} and the localization bias $h^\beta$ describe the
finite-sample behaviour of the localized threshold and the associated oracle
errors.  The conditional score quantiles are explicit, and each coverage error
is evaluated from the conditional distribution rather than from an auxiliary
test sample.  The main text presents the oracle-tracking diagnostic; the
bandwidth bias--variance study, the learned-score diagnostics, and the real-data
comparisons are reported in the appendix.

We take $d=1$, $X\sim\mathrm{Unif}[0,1]$, the fixed response envelope
$\Yset=[-4,4]$, and
\begin{equation}
  \label{eq:exp-dgp-bj}
  Y=|X-x_0|^{\beta}+\eps,
  \qquad x_0=1/2,
  \qquad \eps\sim\mathrm{Unif}[-B,B],\quad B=3,
\end{equation}
with one-sided score $S(x,y)=y$, level $\alpha=0.1$, and
$\beta\in\{0.5,0.75,1\}$.  (The data-generating support is contained in
$[-3,4]\subset\Yset$.)  The regression function is globally $\beta$-H\"older
and has pointwise exponent $\beta$ at $x_0$.  Since the noise distribution
function is $(2B)^{-1}$-Lipschitz, \Aref{assum:lip_condcdf}{S} holds with
exponent $\beta$.  The conditional score density equals $(2B)^{-1}$ on the
relevant quantile window, and
$\oraclequant{x_0}=B(1-2\alpha)=2.4$.  We use the Epanechnikov kernel, which
satisfies \cref{assumK:kernel_assum}.  In each replication we draw
$\tilde x=x_0+hV$ with $V\sim K$, evaluate the conditional error at $x=x_0$ for
the realized auxiliary centre, and average over both the calibration sample and
the auxiliary-centre randomization.  This is the interior
geometry of \cref{thm:rlcp_length,thm:rlcp_coverage}; the simulations are rate
diagnostics, not checks of those theorems' nonasymptotic bandwidth and
sample-size conditions or of the unspecified constants hidden by $\lesssim$.
Each reported mean in the oracle-tracking diagnostic is based on $R=300$
independent calibration samples; the bandwidth study reported in the
The appendix uses $R=1000$.

Both diagnostics are \emph{formal}-RLCP experiments run on the fixed response
envelope $\Yset=[-M,M]$ with $M=B+1=4$, chosen strictly larger than the noise
support $[-B,B]$ so that an infinite localized threshold produces a finite
set-length error.  When the localized calibration weight is insufficient---an
empty window, or $W_{\mathrm{cal}}<\tfrac{1-\alpha}{\alpha}W_{\mathrm{test}}$
(that is, $W_{\mathrm{cal}}<9\,W_{\mathrm{test}}$ at $\alpha=0.1$)---the threshold is
$+\infty$ and the set is the full envelope $[-4,4]$; such thresholds are kept,
never discarded, conditioned out, or capped at the support edge.  The one-sided
oracle set is $[-4,2.4]$, of length $6.4$.

For fixed confidence level, the leading calibration term in
\eqref{eq:th-bound-definition} is $(nh^d)^{-1/2}$; balancing it against the
localization bias $h^\beta$ gives $h^\star(n)\asymp n^{-1/(2\beta+d)}$.  An
empirical bandwidth study confirming this balance is deferred to the
appendix.  Here we set the balanced bandwidth
$h=n^{-1/(2\beta+d)}$ and examine how the realized RLCP set tracks the
conditional score oracle as $n$ grows.

At this bandwidth the high-probability bounds in
\cref{thm:rlcp_length,thm:rlcp_coverage} have leading order
$n^{-\beta/(2\beta+d)}$, up to logarithmic and lower-order terms.  On the
reported grid the formal $+\infty$ event never occurs---its frequency is $0$ for
every $(\beta,n)$---so the shared envelope is a disclosure and a parity with the
bandwidth study, not a change to the data.  Least-squares log--log fits to the
empirical means (\cref{fig:exp-tracking}) give slopes $-0.285$, $-0.320$ and
$-0.325$ (standard errors $0.009$--$0.010$) for the conditional-coverage gap and
$-0.286$, $-0.320$ and $-0.325$ for the absolute set-length error, at
$\beta=0.5,0.75,1$,
against the leading exponents $-0.250$, $-0.300$ and $-0.333$.  The estimates are
ordered by $\beta$ and close to the leading exponents; the steeper values at
$\beta=0.5$ and $0.75$ are consistent with the finite-sample second-order term.
Coverage and length share the same fitted slopes because, while the threshold
remains in the uniform-noise support, the absolute conditional-coverage error
is exactly $(2B)^{-1}$ times the absolute threshold error and the absolute
one-sided set-length error equals the absolute threshold error, so the two
panels differ only by this deterministic transformation.  These finite-grid
slopes illustrate the predicted oracle rate
$-\beta/(2\beta+d)$; they do not trace the multiplicative constants of the
bounds.

\begin{figure}[t]
  \centering
  \includegraphics[width=\linewidth]{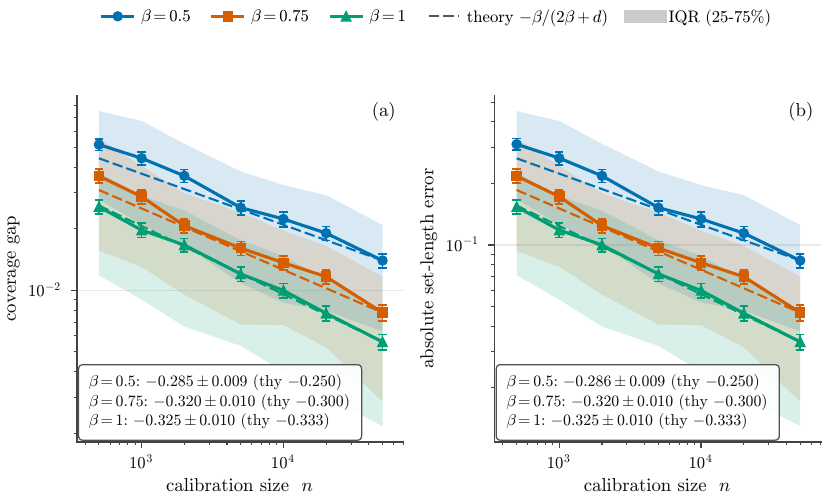}
  \caption{Oracle tracking at the balanced bandwidth $h=n^{-1/(2\beta+d)}$ in
  the fixed-score cusp model ($d=1$, $\alpha=0.1$, $R=300$).  (a) absolute
  conditional-coverage error at $x_0$; (b) absolute set-length error relative
  to the conditional score oracle, both against the calibration size $n$.  Infinite
  localized thresholds are handled formally on the shared envelope $M=4$, and
  the $+\infty$ frequency is $0$ across the grid, so the slopes are unaffected.
  Markers are means over $R=300$ replications; error bars are $\pm2$ estimated
  standard errors, and shaded bands are empirical interquartile ranges.  Dashed
  lines have slopes $-\beta/(2\beta+d)$; fitted slopes are annotated in each
  panel.}
  \label{fig:exp-tracking}
\end{figure}

The appendix gives the full protocols for these fixed-score
diagnostics, including the bandwidth bias--variance study of the formal
set-length and coverage errors and its fitted $h^\star(n)$ balance.  It also
reports a learned-score study built on the worked examples of
\cref{prop:ex-cqr,prop:ex-pit-standard}, measuring the uniform
score-estimation error $\epsilon^{\Delta}_{\infty}$ that drives
\cref{thm:adaptive_rlcp_uniform_length,thm:adaptive_rlcp_uniform_coverage} and
displaying the empirical decomposition of the measured oracle errors into
calibration and score-estimation terms, together with real-data comparisons
against global \textsf{Split-CP} and \textsf{Split-CQR} that foreground the
coverage--length trade-off.

\begin{acks}[Acknowledgments]
The authors used OpenAI Codex to assist with language editing, consistency
checks, and \LaTeX{} formatting. All AI-assisted suggestions were reviewed and
verified by the authors, who take full responsibility for the content of the
manuscript.
\end{acks}

\clearpage
\appendix

\section{Proofs for RLCP with a Fixed Score}
\label{sec:proofs_agnostic}
This section assembles the concentration toolbox underlying the fixed-score analysis: we first introduce the sample-size threshold, weight bound and deviation rate that appear in all subsequent statements, and then establish the calibration event on which the localized empirical CDF concentrates around its population counterpart. For $h\in(0,1]$ and $\delta\in(0,1)$, set
\begin{equation}
\label{eq:definition_n0}
n_0(h,\delta)
  := \frac{64\supnorm{\KDE}}
       {c_0\,\pmin\,h^d}\log\!\left(\frac{4}{\delta}\right),
\end{equation}
and
\begin{equation}
\label{eq:definition-rhom}
  \rhom
  :=
  \frac{4\supnorm{\KDE}}
       {c_0\,n\,\pmin\,h^d}.
\end{equation}
Define
\begin{equation}
\label{eq:definition-varepsilon}
  \varepsilon_{n,h}(\delta)
  :=
  16
  \sqrt{
    \frac{2\supnorm{\KDE}}
         {c_0\,n\,\pmin\,h^d}
  }
  +
  8
  \sqrt{
    \frac{\supnorm{\KDE}\log(4/\delta)}
         {c_0\,n\,\pmin\,h^d}
  }
  +
  \frac{8\supnorm{\KDE}\log(4/\delta)}
       {c_0\,n\,\pmin\,h^d}.
\end{equation}

Whenever the denominator is positive, define the calibration-only normalized
weights
\begin{equation}
\label{eq:bar-weights}
    \bar w_{\tilde{x},h}(X_i)
    :=
    \frac{\locweight[\tilde{x}]{X_i}}
         {\sum_{j\in[n]}\locweight[\tilde{x}]{X_j}},
    \qquad i\in[n],
\end{equation}
and the corresponding localized empirical CDF
\begin{equation}
\label{eq:bar-empirical-cdf}
    \nBarCDFloc{\tilde{x}}(t)
    :=
    \sum_{i\in[n]}
        \bar w_{\tilde{x},h}(X_i)\,\indiacc{S_i\le t},
    \qquad t\in\RR .
\end{equation}
For every finite $t$,
\begin{equation}
\label{eq:bar-emp-relation}
    \nEmpCDFloc{\tilde{x}}(t)
    =
    \bigl(1-\nlocweightnorm{\tilde{x}}{n+1}\bigr)
    \nBarCDFloc{\tilde{x}}(t).
\end{equation}

We begin with an elementary but crucial lower bound on the normalization $\locnorm[\tilde{x}]$, which drives every sample-size threshold in this appendix.

\begin{lemma}
\label{lem:gamma_bound_exact}
Assume \Cref{assumP:marginal-density,assumK:kernel_assum}. Then, for every
$\tilde{x}\in\Xset$ and every $h\in(0,1]$,
\begin{equation}
\label{eq:gamma_bound-corr}
  \locnorm[\tilde{x}] \;\ge\; c_0\,\pmin\,h^d \,,
\end{equation}
where $c_0$ is the structural constant introduced after
\Cref{assumK:kernel_assum}.
\end{lemma}
The next proposition is the cornerstone of the fixed-score analysis: it shows that, with high probability, the test-point weight is uniformly small and the localized empirical CDF stays uniformly close to the localized population CDF.
\begin{proposition}
\label{prop:calibration_event}
Assume \Cref{assumP:marginal-density,assumK:kernel_assum}. Let $S$ be a fixed
measurable score and write $S_i=S(X_i,Y_i)$. Fix
$\tilde{x}\in\Xset$, $h\in(0,1]$ and
$\delta\in(0,1)$. For $t\in\RR$, set
\[
 A_n(t)
 :=
 n^{-1}\sum_{i=1}^n
 \locweight[\tilde{x}]{X_i}\indiacc{S_i\le t},
 \qquad
 B_n
 :=
 n^{-1}\sum_{i=1}^n
 \locweight[\tilde{x}]{X_i},
\]
and
\[
 A(t)
 :=
 \mathbb E\!\left[
 \locweight[\tilde{x}]{X}\indiacc{S(X,Y)\le t}
 \right].
\]
Define the calibration event
\begin{equation}
\label{eq:definition-Omega}
\begin{split}
  \SetOmega[{n,\tilde{x},h,\delta}]
  :=
  &\left\{
  \left|B_n-\locnorm[\tilde{x}]\right|
  \le
  \sqrt{
    \frac{2\supnorm{\KDE}\locnorm[\tilde{x}]\log(4/\delta)}{n}
  }
  +
  \frac{2\supnorm{\KDE}\log(4/\delta)}{3n}
  \right\}
  \\
  &\cap
  \left\{
  \sup_{t\in\RR}\left|A_n(t)-A(t)\right|
  \le
  8
  \sqrt{
    \frac{\supnorm{\KDE}\locnorm[\tilde{x}]}{n}
  }
  +
  \sqrt{
    \frac{2\supnorm{\KDE}\locnorm[\tilde{x}]\log(4/\delta)}{n}
  }
  +
  \frac{4\supnorm{\KDE}\log(4/\delta)}{3n}
  \right\}.
\end{split}
\end{equation}
Then
\[
  \PP\!\left(\SetOmega[{n,\tilde{x},h,\delta}]\right)
  \ge 1-\delta .
\]
Moreover, if $n\ge n_0(h,\delta)$, then, on
$\SetOmega[{n,\tilde{x},h,\delta}]$, the denominator
$\sum_{j\in[n]}\locweight[\tilde{x}]{X_j}$ is positive and the following two
bounds hold simultaneously:
\begin{equation}
\label{eq:cal-weight-bound}
  \sup_{x\in\Xset}
  \nlocweightnorm[x]{\tilde{x}}{n+1}
  \le
  \rhom,
\end{equation}
and
\begin{equation}
\label{eq:cal-cdf-bound}
  \sup_{t\in\RR}
  \left|
    \nBarCDFloc{\tilde{x}}(t)-\CDFloc{\tilde{x}}(t)
  \right|
  \le
  \varepsilon_{n,h}(\delta).
\end{equation}

\end{proposition}

When the score is learned on an independent training fold, the proposition
is applied conditionally on that fold, the expectation defining $A(t)$
being computed under the conditional law; see
\eqref{eq:omega_adapt-definition}.

\subsection{Proof of \texorpdfstring{\Cref{thm:rlcp_length}}{the RLCP length theorem}}
\label{subsec:proofs_th1}

The proof proceeds in two steps: a quantile-stability lemma comparing the RLCP threshold with the localized oracle quantile on the calibration event, and a bias step converting this comparison into a deviation bound at the queried covariate. We first record the sample-size thresholds used below: a base
threshold $n_2$, sufficient for the coverage analysis, and a refined
threshold $n_1\ge n_2$ under which both steps of the length analysis
apply. For $h\in(0,1]$, $\delta\in(0,1)$ and $\alpha\in\ooint{0,1}$,
define
\begin{equation}
\label{eq:definition-n2}
n_2(h,\delta,\alpha)
:=
\min\Bigl\{
n\in\nsets:\;
n\ge n_0(h,\delta),\;
\varepsilon_{n,h}(\delta)+\rhom\le\tfrac{\alpha}{2}
\Bigr\}
\end{equation}
and, for $\boldsymbol\alpha=(\alpha,\alphalow,\alphaup)$ with
$0<\alphalow<\alpha<\alphaup<1$, the minimal calibration size
\begin{equation}
\label{eq:definition-n1}
n_1(h,\delta,\boldsymbol\alpha)
:=
\min\Bigl\{
n\ge n_2(h,\delta,\alpha):\;
\varepsilon_{n,h}(\delta)+2(1-\alpha)\rhom
\le m_{\boldsymbol\alpha}
\Bigr\},
\end{equation}
where $n_0$, $\rhom$ and $\varepsilon_{n,h}(\delta)$ are defined in
\eqref{eq:definition_n0}, \eqref{eq:definition-rhom} and
\eqref{eq:definition-varepsilon}, respectively, and
$m_{\boldsymbol\alpha}>0$ is the level margin defined in
\eqref{eq:definition-I}.
For fixed
$(h,\delta)$, the maps $n\mapsto\varepsilon_{n,h}(\delta)$ and
$n\mapsto\rhom$ are positive, nonincreasing and vanish as
$n\to\infty$; hence
$n_2(h,\delta,\alpha)\le n_1(h,\delta,\boldsymbol\alpha)<\infty$.
Moreover, for every $n\ge n_2(h,\delta,\alpha)$,
\begin{equation}
\label{eq:rev-n2-consequences}
n\ge n_0(h,\delta),
\qquad
\varepsilon_{n,h}(\delta)+\rhom\le\frac{\alpha}{2},
\end{equation}
and, for every $n\ge n_1(h,\delta,\boldsymbol\alpha)$, additionally,
\begin{equation}
\label{eq:rev-n1-consequences}
\varepsilon_{n,h}(\delta)+2(1-\alpha)\rhom
\le m_{\boldsymbol\alpha}.
\end{equation}

\begin{remark}
\label{rem:explicit-n1}
An explicit bound on $n_1(h,\delta,\boldsymbol\alpha)$ follows by solving
the two quadratic inequalities in $\sqrt n$ implied by
\eqref{eq:rev-n2-consequences}--\eqref{eq:rev-n1-consequences}: writing
$\varepsilon_{n,h}(\delta)+\rhom$ and
$\varepsilon_{n,h}(\delta)+2(1-\alpha)\rhom$ in the form $a/\sqrt n+b/n$,
with $a,b$ read off \eqref{eq:definition-rhom} and
\eqref{eq:definition-varepsilon}, the condition $a/\sqrt n+b/n\le t$ holds
as soon as $n\ge\bigl\{(a+\sqrt{a^2+4tb})/(2t)\bigr\}^{2}$.  In
particular,
$n_1(h,\delta,\boldsymbol\alpha)
=\OC\bigl(\log(4/\delta)/(\pmin h^d\,m_{\boldsymbol\alpha}^{2})\bigr)$:
since $m_{\boldsymbol\alpha}<\alpha$, the margin condition is, up to
absolute constants, the binding one.  Similarly,
$n_2(h,\delta,\alpha)
=\OC\bigl(\log(4/\delta)/(\pmin h^d\,\alpha^{2})\bigr)$.
\end{remark}

We also define the deviation rate
\begin{equation}
\label{eq:definition-rate}
\Delta^{\mathrm R}_{n,h}(z;\delta)
:=\frac{1}{\lows{z}}
\Bigl(\varepsilon_{n,h}(\delta)+2(1-\alpha)\rhom
+2^{\beta}L_S h^{\beta}\Bigr),
\end{equation}
which, by \eqref{eq:definition-rhom}--\eqref{eq:definition-varepsilon},
satisfies
\begin{equation}
\label{eq:definition-rate-order}
  \Delta^{\mathrm R}_{n,h}(z;\delta)
  \lesssim
  \bigl\{A_{n,h}^{\mathrm{cal}}(\delta)+h^\beta\bigr\}/\lows{z}.
\end{equation}
For $x,\tilde{x}\in\Xset$ such that
$\nlocweightnorm[x]{\tilde{x}}{n+1}<1$, the \emph{effective calibration
level}
\begin{equation}
\label{eq:effective-level}
p_{n,h}(x;\tilde{x})
:=
\frac{1-\alpha}{1-\nlocweightnorm[x]{\tilde{x}}{n+1}} .
\end{equation}

We can now state the key stability estimate for the RLCP threshold.

\begin{lemma}
\label{lem:T1}
Let $\alpha\in\ooint{0,1}$ and
$\alphalow\in\ooint{0,\alpha}$, $\alphaup\in\ooint{\alpha,1}$, and
$\boldsymbol\alpha=(\alpha,\alphalow,\alphaup)$.
Assume \Cref{assumP:marginal-density,assumK:kernel_assum}
and \Aref{assum:S_up}{S}. Let
\(\delta\in(0,1)\), \(\tilde{x}\in\Xset\), \(0<h\le1\), and suppose
that \(\lows{\tilde x}>0\) and
\(n\ge n_1(h,\delta,\boldsymbol\alpha)\), with \(n_1\) as in
\eqref{eq:definition-n1}. Then, on the event
\(\SetOmega[n,\tilde{x},h,\delta]\) of \eqref{eq:definition-Omega}, the
following holds simultaneously for every \(x\in\Xset\):
\begin{equation}
\label{eq:rev-T1}
\abs{\rlcpq-\locoraclequant{\tilde{x}}}
\le
\frac{\varepsilon_{n,h}(\delta)+2(1-\alpha)\rhom}
     {\lows{\tilde{x}}},
\end{equation}
where \(\lows{\tilde{x}}\) is defined in \eqref{eq:pdfloc_bound}. Moreover,
\(\rlcpq\in\ccint{\min_{i\in[n]}S_i,\,\max_{i\in[n]}S_i}\) is finite.
\end{lemma}

Adding the localization bias to the estimate of \Cref{lem:T1} yields the following deviation bound, from which both main theorems will follow.

\begin{proposition}[RLCP quantile--oracle deviation]
\label{prop:rlcp_quantile_oracle}
Let $\alpha\in\ooint{0,1}$ and
$\alphalow\in\ooint{0,\alpha}$, $\alphaup\in\ooint{\alpha,1}$, and
$\boldsymbol\alpha=(\alpha,\alphalow,\alphaup)$.
Assume \Cref{assumP:marginal-density,assumK:kernel_assum},
\Aref{assum:S_up}{S}, and \Aref{assum:lip_condcdf}{S}.
Let $\delta\in(0,1)$, $\tilde{x}\in\Xset$ and $h>0$. Suppose that
\[
  0<h\le1\wedge\mathrm h^{\mathrm{bias}}_{\boldsymbol\alpha},
  \qquad \lows{\tilde x}>0,
  \qquad
  n\ge n_1(h,\delta,\boldsymbol\alpha),
\]
with $n_1$ as in \eqref{eq:definition-n1}. Then
$\PP(\SetOmega[n,\tilde{x},h,\delta])\ge1-\delta$ and, on
$\SetOmega[n,\tilde{x},h,\delta]$, the following holds simultaneously
for every $x\in\Xset$ with $\norm{x-\tilde{x}}\le h$: the thresholds
$\rlcpq$ and $\oraclequant{x}$ are finite, and
\begin{equation}
\label{eq:rlcp_quantile_oracle}
  \abs{\rlcpq-\oraclequant{x}}
  \le \Delta^{\mathrm R}_{n,h}(\tilde{x};\delta),
\end{equation}
with $\Delta^{\mathrm R}_{n,h}$ as in \eqref{eq:definition-rate}.
\end{proposition}

The proof of \Cref{thm:rlcp_length} now amounts to converting the threshold deviation of \Cref{prop:rlcp_quantile_oracle} into a length deviation via \Cref{assum:S_len}.

\begin{proof}[Proof of \Cref{thm:rlcp_length}]
If $\lows{\tilde x}=0$, the right-hand side of the claimed bound is
vacuous under the convention preceding \Cref{thm:rlcp_length}.  We may
therefore assume throughout that $\lows{\tilde x}>0$.

The bandwidth assumption is precisely the one required by
\Cref{prop:rlcp_quantile_oracle}.  Hence that proposition gives,
on \(\SetOmega[n,\tilde{x},h,\delta]\),
\[
  \abs{\rlcpq-\oraclequant{x}}
  \le \Delta^{\mathrm R}_{n,h}(\tilde{x};\delta),
\]
together with the finiteness of \(\rlcpq\) and \(\oraclequant{x}\). By
\Cref{assum:S_len}, applied at the fixed covariate \(x\) with thresholds
\(\rlcpq\) and \(\oraclequant{x}\),
\[
  \Bigl|\,\abs{\conformalset{x}{\rlcpq}}
        -\abs{\conformalset{x}{\oraclequant{x}}}\,\Bigr|
  \le \Llen\,\abs{\rlcpq-\oraclequant{x}}.
\]
Using \(\rlcpC=\conformalset{x}{\rlcpq}\) from
\eqref{eq:rlcp-prediction-set} and then
\eqref{eq:definition-rate-order} proves the claimed bound.  Since
\Cref{prop:rlcp_quantile_oracle} holds simultaneously for every
$x\in B(\tilde x,h)\cap\Xset$ on the same event, so does the length bound.
\end{proof}

\subsection{Proof of \texorpdfstring{\Cref{thm:rlcp_coverage}}{the RLCP coverage theorem}}

We give a condensed proof, stating only the main steps; the complete
derivations appear in the appendix.

\begin{proof}[Proof of \Cref{thm:rlcp_coverage}]
Fix
$x\in B(\tilde x,h)\cap\Xset$ and work on
$\SetOmega[n,\tilde{x},h,\delta]$. By the definition
\eqref{eq:rlcp-prediction-set} of the prediction set,
$\int_{\Yset}\indi{\rlcpC}(y)\,P_{Y\mid X}(\rmd y\mid x)
=\CondCDF(\rlcpq\mid x)$, and the triangle inequality yields
\begin{align*}
\abs{\CondCDF\bigl(\rlcpq\,\big|\, x\bigr) - (1-\alpha)}
&\le
\abs{\CondCDF\bigl(\rlcpq\,\big|\, x\bigr) - \CDFloc{\tilde{x}}(\rlcpq)} \\
&\quad
+ \abs{\CDFloc{\tilde{x}}(\rlcpq) - p_{n,h}(x;\tilde{x})}
+ \abs{p_{n,h}(x;\tilde{x}) - (1-\alpha)},
\end{align*}
with $p_{n,h}(x;\tilde{x})$ the effective level
\eqref{eq:effective-level}, well defined since
$\nlocweightnorm[x]{\tilde{x}}{n+1}\le\rhom\le\alpha/2$ by
\eqref{eq:cal-weight-bound} and \eqref{eq:rev-n2-consequences}.

The three terms are controlled as follows.  First,
\Cref{assumK:kernel_assum} and \Aref{assum:lip_condcdf}{S} give
$\sup_{t\in\R}\abs{\CDFloc{\tilde{x}}(t)-\CondCDF(t\mid x)}
\le L_S(2h)^\beta$.  Second, since
$\nlocweightnorm[x]{\tilde{x}}{n+1}<\alpha$, relation
\eqref{eq:bar-emp-relation} identifies $\rlcpq$ as the finite quantile
$\quant{p_{n,h}(x;\tilde{x})}{\nBarCDFloc{\tilde{x}}}$; evaluating the
Kolmogorov bound \eqref{eq:cal-cdf-bound} of
\Cref{prop:calibration_event} at this quantile, and using the
continuity of $\CDFloc{\tilde{x}}$ under \Aref{assum:S_up}{S}, yields
$\abs{\CDFloc{\tilde{x}}(\rlcpq)-p_{n,h}(x;\tilde{x})}
\le\varepsilon_{n,h}(\delta)$.  Third, by \eqref{eq:effective-level},
$0\le p_{n,h}(x;\tilde{x})-(1-\alpha)\le2(1-\alpha)\rhom$.

Summing the three bounds,
\[
\abs{\int_{\Yset}\indi{\rlcpC}(y)\,P_{Y\mid X}(\rmd y\mid x)-(1-\alpha)}
\le
\varepsilon_{n,h}(\delta)+2(1-\alpha)\,\rhom+L_S(2h)^{\beta}
\lesssim
A_{n,h}^{\mathrm{cal}}(\delta)+h^{\beta},
\]
the last comparison being as in \eqref{eq:definition-rate-order}.
Neither the event $\SetOmega[n,\tilde{x},h,\delta]$ nor the right-hand
side depends on $x$, so the bound holds simultaneously for all
$x\in B(\tilde x,h)\cap\Xset$.
\end{proof}

\section{Proofs for RLCP with a Learned Score}
\label{sec:proofs_adaptive}
\subsection{Additional Technical lemmas}
The learned-score analysis requires three additional ingredients: a pivotality property identifying $\taustar$ as the localized oracle quantile of the limiting score, a perturbation bound quantifying how score estimation distorts localized CDFs, and a comparison lemma transferring threshold and score errors to coverage and length.

\begin{lemma}
    \label{lem:quant_pivo}
    Let $\alpha\in\ooint{0,1}$, $\alphalow\in\ooint{0,\alpha}$,
    $\alphaup\in\ooint{\alpha,1}$, $\boldsymbol{\alpha}=(\alpha,\alphalow,\alphaup)$,
    $x\in\Xset$, and $0<h\le1$.
    Assume \Cref{assumP:marginal-density,assumK:kernel_assum},
    \Aref{assum:S_up}{\Sstar} and
    \Aref{assum:S_pivot_quantile}{\alpha}, and suppose
    $\lows[\Sstar]{x}>0$. Then the
    localized score CDF $\CDFloc[\Sstar]{x}$ satisfies
    \begin{equation}
    \label{eq:quant_pivo_loc_cdf}
        \CDFloc[\Sstar]{x}(\taustar)=1-\alpha ,
    \end{equation}
    and the pivot is its localized $(1-\alpha)$-quantile:
    \begin{equation}
    \label{eq:quant_pivo_loc_quant}
        \locoraclequant[\Sstar]{x}
        =\quant{1-\alpha}{\CDFloc[\Sstar]{x}}
        =\taustar
        \in\I[x][\balpha][\Sstar] . 
    \end{equation}
    Moreover, for $\DC[X]$-almost every $z\in\Xset$,
    $\CondCDF[\Sstar](\taustar\mid z)=1-\alpha$; consequently
    $\oraclequant[\Sstar]{z}\le\taustar$, and both oracle sets
    $\conformalset[\Sstar]{z}{\oraclequant[\Sstar]{z}}$ and
    $\conformalset[\Sstar]{z}{\taustar}$ have conditional coverage
    $1-\alpha$.
\end{lemma}

\Cref{lem:quant_pivo} concerns the limiting score $\Sstar$; the next lemma quantifies the price of replacing $\Sstar$ by its estimate $\Shat$ at the level of localized CDFs.

\begin{lemma}[Localized Kolmogorov perturbation by score estimation]
\label{lem:kolmo_s_sstar}
Let \(x\in\Xset\) and \(h>0\) be such that \(\locnorm>0\). Assume
\Aref{assum:S_up}{\Sstar} and \(\DeltaS\in L^1(\locmeasure[x,h])\). Then
\begin{equation}
\label{eq:kolmo_s_sstar}
    \sup_{t\in\R}
    \abs{
    \CDFloc[\Shat]{x}(t)
    -
    \CDFloc[\Sstar]{x}(t)
    }
    \le
    \ups\,\Loneloc{\DeltaS}{x,h}.
\end{equation}
\end{lemma}

To conclude this subsection, we transfer threshold and score deviations to the prediction sets.

\begin{lemma}[Length and coverage comparison under score and threshold deviations]
\label{lem:sandwich_length}
Let $x\in\Xset$, and let $r,q\in\R$ and $\Delta\ge0$ be such that
$\abs{q-r}\le\Delta$. If \Aref{assum:S_up}{\Sstar} holds, then
\begin{equation}
\label{eq:sandwich_coverage}
\abs{
\int_{\Yset} \indi{\conformalset[\Shat]{x}{q}}(y)\, P_{Y\mid X}(\rmd y\mid x)
-
\int_{\Yset} \indi{\conformalset[\Sstar]{x}{r}}(y)\, P_{Y\mid X}(\rmd y\mid x)
}
\le
\ups
\bigl[\Delta+\DeltaS(x)\bigr],
\end{equation}
and if in addition \Aref{assum:S_len}{\Sstar} holds, then
\begin{equation}
\label{eq:sandwich_length}
\abs{
\abs{\conformalset[\Shat]{x}{q}}
-
\abs{\conformalset[\Sstar]{x}{r}}
}
\le
\Llen
\bigl[\Delta+\DeltaS(x)\bigr].
\end{equation}
\end{lemma}

\subsection{Proof of \texorpdfstring{\Cref{thm:adaptive_rlcp_uniform_length}}{the learned-score length theorem}}

The argument parallels the proof of \Cref{thm:rlcp_length}, with an additional training-size requirement accounting for score estimation. For $h\in(0,1]$, $\delta\in\ooint{0,1}$ and
$\boldsymbol{\alpha}=(\alpha,\alphalow,\alphaup)$ as in
\eqref{eq:definition-I}, define the minimal calibration size 
\begin{multline}
\label{eq:definition-n1star}
n_1^\star(h,\delta,\boldsymbol{\alpha})
:=
\min\Bigl\{
n\in\nsets:\;
n\ge n_0(h,\delta),\; \\
\varepsilon_{n,h}(\delta)+\rhom\le\tfrac{\alpha}{2},\;
\varepsilon_{n,h}(\delta)+2(1-\alpha)\rhom\le\tfrac{m_{\boldsymbol{\alpha}}}{2}
\Bigr\}
\end{multline}
and the minimal training size
\begin{equation}
\label{eq:definition-ntrain-delta}
\ntrain^{\Delta}(h,\delta,\boldsymbol{\alpha})
:=
\inf\Bigl\{
k\in\N:\;
\ups\,\epsilon^{\Delta}_{\infty}(k,\delta/2)
<
\min\bigl\{\tfrac{\alpha}{2},\tfrac{m_{\boldsymbol{\alpha}}}{2}\bigr\}
\Bigr\},
\end{equation}
where $n_0$, $\rhom$ and $\varepsilon_{n,h}(\delta)$ are defined in
\eqref{eq:definition_n0}, \eqref{eq:definition-rhom} and
\eqref{eq:definition-varepsilon}, respectively, and
$\epsilon^{\Delta}_{\infty}$ is the score-estimation rate of
\Aref{assum:score_estimation_rate}{S}. Both
quantities are finite: for fixed $(h,\delta)$, the maps
$n\mapsto\varepsilon_{n,h}(\delta)$ and $n\mapsto\rhom$
are positive, nonincreasing and vanish as $n\to\infty$, and
$k\mapsto\epsilon^{\Delta}_{\infty}(k,\delta/2)$ is
nonincreasing and vanishes as $k\to\infty$.

Compared with the fixed-score threshold $n_1$ of \eqref{eq:definition-n1},
$n_1^\star$ only halves the margin condition: the level budget
$m_{\boldsymbol{\alpha}}$ is now split evenly between the calibration error
and the score-estimation error
$\ups\,\epsilon^{\Delta}_{\infty}$, which
$\ntrain^{\Delta}$ caps at $m_{\boldsymbol{\alpha}}/2$, so that their sum
stays below $m_{\boldsymbol{\alpha}}$; see the bound of
\Cref{prop:adaptive_rlcp_quantile}.

For $\tilde{x} \in \Xset$, define
\begin{equation}
    \label{eq:omega_train-definition}
\SetTrain
:=
\left\{
\esssup_{x\in\Xset}\DeltaS(x)
\le
\epsilon^{\Delta}_{\infty}(\ntrain,\delta/2)
\right\}
\end{equation}
and
\begin{equation}
    \label{eq:omega_adapt-definition}
    \SetAdapt = \SetOmega[n,\tilde{x},h,\delta/2] \cap \SetTrain,
\end{equation}
where $\SetOmega[n,\tilde{x},h,\delta/2]$ is the event of
\eqref{eq:definition-Omega} computed with the calibration scores
$\Shat(X_i,Y_i)$; the expectation $A(t)$ of \Cref{prop:calibration_event} is
taken conditionally on $\Dtrain[\ntrain]$. If $0<h\le1$ and
\Cref{assumP:marginal-density,assumK:kernel_assum} and
\Aref{assum:score_estimation_rate}{S} hold, then
\begin{equation}
\label{eq:omega_adapt-probability}
\PP\bigl(\SetAdapt\bigr)\ge 1-\delta .
\end{equation}
Indeed, conditionally on $\Dtrain[\ntrain]$, the learned score $\Shat$ is a
fixed deterministic score independent of the calibration data, so
\Cref{prop:calibration_event} applies with $S=\Shat$ at level $\delta/2$ and
gives $\PP(\SetOmega[n,\tilde{x},h,\delta/2]\mid\Dtrain[\ntrain])\ge1-\delta/2$
almost surely; \Aref{assum:score_estimation_rate}{S}
gives $\PP(\SetTrain)\ge1-\delta/2$; and \eqref{eq:omega_adapt-probability}
follows from a union bound.

On the event $\SetAdapt$, the calibration and score-estimation errors combine into a single deviation bound between the RLCP threshold and the pivot $\taustar$; this is the learned-score analogue of \Cref{prop:rlcp_quantile_oracle}.
\begin{proposition}[Adaptive localized quantile control]
\label{prop:adaptive_rlcp_quantile}
Let $\alpha\in\ooint{0,1}$, $\alphalow\in\ooint{0,\alpha}$,
$\alphaup\in\ooint{\alpha,1}$, $\boldsymbol{\alpha}=(\alpha,\alphalow,\alphaup)$,
$\delta\in\ooint{0,1}$, $\tilde{x}\in\Xset$, and $0<h\le1$.
Assume \Cref{assumP:marginal-density,assumK:kernel_assum},
\Aref{assum:S_up}{\Sstar},
\Aref{assum:S_pivot_quantile}{\alpha} and
\Aref{assum:score_estimation_rate}{S}.
Let $m_{\boldsymbol{\alpha}}$ be defined in \eqref{eq:definition-I}.
Suppose that
\[
n\ge n_1^\star(h,\delta/2,\boldsymbol{\alpha}),
\qquad
\ntrain \ge \ntrain^{\Delta}(h, \delta, \boldsymbol{\alpha}),
\]
with $n_1^\star$ and $\ntrain^{\Delta}$ defined in
\eqref{eq:definition-n1star} and \eqref{eq:definition-ntrain-delta},
respectively.  Then
$\PP(\SetAdapt)\ge1-\delta$ and, on
$\SetAdapt$, the following holds
simultaneously for every $x\in\Xset$:
\begin{equation}
\label{eq:adaptive_quantile_control}
\abs{\rlcpq-\taustar}
\le
\frac{\varepsilon_{n,h}(\delta/2) + 2(1-\alpha)\,\rhom + \ups\,\epsilon^{\Delta}_{\infty}(\ntrain,\delta/2)}{\lows[\Sstar]{\tilde{x}}}.
\end{equation}
\end{proposition}

Define the adaptive deviation rate, for $z\in\Xset$, by
\begin{equation}
\label{eq:length_delta}
\Delta^{\mathrm{adapt}}_{n,\ntrain,h}(z;\delta):=
\frac{\varepsilon_{n,h}(\delta/2)
      + 2(1-\alpha)\,\rhom
      + \ups\,\epsilon^{\Delta}_{\infty}(\ntrain,\delta/2)}{\lows[\Sstar]{z}}.
\end{equation}
The dependence on $\boldsymbol{\alpha}$ is left implicit, as for
$\Delta^{\mathrm R}_{n,h}$ in \eqref{eq:definition-rate}.

\begin{proof}[Proof of \Cref{thm:adaptive_rlcp_uniform_length}]
By \Cref{prop:adaptive_rlcp_quantile},
$\PP(\SetAdapt)\ge1-\delta$. On $\SetAdapt$, simultaneously for every
$x\in\Xset$,
\[
\abs{\rlcpq-\taustar}
\le \Delta^{\mathrm{adapt}}_{n,\ntrain,h}(\tilde x;\delta).
\]
Since
$\rlcpC[x][\tilde x]=\conformalset[\Shat]{x}{\rlcpq}$,
\eqref{eq:sandwich_length} of \Cref{lem:sandwich_length}, with
$(r,q)=(\taustar,\rlcpq)$, gives simultaneously for every $x$
\begin{equation}
\label{eq:adaptive-length-pointwise}
\abs{\abs{\rlcpC[x][\tilde x]}
-\abs{\conformalset[\Sstar]{x}{\taustar}}}
\le
\Llen
\bigl[\Delta^{\mathrm{adapt}}_{n,\ntrain,h}(\tilde x;\delta)+\DeltaS(x)\bigr].
\end{equation}

Comparing \eqref{eq:definition-rhom} and \eqref{eq:definition-varepsilon}
with \eqref{eq:th-bound-definition}, using $\log(8/\delta)>1$, yields
$\varepsilon_{n,h}(\delta/2)+2(1-\alpha)\rhom
\lesssim A_{n,h}^{\mathrm{cal}}(\delta/2)$.
Substitution in \eqref{eq:length_delta} and then
\eqref{eq:adaptive-length-pointwise}, restricted to
$B(\tilde x,h)\cap\Xset$, proves
\eqref{eq:adaptive_rlcp_length_bound_random_main}.
\end{proof}

\subsection{Proof of \texorpdfstring{\Cref{thm:adaptive_rlcp_uniform_coverage}}{the learned-score coverage theorem}}
\begin{proof}
On the event $\SetAdapt$ of \Cref{prop:adaptive_rlcp_quantile}, the same
quantile bound used above holds simultaneously for every $x$. Moreover,
\Cref{lem:quant_pivo} gives
$\CondCDF[\Sstar](\taustar\mid x)=1-\alpha$ for
$\DC[X]$-almost every $x$. For each such $x$,
\eqref{eq:sandwich_coverage} of \Cref{lem:sandwich_length}, again with
$(r,q)=(\taustar,\rlcpq)$, therefore gives
\begin{equation}
\label{eq:adaptive-coverage-pointwise}
\abs{
\int_{\Yset} \indi{\rlcpC[x][\tilde x]}(y)\, P_{Y\mid X}(\rmd y\mid x)
- (1-\alpha)}
\le
\ups
\bigl[\Delta^{\mathrm{adapt}}_{n,\ntrain,h}(\tilde x;\delta)+\DeltaS(x)\bigr].
\end{equation}
The rate comparison from the preceding proof, substituted into
\eqref{eq:length_delta} and \eqref{eq:adaptive-coverage-pointwise}, proves
\eqref{eq:adaptive_rlcp_uniform_coverage_simple}; its almost-everywhere
qualification is inherited solely from pivotality.
\end{proof}

\section{Sufficient Condition for Localized Density Minorization}
\label{supp:sec:minorization}

Recall that, for this sufficient condition, the levels are the symmetric
choice \eqref{eq:definition-margin_hmax} and that $\bwalpha$ and the window
$\J$ are defined in \eqref{eq:definition-hmax_window}.

\begin{replemma}{lem:check-assum:pdfloc_bound}
Let $\shapefun\in\shapefunset$.  Suppose
\Cref{assumP:marginal-density,assumK:kernel_assum},
\Aref{assum:S_up}{S}, and \Aref{assum:lip_condcdf}{S} hold.  Assume
that, for every $u\in\Xset$ and $0<h\le1$, there exists
$\sigma_{\shapefun}(u,h)>0$ such that
\[
  \CondPDF(s\mid z)
  \ge\sigma_{\shapefun}(u,h)
  \shapefun(\CondCDF(s\mid z))
\]
for every $z\in B(u,h)\cap\Xset$ and $s\in\R$.  Then, for every
$u\in\Xset$ and $0<h\le1\wedge\bwalpha$, the minorization condition
\eqref{eq:pdfloc_bound} holds:
$\lows{u}\ge\sigma_{\shapefun}(u,h)\inf_{r\in\J}\shapefun(r)>0$.
\end{replemma}

\begin{proof}
Fix $u\in\Xset$ and $0<h\le1\wedge\bwalpha$.  Under
\Cref{assumP:marginal-density,assumK:kernel_assum}, which are in force
throughout Section~\ref{sec:fixed_score} of the main paper,
\Cref{lem:gamma_bound_exact} gives $\locnorm[u]\ge c_0\pmin h^d>0$, so the
localized law $\locmeasure[u,h]$ of \eqref{eq:lambda} is a well-defined
probability measure and the quantities entering \eqref{eq:pdfloc_bound}
make sense.  Since the kernel $\KDE$ is supported in $B(0,1)$
(\Cref{assumK:kernel_assum}), the weight
$\locweight[u]{v}=\KDE((v-u)/h)$ vanishes for $\norm{v-u}>h$, so
$\locmeasure[u,h]$ is supported on $B(u,h)\cap\Xset$.

\emph{Step 1: absolute continuity of the localized score CDF.}
By \Aref{assum:S_up}{S}, each conditional score law has the density
$\CondPDF(\cdot\mid v)$, and by Tonelli's theorem, for every $t\in\R$,
\[
    \CDFloc{u}(t)
    =\int_{\Xset}\CondCDF(t\mid v)\,\locmeasure[u,h](\rmd v)
    =\int_{\Xset}\int_{-\infty}^{t}\CondPDF(s\mid v)\,\rmd s\,
      \locmeasure[u,h](\rmd v)
    =\int_{-\infty}^{t}\PDFloc{u}(s)\,\rmd s ,
\]
with $\PDFloc{u}$ as defined in \eqref{eq:localized-pdf}.  Hence
$\CDFloc{u}$ is absolutely continuous with density $\PDFloc{u}$; in
particular it is continuous on $\R$.

\emph{Step 2: level of the localized CDF on the quantile window.}
Both endpoints of
$\I[u]
=\ccint{\locoraclequant{u}[1-\alphaup],\,\locoraclequant{u}[1-\alphalow]}$
are finite, being left quantiles of the proper distribution function
$\CDFloc{u}$ at levels in $\ooint{0,1}$.  For any $p\in\ooint{0,1}$, right
continuity of $\CDFloc{u}$ at its left quantile gives
$\CDFloc{u}(\locoraclequant{u}[p])\ge p$, while, by the definition of the
left quantile, $\CDFloc{u}(t)<p$ for every $t<\locoraclequant{u}[p]$, so
that the left limit of $\CDFloc{u}$ at $\locoraclequant{u}[p]$ is at most
$p$; the continuity established in Step~1 then forces
$\CDFloc{u}(\locoraclequant{u}[p])=p$.  Consequently, by monotonicity of
$\CDFloc{u}$, for every $t\in\I[u]$,
\begin{equation}
\label{supp:eq:window-level}
    1-\alphaup
    \le
    \CDFloc{u}(t)
    \le
    1-\alphalow .
\end{equation}

\emph{Step 3: the conditional CDFs stay in the window $\J$.}
Let $t\in\I[u]$ and $z\in B(u,h)\cap\Xset$.  Since $\locmeasure[u,h]$ is
supported on $B(u,h)\cap\Xset$ and $\norm{z-w}\le\norm{z-u}+\norm{u-w}
\le2h$ for every $w\in B(u,h)$, \Aref{assum:lip_condcdf}{S} gives
\[
    \bigl|\CondCDF(t\mid z)-\CDFloc{u}(t)\bigr|
    \le
    \int_{\Xset}
        \bigl|\CondCDF(t\mid z)-\CondCDF(t\mid w)\bigr|
        \,\locmeasure[u,h](\rmd w)
    \le
    L_S(2h)^{\beta} .
\]
Combined with \eqref{supp:eq:window-level} and the definition
\eqref{eq:definition-hmax_window} of $\J$, this shows that
\begin{equation}
\label{supp:eq:condcdf-in-J}
    \CondCDF(t\mid z)\in\J
    \qquad
    \text{for every }t\in\I[u]\text{ and every }z\in B(u,h)\cap\Xset .
\end{equation}
Moreover, $h\le\bwalpha$ and \eqref{eq:definition-hmax_window} give
$L_S(2h)^\beta\le m_\alpha$, whence, by
\eqref{eq:definition-margin_hmax},
\begin{equation}
\label{supp:eq:J-containment}
\J
\subseteq
\ccint{1-\alphaup-m_\alpha,\;1-\alphalow+m_\alpha}
=
\ccint{(1-\alpha)^2,\;1-\alpha^2}
\subset\ooint{0,1} :
\end{equation}
indeed, $\alphaup=\alpha+m_\alpha$, $\alphalow=\alpha-m_\alpha$ and
$m_\alpha=\tfrac12\alpha(1-\alpha)$ yield
$1-\alphaup-m_\alpha=1-\alpha-\alpha(1-\alpha)=(1-\alpha)^2>0$ and
$1-\alphalow+m_\alpha=1-\alpha+\alpha(1-\alpha)=1-\alpha^2<1$.

\emph{Step 4: conclusion.}
The function $\shapefun$ is positive and lower semicontinuous on
$\ooint{0,1}$, and $\J$ is a compact subinterval of $\ooint{0,1}$ by
\eqref{supp:eq:J-containment}; a lower semicontinuous function attains its
infimum on a nonempty compact set, so
$\inf_{r\in\J}\shapefun(r)>0$.  Lower semicontinuity also makes
$\shapefun$ Borel measurable, so, for each fixed $t$, the map
$v\mapsto\shapefun(\CondCDF(t\mid v))$ is measurable as the composition of
$\shapefun$ with the measurable map $v\mapsto\CondCDF(t\mid v)$.
Therefore, for every $t\in\I[u]$, the assumed lower bound
\eqref{eq:lowerbound_sym} and \eqref{supp:eq:condcdf-in-J} give
\[
    \PDFloc{u}(t)
    =\int_{\Xset}\CondPDF(t\mid v)\,\locmeasure[u,h](\rmd v)
    \ge
    \sigma_{\shapefun}(u,h)
    \int_{\Xset}\shapefun\bigl(\CondCDF(t\mid v)\bigr)
    \,\locmeasure[u,h](\rmd v)
    \ge
    \sigma_{\shapefun}(u,h)\inf_{r\in\J}\shapefun(r) ,
\]
using once more that $\locmeasure[u,h]$ is a probability measure supported
on $B(u,h)\cap\Xset$.  Taking the essential infimum over $t\in\I[u]$ in
the definition \eqref{eq:pdfloc_bound} of $\lows{u}$ proves the claim.
\end{proof}

\section{Residual Score: Proof of Proposition~\ref{prop:ex-residual-score}}
\label{supp:sec:residual}

The proof rests on an explicit density version of the conditional law of
the residual score, together with a lower bound on that density over the
support of the score.

\begin{lemma}[Residual score density and tail control]
\label{supp:lem:residual-density}
Suppose \Cref{assumP:density}.  Let $S=S_{\mathrm{res}}$, with
$S_{\mathrm{res}}$ and $\mu$ defined in \eqref{eq:ex-mu-def}, and extend
$p_{Y\mid X}(\cdot\mid x)$ by zero outside $\Yset=[-M,M]$.  Then, for
every $x\in\Xset$, the conditional law of $S_{\mathrm{res}}(X,Y)$ given
$X=x$ is absolutely continuous, with density version
\begin{equation}
\label{supp:eq:residual-density}
    \CondPDF(s\mid x)
    =
    \indi{\coint{0,\infty}}(s)
    \Bigl[
        p_{Y\mid X}(\mu(x)+s\mid x)
        +p_{Y\mid X}(\mu(x)-s\mid x)
    \Bigr],
    \qquad s\in\R .
\end{equation}
Moreover, for every $x\in\Xset$ and every
$0\le s\le M+\abs{\mu(x)}$,
\begin{equation}
\label{supp:eq:residual-tail}
    \CondPDF(s\mid x)\ge\low .
\end{equation}
\end{lemma}

\begin{proof}
Fix $x\in\Xset$.  Since the conditional law of $Y$ given $X=x$ is
supported on $[-M,M]$, its mean satisfies $\mu(x)\in[-M,M]$.  For $t<0$,
$\CondCDF(t\mid x)=0$, because $S_{\mathrm{res}}\ge0$.  For $t\ge0$, the
zero extension of the conditional density gives
\[
    \CondCDF(t\mid x)
    =\PP\bigl\{\abs{Y-\mu(x)}\le t\bigm|X=x\bigr\}
    =\int_{\mu(x)-t}^{\mu(x)+t}p_{Y\mid X}(y\mid x)\,\rmd y .
\]
Splitting the integral at $\mu(x)$ and substituting $y=\mu(x)+s$ on
$[\mu(x),\mu(x)+t]$ and $y=\mu(x)-s$ on $[\mu(x)-t,\mu(x)]$ yields
\[
    \CondCDF(t\mid x)
    =\int_0^t
        \Bigl[p_{Y\mid X}(\mu(x)+s\mid x)+p_{Y\mid X}(\mu(x)-s\mid x)\Bigr]
     \rmd s ,
    \qquad t\ge0 .
\]
Together with $\CondCDF(t\mid x)=0$ for $t<0$, this proves absolute
continuity with the density version \eqref{supp:eq:residual-density}.

We now prove \eqref{supp:eq:residual-tail}, distinguishing two regimes.

\emph{Two-branch regime: $0\le s\le M-\abs{\mu(x)}$} (void if
$\abs{\mu(x)}=M$).  Then both arguments in
\eqref{supp:eq:residual-density} lie in $[-M,M]$:
$\mu(x)+s\le\abs{\mu(x)}+s\le M$ and $\mu(x)+s\ge\mu(x)\ge-M$; similarly
$\mu(x)-s\ge-\abs{\mu(x)}-s\ge-M$ and $\mu(x)-s\le\mu(x)\le M$.  By
\Cref{assumP:density},
$\CondPDF(s\mid x)\ge2\low\ge\low$.

\emph{One-branch regime: $M-\abs{\mu(x)}<s\le M+\abs{\mu(x)}$} (void if
$\mu(x)=0$).  Suppose first $\mu(x)\ge0$.  Then
$\mu(x)+s>\mu(x)+M-\mu(x)=M$, so the first branch of
\eqref{supp:eq:residual-density} vanishes, while
$\mu(x)-s\ge\mu(x)-M-\mu(x)=-M$ and
$\mu(x)-s<\mu(x)-M+\mu(x)=2\mu(x)-M\le M$, so
$\mu(x)-s\in[-M,M]$ and \Cref{assumP:density} gives
$\CondPDF(s\mid x)\ge p_{Y\mid X}(\mu(x)-s\mid x)\ge\low$.  The case
$\mu(x)<0$ is symmetric: $\mu(x)-s<-M$ kills the second branch, while
$\mu(x)+s\in[-M,M]$ because
$\mu(x)+s>\mu(x)+M+\mu(x)=M+2\mu(x)\ge-M$ and
$\mu(x)+s\le\mu(x)+M-\mu(x)=M$.  In either case
\eqref{supp:eq:residual-tail} holds.
\end{proof}

\begin{repproposition}{prop:ex-residual-score}
Suppose \Cref{assumP:marginal-density,assumP:density,assumP:density-holder}
and \Cref{assumK:kernel_assum} hold.  Let
$S=S_{\mathrm{res}}$ be defined by \eqref{eq:ex-mu-def}, and set
$\shapefun_{\mathrm{res}}(r)=\indi{(0,1)}(r)$ for $r\in[0,1]$.  Then
$\shapefun_{\mathrm{res}}\in\shapefunset$ and the following statements hold.
\begin{enumerate}[label=(\roman*),leftmargin=0pt]
\item
\Aref{assum:S_up}{S_{\mathrm{res}}} holds with $\ups=2\up$.

\item
the localized density minorization \eqref{eq:pdfloc_bound} holds for the
levels in \eqref{eq:definition-margin_hmax} and every
$0<h\le1\wedge\bwalpha$, with
\[
  \sigma_{\shapefun_{\mathrm{res}}}(u,h)=\low[u,h],
  \qquad
  \lows{u}
  \ge\sigma_{\shapefun_{\mathrm{res}}}(u,h)
  \inf_{r\in\J}\shapefun_{\mathrm{res}}(r)
  =\low[u,h].
\]

\item
\Aref{assum:lip_condcdf}{S_{\mathrm{res}}} holds with the exponent
$\beta$ of \Cref{assumP:density-holder} and $L_S=L_p(1+2M\up)$.

\item
\Aref{assum:S_len}{S_{\mathrm{res}}} holds with $\Llen[x]=2$ for every
$x\in\Xset$.
\end{enumerate}
\end{repproposition}

\begin{proof}
First, $\shapefun_{\mathrm{res}}\in\shapefunset$: the indicator of the
open set $\ooint{0,1}$ is lower semicontinuous, it is positive on
$\ooint{0,1}$, and it vanishes at $0$ and~$1$.  We prove the four items in
the order \ref{item:assum:S_up}, \ref{item:assum:lip_condcdf},
\ref{item:assum:lowerbound_sym}, \ref{ex-item-len}; the H\"older property
\ref{item:assum:lip_condcdf} is established first because the bandwidth
threshold $\bwalpha$ appearing in \ref{item:assum:lowerbound_sym} is
defined in \eqref{eq:definition-hmax_window} through its constant $L_S$.

\emph{Proof of \ref{item:assum:S_up}.}
By \Cref{supp:lem:residual-density}, the conditional score law admits the
density version \eqref{supp:eq:residual-density}.  Each of its two
branches is bounded by $\up$ under \Cref{assumP:density}, so
$\CondPDF(s\mid x)\le2\up$ for every $s\in\R$ and $x\in\Xset$; and the
version \eqref{supp:eq:residual-density} is jointly measurable in $(s,x)$
whenever $(y,x)\mapsto p_{Y\mid X}(y\mid x)$ and $\mu$ are measurable.
Hence \Aref{assum:S_up}{S_{\mathrm{res}}} holds with $\ups=2\up$.

\emph{Proof of \ref{item:assum:lowerbound_sym}.}
We verify the hypothesis \eqref{eq:lowerbound_sym} of
\Cref{lem:check-assum:pdfloc_bound} with
$\shapefun\leftarrow\shapefun_{\mathrm{res}}$ and
$\sigma_{\shapefun_{\mathrm{res}}}(u,h)=\low[u,h]$, which is positive by
\Cref{assumP:density}.  Let $u\in\Xset$, $0<h\le1$,
$z\in B(u,h)\cap\Xset$ and $s\in\R$.  If
$\CondCDF(s\mid z)\in\{0,1\}$, then
$\shapefun_{\mathrm{res}}(\CondCDF(s\mid z))=0$ and there is nothing to
prove.  Otherwise $0<\CondCDF(s\mid z)<1$.  Then $s\ge0$, since the score
CDF vanishes on $\ooint{-\infty,0}$; and $s<M+\abs{\mu(z)}$, since for
$s\ge M+\abs{\mu(z)}$ the interval $[\mu(z)-s,\mu(z)+s]$ contains
$[-M,M]$ and hence $\CondCDF(s\mid z)=1$.  The tail control
\eqref{supp:eq:residual-tail} of \Cref{supp:lem:residual-density}
therefore gives
\[
    \CondPDF(s\mid z)
    \ge\low[z]
    \ge\inf_{v\in B(u,h)\cap\Xset}\low[v]
    =\low[u,h]
    =\low[u,h]\,\shapefun_{\mathrm{res}}\bigl(\CondCDF(s\mid z)\bigr),
\]
which is \eqref{eq:lowerbound_sym}.  Items \ref{item:assum:S_up} and
\ref{item:assum:lip_condcdf} provide \Aref{assum:S_up}{S_{\mathrm{res}}}
and \Aref{assum:lip_condcdf}{S_{\mathrm{res}}}, the latter with the
constant $L_S=L_p(1+2M\up)$ entering
\eqref{eq:definition-hmax_window}; \Cref{lem:check-assum:pdfloc_bound}
then yields, for the levels \eqref{eq:definition-margin_hmax} and every
$0<h\le1\wedge\bwalpha$,
\[
  \lows{u}
  \ge\low[u,h]\,\inf_{r\in\J}\shapefun_{\mathrm{res}}(r)
  =\low[u,h] ,
\]
the final equality because $\J\subset\ooint{0,1}$ by
\eqref{supp:eq:J-containment}, so that
$\shapefun_{\mathrm{res}}\equiv1$ on $\J$.

\emph{Proof of \ref{item:assum:lip_condcdf}.}
Fix $t\in\R$ and $u,u'\in\Xset$.  For $t<0$ both conditional score CDFs
vanish, so we may take $t\ge0$.  By \eqref{eq:ex-mu-def}, adding and
subtracting $\int_{\{y:\abs{y-\mu(u)}\le t\}}p_{Y\mid X}(y\mid u')\,\rmd y$
and then using \Cref{assumP:density-holder},
\begin{align*}
\bigl|\CondCDF(t\mid u)-\CondCDF(t\mid u')\bigr|
&\le
\int_{\Yset}
    \bigl|p_{Y\mid X}(y\mid u)-p_{Y\mid X}(y\mid u')\bigr|\,\rmd y
\\
&\quad+
\int_{\{y:\abs{y-\mu(u)}\le t\}\,\triangle\,\{y:\abs{y-\mu(u')}\le t\}}
    p_{Y\mid X}(y\mid u')\,\rmd y
\\
&\le
L_p\norm{u-u'}^\beta
+
\up\,\bigl|
\{y:\abs{y-\mu(u)}\le t\}\,\triangle\,\{y:\abs{y-\mu(u')}\le t\}
\bigr| .
\end{align*}
The two sets in the symmetric difference are intervals of common length
$2t$, centred at $\mu(u)$ and $\mu(u')$ respectively, so the Lebesgue
measure of their symmetric difference is at most
$2\abs{\mu(u)-\mu(u')}$.  Moreover, by \eqref{eq:ex-mu-def},
$\abs{y}\le M$ on $\Yset$, and \Cref{assumP:density-holder},
\[
    \abs{\mu(u)-\mu(u')}
    \le
    \int_{\Yset}\abs{y}\,
    \bigl|p_{Y\mid X}(y\mid u)-p_{Y\mid X}(y\mid u')\bigr|\,\rmd y
    \le
    M L_p\norm{u-u'}^\beta .
\]
Combining the three displays and taking the supremum over $t\in\R$ yields
\[
    \sup_{t\in\R}
    \bigl|\CondCDF(t\mid u)-\CondCDF(t\mid u')\bigr|
    \le
    L_p\bigl(1+2M\up\bigr)\norm{u-u'}^\beta ,
\]
which is \Aref{assum:lip_condcdf}{S_{\mathrm{res}}} with the claimed
constant.

\emph{Proof of \ref{ex-item-len}.}
Fix $x\in\Xset$.  For $q<0$, $\conformalset{x}{q}=\emptyset$.  For
$0\le q\le q'$,
\[
    \conformalset{x}{q}
    =
    [\mu(x)-q,\,\mu(x)+q]\cap[-M,M]
    \subseteq
    \conformalset{x}{q'} ,
\]
and
\begin{align*}
    \bigl|\conformalset{x}{q'}\bigr|-
    \bigl|\conformalset{x}{q}\bigr|
    &=
    \bigl|\conformalset{x}{q'}\setminus\conformalset{x}{q}\bigr| \\
    \le
    &\bigl|[\mu(x)-q',\,\mu(x)-q)\bigr|
    +\bigl|(\mu(x)+q,\,\mu(x)+q']\bigr|
    =2(q'-q).
\end{align*}
If $q<0\le q'$, then
$\bigl|\conformalset{x}{q'}\bigr|-\bigl|\conformalset{x}{q}\bigr|
=\bigl|\conformalset{x}{q'}\bigr|\le2q'\le2(q'-q)$, and if $q\le q'<0$
both sets are empty.  Hence, for all $q\le q'$,
\[
    0\le
    \bigl|\conformalset{x}{q'}\bigr|-
    \bigl|\conformalset{x}{q}\bigr|
    \le 2(q'-q),
\]
which is \Aref{assum:S_len}{S_{\mathrm{res}}} with $\Llen[x]=2$.
\end{proof}

\section{Conformalized Quantile Regression Score: Proof of Proposition~\ref{prop:ex-cqr} and a Relaxed Variant}
\label{supp:sec:cqr}

As in \Cref{supp:sec:residual}, the proof rests on an explicit density
version of the conditional score law and a lower bound on that density
over the support of the score.  For the CQR score the support moves with
the covariate; the following lemma identifies it.

\begin{lemma}[CQR score density and tail control]
\label{supp:lem:cqr-density}
Suppose \Cref{assumP:density-unif}.  Let $\alpha\in\ooint{0,1}$ and
$0<\alo<\ahi<1$ with $\ahi-\alo=1-\alpha$, let
$\Sstar=\Sstar_{\mathrm{cqr}}$, with $\Sstar_{\mathrm{cqr}}$ and
$\cqrquant{\alo},\cqrquant{\ahi}$ defined in \eqref{eq:ex-cqr-def}, and
extend $p_{Y\mid X}(\cdot\mid x)$ by zero outside $\Yset=[-M,M]$.  For
$x\in\Xset$, set
\[
s^{(\Sstar)}_-(x)
:= -\tfrac{1}{2}\bigl(\cqrquant{\ahi}(x)-\cqrquant{\alo}(x)\bigr),
\qquad
s^{(\Sstar)}_+(x)
:= \max\bigl\{M-\cqrquant{\ahi}(x),\;M+\cqrquant{\alo}(x)\bigr\}.
\]
Then, for every $x\in\Xset$, the conditional law of
$\Sstar_{\mathrm{cqr}}(X,Y)$ given $X=x$ is absolutely continuous: for
every $t\ge s^{(\Sstar)}_-(x)$,
\begin{multline}
\label{supp:eq:cqr-cdf}
\CondCDF[\Sstar](t\mid x)=(1-\alpha)
\\
+\int_0^t\Bigl(
p_{Y\mid X}(s+\cqrquant{\ahi}(x)\mid x)\,
\indi{s+\cqrquant{\ahi}(x)\le M}
+p_{Y\mid X}(\cqrquant{\alo}(x)-s\mid x)\,
\indi{\cqrquant{\alo}(x)-s\ge-M}
\Bigr)\rmd s,
\end{multline}
while $\CondCDF[\Sstar](t\mid x)=0$ for $t<s^{(\Sstar)}_-(x)$, with
density version
\begin{multline}
\label{supp:eq:cqr-density}
    \CondPDF[\Sstar](s\mid x)
    = \indi{\coint{s^{(\Sstar)}_-(x),\infty}}(s)
    \Bigl[p_{Y\mid X}(s+\cqrquant{\ahi}(x)\mid x)\,
    \indi{s+\cqrquant{\ahi}(x)\le M}
    \\
    +p_{Y\mid X}(\cqrquant{\alo}(x)-s\mid x)\,
    \indi{\cqrquant{\alo}(x)-s\ge-M}\Bigr],
    \qquad s\in\R .
\end{multline}
Moreover, $\Sstar_{\mathrm{cqr}}(x,Y)\le s^{(\Sstar)}_+(x)$ almost surely
under $P_{Y\mid X=x}$, and for every $x\in\Xset$ and every
$s^{(\Sstar)}_-(x)\le s\le s^{(\Sstar)}_+(x)$,
\begin{equation}
\label{supp:eq:cqr-tail}
    \CondPDF[\Sstar](s\mid x)
    \ge \uniflow .
\end{equation}
\end{lemma}

\begin{proof}
Fix $x\in\Xset$.  Under \Cref{assumP:density-unif}, $F_{Y\mid X}(\cdot\mid x)$
is continuous and strictly increasing on $[-M,M]$, with
$F_{Y\mid X}(-M\mid x)=0$ and $F_{Y\mid X}(M\mid x)=1$; hence
$\cqrquant{\alo}(x),\cqrquant{\ahi}(x)\in[-M,M]$ and
$F_{Y\mid X}(\cqrquant{\tau}(x)\mid x)=\tau$ for
$\tau\in\{\alo,\ahi\}$.  In particular,
\begin{equation}
\label{supp:eq:cqr-quantile-mass}
\int_{\cqrquant{\alo}(x)}^{\cqrquant{\ahi}(x)}
p_{Y\mid X}(y\mid x)\,\rmd y
=\ahi-\alo=1-\alpha .
\end{equation}
By \eqref{eq:ex-cqr-def}, for every $y$ and $t$,
$\Sstar_{\mathrm{cqr}}(x,y)\le t$ if and only if
$\cqrquant{\alo}(x)-t\le y\le\cqrquant{\ahi}(x)+t$.

\emph{The score is bounded below by $s^{(\Sstar)}_-(x)$ and above by
$s^{(\Sstar)}_+(x)$.}
For every $y$,
\[
\Sstar_{\mathrm{cqr}}(x,y)
=\max\bigl\{\cqrquant{\alo}(x)-y,\;y-\cqrquant{\ahi}(x)\bigr\}
\ge\tfrac12\bigl[(\cqrquant{\alo}(x)-y)+(y-\cqrquant{\ahi}(x))\bigr]
=s^{(\Sstar)}_-(x),
\]
so $\CondCDF[\Sstar](t\mid x)=0$ for $t<s^{(\Sstar)}_-(x)$.  On the other
hand, $y\mapsto\max\{\cqrquant{\alo}(x)-y,\,y-\cqrquant{\ahi}(x)\}$ is a
maximum of two affine functions, hence convex, and attains its maximum
over $[-M,M]$ at an endpoint, where it equals $M+\cqrquant{\alo}(x)$ or
$M-\cqrquant{\ahi}(x)$; since $Y\in[-M,M]$ almost surely,
$\Sstar_{\mathrm{cqr}}(x,Y)\le s^{(\Sstar)}_+(x)$ almost surely.

\emph{The CDF formula for $s^{(\Sstar)}_-(x)\le t<0$.}
In this range,
\begin{equation}
\label{supp:eq:cqr-inside}
-M \le \cqrquant{\alo}(x)\le \cqrquant{\alo}(x)-t
\le \cqrquant{\ahi}(x)+t \le \cqrquant{\ahi}(x) \le M ,
\end{equation}
the middle inequality because $t\ge s^{(\Sstar)}_-(x)$.  Hence, using
\eqref{supp:eq:cqr-quantile-mass} and the substitutions
$y=\cqrquant{\alo}(x)-s$ and $y=\cqrquant{\ahi}(x)+s$,
\begin{align*}
    \CondCDF[\Sstar](t\mid x)
    &=\int_{\cqrquant{\alo}(x)-t}^{\cqrquant{\ahi}(x)+t}
      p_{Y\mid X}(y\mid x)\,\rmd y
    \\
    &=\int_{\cqrquant{\alo}(x)}^{\cqrquant{\ahi}(x)}
      p_{Y\mid X}(y\mid x)\,\rmd y
    -\int_{\cqrquant{\alo}(x)}^{\cqrquant{\alo}(x)-t}
      p_{Y\mid X}(y\mid x)\,\rmd y
    -\int_{\cqrquant{\ahi}(x)+t}^{\cqrquant{\ahi}(x)}
      p_{Y\mid X}(y\mid x)\,\rmd y
    \\
    &=(1-\alpha)
    +\int_0^t\Bigl(
      p_{Y\mid X}(s+\cqrquant{\ahi}(x)\mid x)
      +p_{Y\mid X}(\cqrquant{\alo}(x)-s\mid x)\Bigr)\rmd s .
\end{align*}
By \eqref{supp:eq:cqr-inside}, for every $s\in[t,0]$ both indicators in
\eqref{supp:eq:cqr-cdf} equal one, so this is exactly
\eqref{supp:eq:cqr-cdf} on this range.

\emph{The CDF formula for $t\ge0$.}
Using the zero extension of the density,
\begin{align*}
   \CondCDF[\Sstar](t\mid x)
   &=\int_{\cqrquant{\alo}(x)-t}^{\cqrquant{\ahi}(x)+t}
   \indi{y\in[-M,M]}\,p_{Y\mid X}(y\mid x)\,\rmd y
   \\
   &=\int_{\cqrquant{\alo}(x)}^{\cqrquant{\ahi}(x)}
   p_{Y\mid X}(y\mid x)\,\rmd y
   +\int_{\cqrquant{\alo}(x)-t}^{\cqrquant{\alo}(x)}
   \indi{y\in[-M,M]}\,p_{Y\mid X}(y\mid x)\,\rmd y
   \\
   &\qquad
   +\int_{\cqrquant{\ahi}(x)}^{\cqrquant{\ahi}(x)+t}
   \indi{y\in[-M,M]}\,p_{Y\mid X}(y\mid x)\,\rmd y ,
\end{align*}
where the middle integrand needs no indicator because
$[\cqrquant{\alo}(x),\cqrquant{\ahi}(x)]\subseteq[-M,M]$.  Substituting
$y=\cqrquant{\alo}(x)-s$ in the second integral and
$y=\cqrquant{\ahi}(x)+s$ in the third, and noting that for $s\ge0$ the
constraints $y\in[-M,M]$ reduce to $\cqrquant{\alo}(x)-s\ge-M$ and
$s+\cqrquant{\ahi}(x)\le M$ respectively (the other side being automatic,
since $\cqrquant{\alo}(x)-s\le\cqrquant{\alo}(x)\le M$ and
$s+\cqrquant{\ahi}(x)\ge\cqrquant{\ahi}(x)\ge-M$), gives
\eqref{supp:eq:cqr-cdf} on $[0,\infty)$ as well, using again
\eqref{supp:eq:cqr-quantile-mass}.

Combining the three ranges proves \eqref{supp:eq:cqr-cdf}, and in
particular the conditional score law is absolutely continuous.
Differentiating \eqref{supp:eq:cqr-cdf} with respect to $t$ yields the
bracketed expression in \eqref{supp:eq:cqr-density} for almost every
$s\ge s^{(\Sstar)}_-(x)$; since $\CondCDF[\Sstar](\cdot\mid x)$ vanishes
on $\ooint{-\infty,s^{(\Sstar)}_-(x)}$, the density may be taken to
vanish there, which gives \eqref{supp:eq:cqr-density}.

\emph{Tail control.}
Let $s^{(\Sstar)}_-(x)\le s\le s^{(\Sstar)}_+(x)$.  If
$s\le M-\cqrquant{\ahi}(x)$, then $s+\cqrquant{\ahi}(x)\le M$ and, since
$s\ge s^{(\Sstar)}_-(x)$,
\[
s+\cqrquant{\ahi}(x)
\ge
\tfrac{1}{2}\bigl(\cqrquant{\alo}(x)+\cqrquant{\ahi}(x)\bigr)
\ge -M ,
\]
so the first term in \eqref{supp:eq:cqr-density} is at least $\uniflow$
by \Cref{assumP:density-unif}.  Otherwise $s>M-\cqrquant{\ahi}(x)$, and
then $s\le s^{(\Sstar)}_+(x)$ forces $s\le M+\cqrquant{\alo}(x)$, so
$\cqrquant{\alo}(x)-s\ge-M$ and, since $s\ge s^{(\Sstar)}_-(x)$,
\[
\cqrquant{\alo}(x)-s
\le
\tfrac{1}{2}\bigl(\cqrquant{\alo}(x)+\cqrquant{\ahi}(x)\bigr)
\le M ,
\]
so the second term in \eqref{supp:eq:cqr-density} is at least $\uniflow$.
In either case \eqref{supp:eq:cqr-tail} follows, the two terms of
\eqref{supp:eq:cqr-density} being nonnegative.
\end{proof}

\begin{repproposition}{prop:ex-cqr}
Assume \Cref{assumP:marginal-density}, \Cref{assumP:density-unif},
\Cref{assumP:density-holder}, and \Cref{assumK:kernel_assum}.  Let
$\Yset=[-M,M]$, $\alpha\in\ooint{0,1}$, and let
$\Sstar=\Sstar_{\mathrm{cqr}}$ be defined in \eqref{eq:ex-cqr-def}.  Then:
\begin{enumerate}[label=(\roman*),leftmargin=0pt]
\item
\Aref{assum:S_up}{\Sstar_{\mathrm{cqr}}} holds with
$\ups=2\,\up$.

\item
\Aref{assum:lip_condcdf}{\Sstar_{\mathrm{cqr}}} holds with the exponent
$\beta$ of \Cref{assumP:density-holder} and
$L_S=L_p\bigl(1+2\up/\uniflow\bigr)$.  Moreover, for the levels
$\alphalow,\alphaup$ in \eqref{eq:definition-margin_hmax} and every
$0<h\le1\wedge\bwalpha$, with $\bwalpha$ given by
\eqref{eq:definition-hmax_window} for this $L_S$, the localized density
minorization \eqref{eq:pdfloc_bound} holds with
$\lows[\Sstar_{\mathrm{cqr}}]{u}\ge\uniflow$ for every $u\in\Xset$.

\item
\Aref{assum:S_len}{\Sstar_{\mathrm{cqr}}} holds with $\Llen[x]=2$ for every
$x\in\Xset$.

\item
\Aref{assum:S_pivot_quantile}{\alpha} holds with pivot $\taustar=0$.
\end{enumerate}
\end{repproposition}

\begin{proof}
Throughout, fix $x\in\Xset$, let $s^{(\Sstar)}_-(x)$ and
$s^{(\Sstar)}_+(x)$ be as in \Cref{supp:lem:cqr-density}, and work with
the density version \eqref{supp:eq:cqr-density}.

\emph{Proof of \ref{item:ex2-item-a1}.}
By \Cref{supp:lem:cqr-density}, the conditional score law admits the
density version \eqref{supp:eq:cqr-density}, a sum of at most two terms,
each bounded by $\up$ under \Cref{assumP:density-unif}; the version is
jointly measurable whenever $(y,x)\mapsto p_{Y\mid X}(y\mid x)$ and the
quantile functions $\cqrquant{\alo},\cqrquant{\ahi}$ are measurable.
Hence \Aref{assum:S_up}{\Sstar_{\mathrm{cqr}}} holds with $\ups=2\up$.

\emph{Proof of \ref{item:ex2-item-a2}.}
The tail control \eqref{supp:eq:cqr-tail} alone does not yield the
minorization: it bounds the conditional score density on the score
support at each covariate, whereas \eqref{eq:pdfloc_bound} concerns the
localized density on the quantile window $\I[u][\balpha][\Sstar]$, and
nothing guarantees that this window sits inside the common support of the
neighbouring conditional score laws.  We therefore verify the two
hypotheses of \Cref{lem:check-assum:pdfloc_bound} and conclude through
that lemma.

\emph{Step 1: H\"older continuity of the conditional score CDF.}
Fix $t\in\R$ and $u,u'\in\Xset$.  By \eqref{eq:ex-cqr-def},
$\CondCDF[\Sstar](t\mid u)
=\int_{[\cqrquant{\alo}(u)-t,\,\cqrquant{\ahi}(u)+t]\cap\Yset}
p_{Y\mid X}(y\mid u)\,\rmd y$, and similarly at $u'$.  Adding and
subtracting
$\int_{[\cqrquant{\alo}(u)-t,\,\cqrquant{\ahi}(u)+t]\cap\Yset}
p_{Y\mid X}(y\mid u')\,\rmd y$, and then using
\Cref{assumP:density-holder},
\begin{multline}
\label{supp:eq:cqr-condcdf-holder}
\bigl|\CondCDF[\Sstar](t\mid u)-\CondCDF[\Sstar](t\mid u')\bigr|
\le
\int_{\Yset}
    \bigl|p_{Y\mid X}(y\mid u)-p_{Y\mid X}(y\mid u')\bigr|\,\rmd y
\\
+
\up\,
\Bigl|
[\cqrquant{\alo}(u)-t,\,\cqrquant{\ahi}(u)+t]\,\triangle\,
[\cqrquant{\alo}(u')-t,\,\cqrquant{\ahi}(u')+t]
\Bigr| ,
\end{multline}
the second term bounding the integral of $p_{Y\mid X}(\cdot\mid u')\le\up$
over the symmetric difference of the two intervals.  The first term is at
most $L_p\norm{u-u'}^\beta$ by \Cref{assumP:density-holder}.  The two sets
in the second term are (possibly empty) intervals whose left endpoints
differ by $\abs{\cqrquant{\alo}(u)-\cqrquant{\alo}(u')}$ and whose right
endpoints differ by $\abs{\cqrquant{\ahi}(u)-\cqrquant{\ahi}(u')}$, so the
Lebesgue measure of their symmetric difference is at most
\[
\abs{\cqrquant{\alo}(u)-\cqrquant{\alo}(u')}
+
\abs{\cqrquant{\ahi}(u)-\cqrquant{\ahi}(u')} .
\]
Moreover, for $\tau\in\{\alo,\ahi\}$: both quantiles
$\cqrquant{\tau}(u),\cqrquant{\tau}(u')$ lie in $[-M,M]$, and
$F_{Y\mid X}(\cqrquant{\tau}(v)\mid v)=\tau$ for $v\in\{u,u'\}$
(\Cref{supp:lem:cqr-density}); the uniform density lower bound
$\uniflow$ on $[-M,M]$, this identity, and
\Cref{assumP:density-holder} then give
\begin{align}
\label{supp:eq:cqr-quant-holder}
\uniflow\,\abs{\cqrquant{\tau}(u)-\cqrquant{\tau}(u')}
&\le
\bigl|F_{Y\mid X}(\cqrquant{\tau}(u)\mid u')
-F_{Y\mid X}(\cqrquant{\tau}(u')\mid u')\bigr|
\notag\\
&=
\bigl|F_{Y\mid X}(\cqrquant{\tau}(u)\mid u')
-F_{Y\mid X}(\cqrquant{\tau}(u)\mid u)\bigr|
\le
L_p\norm{u-u'}^{\beta} ,
\end{align}
the first inequality because $F_{Y\mid X}(\cdot\mid u')$ has density at
least $\uniflow$ between the two quantiles, and the last because a
difference of distribution functions at a common point is bounded by the
$L^1$ distance of the densities.  Combining the last two displays with
\eqref{supp:eq:cqr-condcdf-holder} and taking the supremum over $t\in\R$
yields
\[
\sup_{t\in\R}
\bigl|\CondCDF[\Sstar](t\mid u)-\CondCDF[\Sstar](t\mid u')\bigr|
\le
L_p\Bigl(1+2\frac{\up}{\uniflow}\Bigr)\norm{u-u'}^{\beta},
\]
which is \Aref{assum:lip_condcdf}{\Sstar_{\mathrm{cqr}}} with
$L_S=L_p(1+2\up/\uniflow)$.

\emph{Step 2: shape-function lower bound.}
With the shape function $\shapefun_{\mathrm{res}}=\indi{(0,1)}$ of
\Cref{prop:ex-residual-score}, we claim that, for every $z\in\Xset$ and
$s\in\R$,
\begin{equation}
\label{supp:eq:cqr-shape-lb}
\CondPDF[\Sstar](s\mid z)
\ge
\uniflow\,
\shapefun_{\mathrm{res}}\bigl(\CondCDF[\Sstar](s\mid z)\bigr),
\end{equation}
that is, \eqref{eq:lowerbound_sym} holds with
$\sigma_{\shapefun_{\mathrm{res}}}(u,h)=\uniflow$, uniformly in $(u,h)$.
If $\CondCDF[\Sstar](s\mid z)\in\{0,1\}$, the right-hand side vanishes
and there is nothing to prove.  Otherwise
$\CondCDF[\Sstar](s\mid z)\in\ooint{0,1}$: then
$s\ge s^{(\Sstar)}_-(z)$, since the CDF vanishes below
$s^{(\Sstar)}_-(z)$, and $s\le s^{(\Sstar)}_+(z)$, since
$\Sstar_{\mathrm{cqr}}(z,Y)\le s^{(\Sstar)}_+(z)$ almost surely
(\Cref{supp:lem:cqr-density}) forces $\CondCDF[\Sstar](s\mid z)=1$ for
$s\ge s^{(\Sstar)}_+(z)$.  The tail control \eqref{supp:eq:cqr-tail}
then gives $\CondPDF[\Sstar](s\mid z)\ge\uniflow$, which is
\eqref{supp:eq:cqr-shape-lb}.

\emph{Step 3: conclusion.}
By \ref{item:ex2-item-a1} and Step~1,
\Aref{assum:S_up}{\Sstar_{\mathrm{cqr}}} and
\Aref{assum:lip_condcdf}{\Sstar_{\mathrm{cqr}}} hold, and Step~2 supplies
\eqref{eq:lowerbound_sym}; hence \Cref{lem:check-assum:pdfloc_bound}
applies with $\shapefun\leftarrow\shapefun_{\mathrm{res}}$ and
$\sigma_{\shapefun_{\mathrm{res}}}(u,h)\leftarrow\uniflow$---a single
constant, uniform in $(u,h)$, which is the payoff of the uniform lower
bound in \Cref{assumP:density-unif}.  For the margin levels
$\alphalow,\alphaup$ of \eqref{eq:definition-margin_hmax}, every
$u\in\Xset$ and every $0<h\le1\wedge\bwalpha$, with $\bwalpha$ given by
\eqref{eq:definition-hmax_window} for the constant $L_S$ of Step~1, it
yields
\[
\lows[\Sstar_{\mathrm{cqr}}]{u}
\ge\uniflow\,\inf_{r\in\J}\shapefun_{\mathrm{res}}(r)
=\uniflow ,
\]
the final equality because $\J\subset\ooint{0,1}$ by
\eqref{supp:eq:J-containment}, so that
$\shapefun_{\mathrm{res}}\equiv1$ on $\J$.

\emph{Proof of \ref{item:ex2-item-a4}.}
For $q<s^{(\Sstar)}_-(x)$ we have
$\conformalset[\Sstar]{x}{q}=\emptyset$, while otherwise
\[
\conformalset[\Sstar]{x}{q}
=\Yset\cap
\bigl[\cqrquant{\alo}(x)-q,\;\cqrquant{\ahi}(x)+q\bigr].
\]
Let $q\le q'$.  The case $q\le q'<s^{(\Sstar)}_-(x)$ is trivial.  If
$q<s^{(\Sstar)}_-(x)\le q'$, then
\[
\bigl|\conformalset[\Sstar]{x}{q'}\setminus
\conformalset[\Sstar]{x}{q}\bigr|
=
\bigl|\conformalset[\Sstar]{x}{q'}\bigr|
\le
\cqrquant{\ahi}(x)-\cqrquant{\alo}(x)+2q'
=
2\bigl(q'-s^{(\Sstar)}_-(x)\bigr)
\le
2(q'-q).
\]
Otherwise $q'\ge q\ge s^{(\Sstar)}_-(x)$, the two sets are nested, and
\[
\conformalset[\Sstar]{x}{q'}\setminus\conformalset[\Sstar]{x}{q}
\subseteq
\bigl[\cqrquant{\alo}(x)-q',\,\cqrquant{\alo}(x)-q\bigr]
\cup
\bigl[\cqrquant{\ahi}(x)+q,\,\cqrquant{\ahi}(x)+q'\bigr],
\]
so that
$\bigl|\conformalset[\Sstar]{x}{q'}\setminus
\conformalset[\Sstar]{x}{q}\bigr|\le2(q'-q)$.  Hence
\Aref{assum:S_len}{\Sstar_{\mathrm{cqr}}} holds with $\Llen[x]=2$,
uniformly in $x$.

\emph{Proof of \ref{item:ex2-item-a5}.}
Evaluating \eqref{supp:eq:cqr-cdf} at $t=0$ gives
$\CondCDF[\Sstar](0\mid x)=1-\alpha$ for \emph{every} $x\in\Xset$; in
particular \eqref{eq:pivotality} holds for $P_X$-almost every $x$, and
\Aref{assum:S_pivot_quantile}{\alpha} holds with pivot $\taustar=0$.
\end{proof}

\paragraph{Unbounded responses under a local density floor.}
In what follows, the response is now allowed to be unbounded, and the density
floor is required only on fixed neighbourhoods of the two target quantiles;
the global upper density bound is retained.  This relaxation admits Gaussian,
and more generally location--scale, conditional laws whenever their scales
and base densities give uniform constants, whereas the common-support
assumption (\Cref{assumP:density-unif}) excluded them.
\Cref{assumP:marginal-density}, including the
design-density floor $p_X\ge\pmin$, and \Cref{assumP:density-holder} remain in
force.  Under the setting below, $[0,1]^d$ is $(2^{-d},1)$-regular, so the
geometric part of \Cref{assumP:marginal-density} is automatic.
Fix $\alpha\in\ooint{0,1}$ and $0<\alo<\ahi<1$ with $\ahi-\alo=1-\alpha$.

\begingroup
\setcounter{assumP}{1}
\renewcommand{\theassumP}{2\ensuremath{''}}
\def\theHassumP{2doubleprime}
\begin{assumP}[Relaxed conditional density]
\label{assumP:density-local}
The covariate and response spaces are $\Xset=[0,1]^d$ and $\Yset=\R$.
There exist constants $0<\uniflow\le\up<\infty$ and $\rhofloor>0$, and a
jointly measurable version of the conditional Lebesgue density
$p_{Y\mid X}$, such that, for $\DC[X]$-almost every $x\in\Xset$,
\[
 p_{Y\mid X}(y\mid x)\le\up,
 \qquad y\in\R,
\]
and, for each $\tau\in\{\alo,\ahi\}$,
\[
 p_{Y\mid X}(y\mid x)\ge\uniflow
 \qquad\text{whenever}\qquad
 \abs{y-\cqrquant{\tau}(x)}\le\rhofloor,
\]
where the left-quantile functions $\cqrquant{\tau}$ are those in
\eqref{eq:ex-cqr-def}.
\end{assumP}
\setcounter{assumP}{3}
\endgroup

Throughout this part, fix a Borel set $\Xzero\subseteq\Xset$ of full
$\DC[X]$-measure on which the two displayed bounds of
\Cref{assumP:density-local} hold, and set
\begin{equation}
\label{supp:eq:rstar-def}
 \rstar:=\rhofloor\wedge\frac{1-\alpha}{2\up}.
\end{equation}
The choice of conditional version off $\Xzero$ is immaterial: replacing
$p_{Y\mid X}(\cdot\mid x)$ by a fixed probability density for
$x\notin\Xzero$ changes neither the joint law of $(X,Y)$ nor any
$\DC[X]$-almost-everywhere statement below, so a regular conditional
version may be chosen for which the bounds established on $\Xzero$ hold
with the stated constants.

\begin{proposition}[The CQR score under the relaxed density assumption]
\label{prop:ex-cqr-local}
Assume \Cref{assumP:marginal-density,assumP:density-holder},
\Cref{assumK:kernel_assum}, and
\Cref{assumP:density-local}.  Let
$\alpha\in\ooint{0,1}$ and let $\Sstar=\Sstar_{\mathrm{cqr}}$ be defined in
\eqref{eq:ex-cqr-def}.  Then:
\begin{enumerate}[label=(\roman*),leftmargin=0pt]
\item\label{item:ex2loc-a1}
\Aref{assum:S_up}{\Sstar_{\mathrm{cqr}}} holds with
$\ups=2\,\up$.

\item\label{item:ex2loc-a2}
For any $0<\alphalow<\alpha<\alphaup<1$ and
$\balpha=(\alpha,\alphalow,\alphaup)$ satisfying the strict cap
\[
 (\alpha-\alphalow)\vee(\alphaup-\alpha)
 <2\uniflow\,\rstar,
\]
the localized density minorization \eqref{eq:pdfloc_bound} holds with
$\lows[\Sstar_{\mathrm{cqr}}]{u}\ge2\uniflow$ for every $u\in\Xset$ and
every $0<h\le1$.

\item\label{item:ex2loc-a4}
\Aref{assum:S_len}{\Sstar_{\mathrm{cqr}}} holds with $\Llen[x]=2$ for every
$x\in\Xset$.

\item\label{item:ex2loc-a5}
\Aref{assum:S_pivot_quantile}{\alpha} holds with pivot $\taustar=0$.
\end{enumerate}
\end{proposition}

\begingroup
\renewcommand{\thelemma}{\arabic{lemma}\ensuremath{'}}
\def\theHlemma{supp-cqr-2prime}
\begin{lemma}[CQR score law and local floor]
\label{supp:lem:cqr-density-local}
Suppose \Cref{assumP:density-local}.  Let
$\alpha\in\ooint{0,1}$ and $0<\alo<\ahi<1$ with
$\ahi-\alo=1-\alpha$, and let $\Sstar=\Sstar_{\mathrm{cqr}}$, with the CQR
score and quantile functions defined in \eqref{eq:ex-cqr-def}.  Recall the
notation
\[
 s^{(\Sstar)}_-(x)
 :=-\tfrac12\bigl(\cqrquant{\ahi}(x)-\cqrquant{\alo}(x)\bigr).
\]
Then, for every $x\in\Xzero$, the following statements hold.
\begin{enumerate}[label=(\alph*),leftmargin=0pt]
\item\label{supp:item:cqr-quantile-identity}
For each $\tau\in\{\alo,\ahi\}$,
\begin{equation}
\label{supp:eq:cqr-quantile-identity}
 F_{Y\mid X}(\cqrquant{\tau}(x)\mid x)=\tau.
\end{equation}

\item\label{supp:item:cqr-gap}
The two target quantiles satisfy
\begin{equation}
\label{supp:eq:cqr-gap}
 \cqrquant{\ahi}(x)-\cqrquant{\alo}(x)
 \ge \frac{1-\alpha}{\up},
 \qquad
 s^{(\Sstar)}_-(x)\le-\frac{1-\alpha}{2\up}.
\end{equation}

\item\label{supp:item:cqr-score-law}
For $t\ge s^{(\Sstar)}_-(x)$,
\begin{equation}
\label{supp:eq:cqr-cdf-local}
 \CondCDF[\Sstar](t\mid x)
 =
 \int_{\cqrquant{\alo}(x)-t}^{\cqrquant{\ahi}(x)+t}
 p_{Y\mid X}(y\mid x)\,\rmd y,
\end{equation}
while $\CondCDF[\Sstar](t\mid x)=0$ for
$t<s^{(\Sstar)}_-(x)$.  The conditional score law is absolutely continuous
with density version
\begin{equation}
\label{supp:eq:cqr-density-local}
 \CondPDF[\Sstar](s\mid x)
 =
 \indiacc{s\ge s^{(\Sstar)}_-(x)}
 \Bigl[
 p_{Y\mid X}(s+\cqrquant{\ahi}(x)\mid x)
 +p_{Y\mid X}(\cqrquant{\alo}(x)-s\mid x)
 \Bigr],
 \qquad s\in\R.
\end{equation}
This version is jointly measurable whenever $p_{Y\mid X}$ and the two
quantile maps are measurable.

\item\label{supp:item:cqr-local-floor}
With $\rstar$ defined in \eqref{supp:eq:rstar-def},
\begin{equation}
\label{supp:eq:cqr-local-floor}
 \CondPDF[\Sstar](s\mid x)\ge2\uniflow,
 \qquad s\in[-\rstar,\rstar].
\end{equation}
\end{enumerate}
\end{lemma}
\endgroup
\addtocounter{theorem}{-1}

\begin{proof}
Fix $x\in\Xzero$.

\emph{Proof of \ref{supp:item:cqr-quantile-identity}.}
The conditional law is absolutely continuous under
\Cref{assumP:density-local}, so $F_{Y\mid X}(\cdot\mid x)$ is continuous on
$\R$.  Fix $\tau\in\{\alo,\ahi\}$ and write
$q=\cqrquant{\tau}(x)$.  Right continuity of a CDF at its left quantile gives
$F_{Y\mid X}(q\mid x)\ge\tau$.  By the definition in
\eqref{eq:ex-cqr-def}, $F_{Y\mid X}(z\mid x)<\tau$ for every $z<q$; taking
$z\uparrow q$ and using continuity gives $F_{Y\mid X}(q\mid x)\le\tau$.
Thus equality holds, proving \eqref{supp:eq:cqr-quantile-identity}.  No
strict monotonicity of the conditional CDF is used.

\emph{Proof of \ref{supp:item:cqr-gap}.}
Monotonicity of the left-quantile map gives
$\cqrquant{\alo}(x)\le\cqrquant{\ahi}(x)$.  By
\eqref{supp:eq:cqr-quantile-identity} and absolute continuity,
\begin{align*}
 1-\alpha
 &=\ahi-\alo
 =\int_{\cqrquant{\alo}(x)}^{\cqrquant{\ahi}(x)}
 p_{Y\mid X}(y\mid x)\,\rmd y
 \\
 &\le
 \up\bigl(\cqrquant{\ahi}(x)-\cqrquant{\alo}(x)\bigr).
\end{align*}
The first inequality in \eqref{supp:eq:cqr-gap} follows, and the second is
its immediate consequence through the retained definition of
$s^{(\Sstar)}_-(x)$.

\emph{Proof of \ref{supp:item:cqr-score-law}.}
By \eqref{eq:ex-cqr-def}, for every $y,t\in\R$,
$\Sstar_{\mathrm{cqr}}(x,y)\le t$ if and only if
\[
 \cqrquant{\alo}(x)-t\le y\le\cqrquant{\ahi}(x)+t.
\]
Moreover,
\[
 \Sstar_{\mathrm{cqr}}(x,y)
 \ge\tfrac12\bigl[(\cqrquant{\alo}(x)-y)
 +(y-\cqrquant{\ahi}(x))\bigr]
 =s^{(\Sstar)}_-(x).
\]
Hence the score event is empty for $t<s^{(\Sstar)}_-(x)$, while for
$t\ge s^{(\Sstar)}_-(x)$ it is exactly the interval in
\eqref{supp:eq:cqr-cdf-local}, which proves that formula.  Splitting that integral
at $\tfrac12(\cqrquant{\alo}(x)+\cqrquant{\ahi}(x))$, and using the
substitutions $y=\cqrquant{\alo}(x)-s$ and
$y=\cqrquant{\ahi}(x)+s$, gives
\[
 \CondCDF[\Sstar](t\mid x)
 =\int_{s^{(\Sstar)}_-(x)}^t
 \Bigl[
 p_{Y\mid X}(\cqrquant{\alo}(x)-s\mid x)
 +p_{Y\mid X}(s+\cqrquant{\ahi}(x)\mid x)
 \Bigr]\,\rmd s
\]
for every $t\ge s^{(\Sstar)}_-(x)$.  Together with the vanishing CDF below
that point, this proves absolute continuity and the density version
\eqref{supp:eq:cqr-density-local}.  There are no truncation indicators on
$\Yset=\R$: both branches are present for every
$s\ge s^{(\Sstar)}_-(x)$.  Joint measurability follows by composition from
the joint measurability of $p_{Y\mid X}$ and measurability of the quantile
maps.  The latter is automatic here: joint measurability of the density
makes $(y,x)\mapsto F_{Y\mid X}(y\mid x)$ jointly measurable.  For each
$\tau\in\{\alo,\ahi\}$ and every $a\in\R$,
\[
 \{x:\cqrquant{\tau}(x)<a\}
 =\bigcup_{v\in\mathbb{Q},\,v<a}
   \{x:F_{Y\mid X}(v\mid x)\ge\tau\}.
\]

\emph{Proof of \ref{supp:item:cqr-local-floor}.}
Let $s\in[-\rstar,\rstar]$.  By the definition \eqref{supp:eq:rstar-def} of
$\rstar$ and \eqref{supp:eq:cqr-gap},
\[
 s\ge-\rstar\ge-\frac{1-\alpha}{2\up}
 \ge s^{(\Sstar)}_-(x),
\]
so the indicator in \eqref{supp:eq:cqr-density-local} is active.  Moreover,
\[
 \abs{(s+\cqrquant{\ahi}(x))-\cqrquant{\ahi}(x)}
 =\abs{s}\le\rhofloor,
 \qquad
 \abs{(\cqrquant{\alo}(x)-s)-\cqrquant{\alo}(x)}
 =\abs{s}\le\rhofloor.
\]
Each branch in \eqref{supp:eq:cqr-density-local} is therefore at least
$\uniflow$ by \Cref{assumP:density-local}, which yields
\eqref{supp:eq:cqr-local-floor}.  This is exactly where the two-neighbourhood
structure is used; no density floor between the quantiles is needed.
\end{proof}

\begin{proof}[Proof of \Cref{prop:ex-cqr-local}]
We prove the four items in the order
\ref{item:ex2loc-a1}, \ref{item:ex2loc-a5},
\ref{item:ex2loc-a4}, \ref{item:ex2loc-a2}.  The rationale for proving the
H\"older property first in the proof of \Cref{prop:ex-cqr}---its constant
entered the bandwidth cap $\bwalpha$---no longer applies.

\emph{Proof of \ref{item:ex2loc-a1}.}
For every $x\in\Xzero$, \Cref{supp:lem:cqr-density-local} gives the density
version
\eqref{supp:eq:cqr-density-local}.  It is a sum of two branches, each bounded by
$\up$ under \Cref{assumP:density-local}, and is jointly measurable whenever
$p_{Y\mid X}$ and the quantile maps are measurable.  Thus the score-density
bound is $2\up$ on the $\DC[X]$-full set $\Xzero$.  With the immaterial
null-set choice of regular conditional version described after
\Cref{assumP:density-local}, this is
\Aref{assum:S_up}{\Sstar_{\mathrm{cqr}}} with $\ups=2\up$.

\emph{Proof of \ref{item:ex2loc-a5}.}
For every $x\in\Xzero$, \eqref{supp:eq:cqr-cdf-local} at $t=0$ and
\eqref{supp:eq:cqr-quantile-identity} give
\[
 \CondCDF[\Sstar](0\mid x)
 =\int_{\cqrquant{\alo}(x)}^{\cqrquant{\ahi}(x)}
 p_{Y\mid X}(y\mid x)\,\rmd y
 =\ahi-\alo=1-\alpha.
\]
Since $\Xzero$ is $\DC[X]$-full, this is precisely the almost-everywhere
identity in \Cref{def:pivotality}; hence
\Aref{assum:S_pivot_quantile}{\alpha} holds with pivot $\taustar=0$.

\emph{Proof of \ref{item:ex2loc-a4}.}
Fix $x\in\Xset$.  The two left quantiles in \eqref{eq:ex-cqr-def} are
finite and ordered because $F_{Y\mid X}(\cdot\mid x)$ is a proper CDF and
$\alo<\ahi$.  For $q<s^{(\Sstar)}_-(x)$,
$\conformalset[\Sstar]{x}{q}=\emptyset$, whereas for
$q\ge s^{(\Sstar)}_-(x)$,
\[
 \conformalset[\Sstar]{x}{q}
 =\bigl[\cqrquant{\alo}(x)-q,\;\cqrquant{\ahi}(x)+q\bigr]
\]
and
\[
 \bigl|\conformalset[\Sstar]{x}{q}\bigr|
 =\cqrquant{\ahi}(x)-\cqrquant{\alo}(x)+2q
 =2\bigl(q-s^{(\Sstar)}_-(x)\bigr).
\]
Let $q\le q'$.  If both are below $s^{(\Sstar)}_-(x)$, both sets are empty.
If $q<s^{(\Sstar)}_-(x)\le q'$, then
\[
 \bigl|\conformalset[\Sstar]{x}{q'}\bigr|
 =2\bigl(q'-s^{(\Sstar)}_-(x)\bigr)
 \le2(q'-q).
\]
If $s^{(\Sstar)}_-(x)\le q\le q'$, the difference of the two lengths is
exactly $2(q'-q)$.  These three cases prove
\Aref{assum:S_len}{\Sstar_{\mathrm{cqr}}} with $\Llen[x]=2$.

\emph{Proof of \ref{item:ex2loc-a2}.}
Apply \Cref{supp:lem:pivotal-local-floor} with
$\pfloor=2\uniflow$, $r=\rstar$, and $\taustar=0$.
The conditional score-density floor is
\eqref{supp:eq:cqr-local-floor}, and pivotality is
\ref{item:ex2loc-a5}.  The strict margin cap in the proposition is exactly
the cap required by that lemma.  It follows that, for every $u\in\Xset$ and
every $0<h\le1$,
\[
 \lows[\Sstar_{\mathrm{cqr}}]{u}\ge2\uniflow.
\]
Unlike the proof of \Cref{prop:ex-cqr}\ref{item:ex2-item-a2}, this argument
requires neither a bandwidth cap nor the H\"older input
\Aref{assum:lip_condcdf}{\Sstar_{\mathrm{cqr}}}; it is structurally parallel
to \Cref{prop:ex-pit-standard}\ref{item:ex3-item-a2}.
\end{proof}

\begin{lemma}[Pivotal local-floor minorization]
\label{supp:lem:pivotal-local-floor}
Let $\alpha\in\ooint{0,1}$, $\alphalow\in\ooint{0,\alpha}$,
$\alphaup\in\ooint{\alpha,1}$, and
$\balpha=(\alpha,\alphalow,\alphaup)$.  Suppose
\Cref{assumP:marginal-density,assumK:kernel_assum},
\Aref{assum:S_up}{\Sstar}, and
\Aref{assum:S_pivot_quantile}{\alpha}.  Assume that there exist
$\pfloor>0$ and $r>0$ such that, for $\DC[X]$-almost every $x\in\Xset$,
\[
 \CondPDF[\Sstar](s\mid x)\ge\pfloor
 \qquad\text{for Lebesgue-almost every }
 s\in[\taustar-r,\taustar+r].
\]
If the margins satisfy the strict cap
\[
 (\alpha-\alphalow)\vee(\alphaup-\alpha)<\pfloor r,
\]
then, for every $u\in\Xset$ and every $0<h\le1$,
\[
 \I[u][\balpha][\Sstar]
 \subseteq
 \left[
  \taustar-\frac{\alphaup-\alpha}{\pfloor},\;
  \taustar+\frac{\alpha-\alphalow}{\pfloor}
 \right]
 \subset
 \ooint{\taustar-r,\taustar+r},
 \qquad
 \lows[\Sstar]{u}\ge\pfloor.
\]
\end{lemma}

\begin{proof}
Fix $u\in\Xset$ and $0<h\le1$.

\emph{Step 1: the localized law.}
By \Cref{lem:gamma_bound_exact},
$\locnorm[u]\ge c_0\pmin h^d>0$; this is the only point at which
$h\le1$ enters.  Hence the localized law $\locmeasure[{u,h}]$ is a
well-defined probability measure.  Its definition gives
$\locmeasure[{u,h}]\ll\DC[X]$, so every $\DC[X]$-null set is also
$\locmeasure[{u,h}]$-null.

\emph{Step 2: absolute continuity and the pivotal level.}
By \Aref{assum:S_up}{\Sstar}, the conditional score laws admit a jointly
measurable density version.  Tonelli's theorem therefore gives, for every
$t\in\R$,
\begin{align*}
 \CDFloc[\Sstar]{u}(t)
 &=\int_{\Xset}\CondCDF[\Sstar](t\mid x)\,
      \locmeasure[{u,h}](\rmd x) \\
 &=\int_{-\infty}^{t}
      \left\{\int_{\Xset}\CondPDF[\Sstar](s\mid x)\,
      \locmeasure[{u,h}](\rmd x)\right\}\rmd s
 =\int_{-\infty}^{t}\PDFloc[\Sstar]{u}(s)\,\rmd s.
\end{align*}
Thus $\CDFloc[\Sstar]{u}$ is absolutely continuous with density
$\PDFloc[\Sstar]{u}$.  Moreover, Step~1 of \Cref{lem:quant_pivo}, which
integrates the pivotality identity against $\locmeasure[{u,h}]$, yields
\[
 \CDFloc[\Sstar]{u}(\taustar)=1-\alpha.
\]

\emph{Step 3: transfer of the density floor.}
Consider the jointly measurable exceptional set
\[
 E:=\bigl\{(s,x)\in\R\times\Xset:
       \abs{s-\taustar}\le r,\ 
       \CondPDF[\Sstar](s\mid x)<\pfloor\bigr\}.
\]
For $\DC[X]$-almost every $x$, the $s$-section $E^x$ is Lebesgue-null by
hypothesis.  Tonelli's theorem gives
\[
 (\operatorname{Leb}\otimes\DC[X])(E)=0.
\]
Fubini's theorem then implies
that, for Lebesgue-almost every
$s\in[\taustar-r,\taustar+r]$, the $x$-section $E_s$ is
$\DC[X]$-null, and therefore $\locmeasure[{u,h}]$-null by Step~1.  For
such $s$,
\[
 \PDFloc[\Sstar]{u}(s)
 =\int_{\Xset}\CondPDF[\Sstar](s\mid x)\,
       \locmeasure[{u,h}](\rmd x)
 \ge\pfloor.
\]
Consequently, the localized density is at least $\pfloor$ for
Lebesgue-almost every $s$ throughout the stated window.

\emph{Step 4: CDF increments inside the floor window.}
For every $t\in[0,r]$, absolute continuity, Step~3, and the pivotal identity
give
\[
 \CDFloc[\Sstar]{u}(\taustar+t)
 =1-\alpha+
   \int_{\taustar}^{\taustar+t}\PDFloc[\Sstar]{u}(s)\,\rmd s
 \ge1-\alpha+\pfloor t,
\]
and
\[
 \CDFloc[\Sstar]{u}(\taustar-t)
 =1-\alpha-
   \int_{\taustar-t}^{\taustar}\PDFloc[\Sstar]{u}(s)\,\rmd s
 \le1-\alpha-\pfloor t.
\]

\emph{Step 5: the right endpoint of the quantile window.}
Set $t_0=(\alpha-\alphalow)/\pfloor$.  The strict margin cap gives
$t_0<r$, and Step~4 yields
\[
 \CDFloc[\Sstar]{u}(\taustar+t_0)
 \ge1-\alpha+\pfloor t_0=1-\alphalow.
\]
By the definition of the left quantile,
\[
 \quant{1-\alphalow}{\CDFloc[\Sstar]{u}}
 \le\taustar+t_0
 =\taustar+\frac{\alpha-\alphalow}{\pfloor}.
\]

\emph{Step 6: the left endpoint of the quantile window.}
Set $t_1=(\alphaup-\alpha)/\pfloor<r$.  For every $s\in\ocint{t_1,r}$,
Step~4 gives
\[
 \CDFloc[\Sstar]{u}(\taustar-s)
 \le1-\alpha-\pfloor s
 <1-\alpha-\pfloor t_1
 =1-\alphaup.
\]
For every $s>r$, monotonicity and the case $s=r$ give
\[
 \CDFloc[\Sstar]{u}(\taustar-s)
 \le\CDFloc[\Sstar]{u}(\taustar-r)
 \le1-\alpha-\pfloor r
 <1-\alphaup,
\]
where the last inequality again uses the strict cap.  It follows that
$\CDFloc[\Sstar]{u}(z)<1-\alphaup$ for every
$z<\taustar-t_1$.  The definition of the left quantile therefore gives
\[
 \quant{1-\alphaup}{\CDFloc[\Sstar]{u}}
 \ge\taustar-t_1
 =\taustar-\frac{\alphaup-\alpha}{\pfloor}.
\]

\emph{Step 7: conclusion.}
By the definition of the localized quantile window in the main paper, its
left and right endpoints are precisely the quantiles bounded in Steps~6 and
5, respectively.  Hence those steps give its inclusion in the displayed
closed interval in the statement.  Since both margin ratios are strictly
smaller than $r$, that closed interval is contained in
$\ooint{\taustar-r,\taustar+r}$.  Step~3 then shows that the essential
infimum in \eqref{eq:pdfloc_bound} is at least $\pfloor$, proving
$\lows[\Sstar]{u}\ge\pfloor$.
\end{proof}

\runinhead{Why strictness is necessary.}
The strict cap cannot in general be replaced by a non-strict one.  If
$\alphaup-\alpha=\pfloor r$, a localized CDF whose density is exactly
$\pfloor$ on $[\taustar-r,\taustar+r]$ and zero immediately to the left can
be flat at level $1-\alphaup$ there; its left quantile at that level may then
lie strictly below $\taustar-r$, outside the floor region, and the
minorization constant can be zero.

\section{Distributional Score: Proof of Proposition~\ref{prop:ex-pit-standard}}
\label{supp:sec:pit}

The distributional score is fully pivotal: its conditional law is the same
for every covariate value.  The first lemma makes this explicit; the
second is a general device converting the response-density bounds of
\Cref{assumP:density} into the length regularity
\Aref{assum:S_len}{S}.

\begin{lemma}[Distributional score density]
\label{supp:lem:pit-density}
Suppose \Cref{assumP:density}.  Let $\Sstar$ be the distributional score
defined in \eqref{eq:ex-pit-def}.  Then, for every $x\in\Xset$, the
conditional law of $\Sstar(X,Y)$ given $X=x$ is absolutely continuous:
for every $t\ge0$,
\begin{equation}
\label{supp:eq:pit-cdf}
    \CondCDF[\Sstar](t\mid x)=\min\{2t,\,1\},
\end{equation}
while $\CondCDF[\Sstar](t\mid x)=0$ for $t<0$, with density version
\begin{equation}
\label{supp:eq:pit-density}
    \CondPDF[\Sstar](s\mid x) = 2\,\indi{0\le s\le\frac12},
    \qquad s\in\R .
\end{equation}
In particular, the conditional law of $\Sstar(X,Y)$ given $X=x$ is the
uniform distribution on $\ccint{0,\frac12}$ and does not depend on
$x\in\Xset$.
\end{lemma}

\begin{proof}
Fix $x\in\Xset$.  Under \Cref{assumP:density}, the conditional
distribution function $F_{Y\mid X}(\cdot\mid x)$ is continuous and
strictly increasing on $[-M,M]$, with $F_{Y\mid X}(-M\mid x)=0$ and
$F_{Y\mid X}(M\mid x)=1$; in particular
$F_{Y\mid X}(\cqrquant{\tau}(x)\mid x)=\tau$ for every
$\tau\in\ooint{0,1}$, with $\cqrquant{\tau}$ the conditional quantile
function of \eqref{eq:ex-cqr-def}.

If $t<0$, then $\Sstar(x,Y)\ge0$ almost surely, so
$\CondCDF[\Sstar](t\mid x)=0$.

Fix $t\ge0$.  By the definition \eqref{eq:ex-pit-def} of $\Sstar$,
\[
\CondCDF[\Sstar](t\mid x)
= \int_{\Yset}
\indi{\frac12-t\,\le\,F_{Y\mid X}(y\mid x)\,\le\,\frac12+t}\;
p_{Y\mid X}(y\mid x)\,\rmd y .
\]
If $t\ge\tfrac12$, the indicator equals one for every $y\in\Yset$, since
$F_{Y\mid X}(y\mid x)\in[0,1]\subseteq[\tfrac12-t,\tfrac12+t]$, whence
$\CondCDF[\Sstar](t\mid x)=1$.  If $0\le t<\tfrac12$, then, by continuity
and strict monotonicity of $F_{Y\mid X}(\cdot\mid x)$ on $[-M,M]$, the
integration domain is the interval
$[\cqrquant{\frac12-t}(x),\,\cqrquant{\frac12+t}(x)]$: indeed,
$F_{Y\mid X}(y\mid x)\ge\tfrac12-t$ if and only if
$y\ge\cqrquant{\frac12-t}(x)$, and
$F_{Y\mid X}(y\mid x)\le\tfrac12+t$ if and only if
$y\le\cqrquant{\frac12+t}(x)$.  Therefore
\[
    \CondCDF[\Sstar](t\mid x)
= \int_{\cqrquant{\frac12-t}(x)}^{\cqrquant{\frac12+t}(x)}
p_{Y\mid X}(y\mid x)\,\rmd y
= \Bigl(\tfrac12+t\Bigr)-\Bigl(\tfrac12-t\Bigr)
= 2t .
\]
Combining the three cases proves \eqref{supp:eq:pit-cdf}; in particular
the conditional score law is absolutely continuous.  Differentiating
$t\mapsto\CondCDF[\Sstar](t\mid x)$ on $\coint{0,\frac12}$ yields
$\CondPDF[\Sstar](s\mid x)=2$ for almost every
$s\in\ccint{0,\frac12}$; since $\CondCDF[\Sstar](\cdot\mid x)$ vanishes
on $\ooint{-\infty,0}$ and is constant on $\coint{\frac12,\infty}$, the
density may be taken to vanish there, which is
\eqref{supp:eq:pit-density}.  Finally, \eqref{supp:eq:pit-cdf} is the
distribution function of the uniform law on $\ccint{0,\frac12}$, which
does not depend on $x$.
\end{proof}

\begin{lemma}[Length regularity from the response density]
\label{supp:lem:len-quant}
Assume \Cref{assumP:density} and \Aref{assum:S_up}{S}, for a measurable
score $S$.  Then \Aref{assum:S_len}{S} holds with
$\Llen[x]=\ups/\low$ for every $x\in\Xset$.
\end{lemma}

\begin{proof}
Fix $x\in\Xset$ and $q\le q'$ (the general case follows by exchanging the
roles of $q$ and $q'$).  Since
$\conformalset{x}{q}\subseteq\conformalset{x}{q'}$ by
\eqref{eq:conf_set}, and since the conditional density is bounded below by
$\low$ on $\Yset$ (\Cref{assumP:density}),
\begin{align*}
\Bigl|\,\abs{\conformalset{x}{q'}}-
        \abs{\conformalset{x}{q}}\,\Bigr|
&= \bigl|\conformalset{x}{q'}\setminus\conformalset{x}{q}\bigr|
 = \int_{\Yset}\indiacc{q<S(x,y)\le q'}\,\rmd y
\\
&\le \frac{1}{\low}
   \int_{\Yset}p_{Y\mid X}(y\mid x)\,
          \indiacc{q<S(x,y)\le q'}\,\rmd y \\
 &= \frac{\CondCDF(q'\mid x)-\CondCDF(q\mid x)}{\low} .
\end{align*}
By \Aref{assum:S_up}{S}, the conditional score law at $x$ is absolutely
continuous with density $\CondPDF(\cdot\mid x)\le\ups$, so
\[
  \CondCDF(q'\mid x)-\CondCDF(q\mid x)
  =\int_q^{q'}\CondPDF(s\mid x)\,\rmd s
  \le\ups\,(q'-q) .
\]
Combining the two displays proves \Aref{assum:S_len}{S} with
$\Llen[x]=\ups/\low$.
\end{proof}

\begin{repproposition}{prop:ex-pit-standard}
Assume \Cref{assumP:density,assumP:marginal-density} and
\Cref{assumK:kernel_assum}.  Let $\Yset=[-M,M]$,
$\alpha\in\ooint{0,1}$, and
let $\Sstar$ be the distributional score defined in \eqref{eq:ex-pit-def}.
Then:
\begin{enumerate}[label=(\roman*),leftmargin=0pt]
\item
\Aref{assum:S_up}{\Sstar} holds with
$\ups=2$.

\item
For every $0<\alphalow<\alpha<\alphaup<1$, the localized density
minorization holds with $\lows[\Sstar]{u}=2$ for every
$u\in\Xset$ and $0<h\le1$.

\item
\Aref{assum:S_len}{\Sstar} holds with $\Llen[x]=2/\low$ for every
$x\in\Xset$.

\item
\Aref{assum:S_pivot_quantile}{\alpha} holds with pivot
$\taustar=(1-\alpha)/2$.
\end{enumerate}
\end{repproposition}

\begin{proof}
Throughout, work with the density version \eqref{supp:eq:pit-density} of
\Cref{supp:lem:pit-density}.

\emph{Proof of \ref{item:ex3-item-a1}.}
The density version \eqref{supp:eq:pit-density} is bounded by $2$ and does
not depend on $x$, so it is jointly measurable; hence
\Aref{assum:S_up}{\Sstar} holds with $\ups=2$.

\emph{Proof of \ref{item:ex3-item-a2}.}
Fix $u\in\Xset$, $0<h\le1$ and
$0<\alphalow<\alpha<\alphaup<1$.  Since $h\le1$,
\Cref{lem:gamma_bound_exact} gives $\locnorm[u]>0$, so
$\locmeasure[u,h]$ of \eqref{eq:lambda} is a probability measure.
Because the conditional score law does not depend on the covariate
(\Cref{supp:lem:pit-density}), the localized mixture
\eqref{eq:localized-cdf} coincides with it:
\[
\CDFloc[\Sstar]{u}(t)
=\int_{\Xset}\CondCDF[\Sstar](t\mid v)\,\locmeasure[u,h](\rmd v)
=\CondCDF[\Sstar](t\mid u),
\qquad t\in\R ,
\]
which equals $\min\{2t,1\}$ for $t\ge0$ and $0$ for $t<0$ by
\eqref{supp:eq:pit-cdf}, with density version
$\PDFloc[\Sstar]{u}(t)=2\,\indi{0\le t\le\frac12}$ by
\eqref{eq:localized-pdf} and \eqref{supp:eq:pit-density}.
For any $p\in\ooint{0,1}$, the left $p$-quantile of this distribution
function equals $p/2\in\ooint{0,\frac12}$.  Hence
\[
\I[u][\balpha][\Sstar]
=\Bigl[\frac{1-\alphaup}{2},\;\frac{1-\alphalow}{2}\Bigr]
\subset\Bigl(0,\frac12\Bigr),
\]
an interval on which $\PDFloc[\Sstar]{u}$ is identically equal to $2$;
the essential infimum in \eqref{eq:pdfloc_bound} therefore equals
$\lows[\Sstar]{u}=2$.

\emph{Proof of \ref{item:ex3-item-a4}.}
By \ref{item:ex3-item-a1}, \Aref{assum:S_up}{\Sstar} holds with
$\ups=2$, and \Cref{assumP:density} supplies the lower density bound
$\low$.  \Cref{supp:lem:len-quant} applies for every $x\in\Xset$ and
yields \Aref{assum:S_len}{\Sstar} with $\Llen[x]=\ups/\low=2/\low$.

\emph{Proof of \ref{item:ex3-item-a5}.}
From \eqref{supp:eq:pit-cdf}, for every $x\in\Xset$,
\[
\CondCDF[\Sstar]\Bigl(\frac{1-\alpha}{2}\Bigm|x\Bigr)
=\min\{1-\alpha,\,1\}=1-\alpha .
\]
In particular \eqref{eq:pivotality} holds for $P_X$-almost every
$x\in\Xset$, and \Aref{assum:S_pivot_quantile}{\alpha} holds with pivot
$\taustar=(1-\alpha)/2$.
\end{proof}

\section{Lemmas for Fixed-Score Analysis}
\label{supp:sec:fixed}
\begin{replemma}{lem:gamma_bound_exact}
Assume \Cref{assumP:marginal-density,assumK:kernel_assum}. Then, for every
$\tilde{x}\in\Xset$ and every $h\in(0,1]$,
\begin{equation}
\label{eq:supp-gamma-bound-corr}
  \locnorm[\tilde{x}] \;\ge\; c_0\,\pmin\,h^d \,,
\end{equation}
where $c_0$ is the structural constant introduced after
\Cref{assumK:kernel_assum}.
\end{replemma}

\begin{proof}
By the kernel lower bound and $p_X\ge\pmin$,
\begin{align*}
\locnorm[\tilde x]
&\ge \kappa\pmin
  \int_{\Xset}\indiacc{\|u-\tilde x\|\le\eta h}\,\rmd u \\
&=\kappa\pmin\,|\Xset\cap B(\tilde x,\eta h)|.
\end{align*}
Since $h\le1$ and $\eta\le r_0$, support regularity gives
\[
|\Xset\cap B(\tilde x,\eta h)|
\ge \creg |B(\tilde x,\eta h)|
=\creg\eta^d|B(0,1)|h^d.
\]
The conclusion follows from
$c_0=\kappa\creg\eta^d|B(0,1)|$.
\end{proof}

\begin{repproposition}{prop:calibration_event}
Assume \Cref{assumP:marginal-density,assumK:kernel_assum}. Let $S$ be a fixed
measurable score and write $S_i=S(X_i,Y_i)$. Fix
$\tilde{x}\in\Xset$, $h\in(0,1]$ and
$\delta\in(0,1)$. For $t\in\RR$, set
\[
 A_n(t)
 :=
 n^{-1}\sum_{i=1}^n
 \locweight[\tilde{x}]{X_i}\indiacc{S_i\le t},
 \qquad
 B_n
 :=
 n^{-1}\sum_{i=1}^n
 \locweight[\tilde{x}]{X_i},
\]
and
\[
 A(t)
 :=
 \mathbb E\!\left[
 \locweight[\tilde{x}]{X}\indiacc{S(X,Y)\le t}
 \right].
\]
Define the calibration event
\begin{equation}
\label{eq:supp-definition-Omega}
\begin{split}
  \SetOmega[{n,\tilde{x},h,\delta}]
  :=
  &\left\{
  \left|B_n-\locnorm[\tilde{x}]\right|
  \le
  \sqrt{
    \frac{2\supnorm{\KDE}\locnorm[\tilde{x}]\log(4/\delta)}{n}
  }
  +
  \frac{2\supnorm{\KDE}\log(4/\delta)}{3n}
  \right\}
  \\
  &\cap
  \left\{
  \sup_{t\in\RR}\left|A_n(t)-A(t)\right|
  \le
  8
  \sqrt{
    \frac{\supnorm{\KDE}\locnorm[\tilde{x}]}{n}
  }
  +
  \sqrt{
    \frac{2\supnorm{\KDE}\locnorm[\tilde{x}]\log(4/\delta)}{n}
  }
  +
  \frac{4\supnorm{\KDE}\log(4/\delta)}{3n}
  \right\}.
\end{split}
\end{equation}
Then
\[
  \PP\!\left(\SetOmega[{n,\tilde{x},h,\delta}]\right)
  \ge 1-\delta .
\]
Moreover, if $n\ge n_0(h,\delta)$, then, on
$\SetOmega[{n,\tilde{x},h,\delta}]$, the denominator
$\sum_{j\in[n]}\locweight[\tilde{x}]{X_j}$ is positive and the following two
bounds hold simultaneously:
\begin{equation}
\label{eq:supp-cal-weight-bound}
  \sup_{x\in\Xset}
  \nlocweightnorm[x]{\tilde{x}}{n+1}
  \le
  \rhom,
\end{equation}
and
\begin{equation}
\label{eq:supp-cal-cdf-bound}
  \sup_{t\in\RR}
  \left|
    \nBarCDFloc{\tilde{x}}(t)-\CDFloc{\tilde{x}}(t)
  \right|
  \le
  \varepsilon_{n,h}(\delta).
\end{equation}

\end{repproposition}

When the score is learned on an independent training fold, the proposition
is applied conditionally on that fold, the expectation defining $A(t)$
being computed under the conditional law; see
\eqref{eq:omega_adapt-definition}.

\begin{proof}[Proof of \Cref{prop:calibration_event}]
\emph{Step 1: deviation of $B_n$.}
The variables $\bigl(\locweight[\tilde{x}]{X_i}\bigr)_{i\in[n]}$ are \iid,
take values in $[0,\supnorm{\KDE}]$, and satisfy
$\E[\locweight[\tilde{x}]{X_1}]=\locnorm[\tilde{x}]$ by
\eqref{eq:normalizing-constant} and
$\operatorname{Var}\bigl(\locweight[\tilde{x}]{X_1}\bigr)
\le\E\bigl[\locweight[\tilde{x}]{X_1}^2\bigr]
\le\supnorm{\KDE}\,\locnorm[\tilde{x}]$.
Bernstein's inequality
\citep[Theorem~2.10]{boucheron2013concentration}, applied to
$\sum_{i=1}^n(\locweight[\tilde{x}]{X_i}-\locnorm[\tilde{x}])$ and to its
negative, gives, for each sign and every $t>0$,
\[
  \PP\Bigl(\pm\textstyle\sum_{i=1}^n
  \bigl(\locweight[\tilde{x}]{X_i}-\locnorm[\tilde{x}]\bigr)\ge t\Bigr)
  \le \exp\!\Bigl(-\frac{t^2}
  {2\bigl(n\supnorm{\KDE}\locnorm[\tilde{x}]
  +\supnorm{\KDE}\,t/3\bigr)}\Bigr).
\]
Equating the exponent to $-\log(4/\delta)$, solving the resulting quadratic
inequality in $t$, and using $\sqrt{a+b}\le\sqrt a+\sqrt b$ shows that each
one-sided tail exceeds
$\sqrt{2n\supnorm{\KDE}\locnorm[\tilde{x}]\log(4/\delta)}
+\tfrac23\supnorm{\KDE}\log(4/\delta)$ with
probability at most $\delta/4$; hence the first event in
\eqref{eq:supp-definition-Omega} fails with probability at most $\delta/2$.

\emph{Step 2: uniform deviation of $A_n$.}
Right continuity in $t$ of $A_n$ and $A$ shows that the suprema over
$t\in\RR$ and $t\in\mathbb{Q}$ coincide, so it suffices to control the
countable family of functions
$(v,y)\mapsto\locweight[\tilde{x}]{v}\indiacc{S(v,y)\le t}$,
$t\in\mathbb{Q}$.  First, by symmetrization with \iid\ Rademacher variables
$\epsilon_{1:n}$ independent of the data,
\[
\E\sup_{t\in\RR}|A_n(t)-A(t)|
\le
\frac2n\,
\E\sup_{t\in\RR}\Bigl|\sum_{i=1}^n\epsilon_i
\locweight[\tilde{x}]{X_i}\indiacc{S_i\le t}\Bigr|.
\]
Conditionally on the data, order the scores as
$S_{\pi(1)}\le\cdots\le S_{\pi(n)}$; the sets $\{i:S_i\le t\}$ are the
prefixes of this order, so the inner supremum equals
$\max_{0\le k\le n}
\bigl|\sum_{j\le k}\epsilon_{\pi(j)}\locweight[\tilde{x}]{X_{\pi(j)}}\bigr|$,
the maximal absolute value of a martingale in $k$, and Doob's $L^2$
inequality gives
\[
\E\Bigl[\max_{0\le k\le n}
\bigl|\sum_{j\le k}\epsilon_{\pi(j)}\locweight[\tilde{x}]{X_{\pi(j)}}\bigr|
\Bigm|X_{1:n},Y_{1:n}\Bigr]
\le2\Bigl(\sum_{i=1}^n\locweight[\tilde{x}]{X_i}^2\Bigr)^{1/2}.
\]
Since
$\E\bigl[\sum_{i=1}^n\locweight[\tilde{x}]{X_i}^2\bigr]
\le n\supnorm{\KDE}\locnorm[\tilde{x}]$, Jensen's inequality yields
\begin{equation}
\label{eq:expectation-bound-cal}
\E\sup_{t\in\RR}|A_n(t)-A(t)|
\le 4\sqrt{\supnorm{\KDE}\locnorm[\tilde{x}]/n}.
\end{equation}
Next apply Bousquet's version of Talagrand's inequality
\citep[Theorem~12.5]{boucheron2013concentration} to the countable class
\[
\left\{
 \frac{\xi}{\supnorm{\KDE}}
 \left(
  \locweight[\tilde{x}]{v}\indiacc{S(v,y)\le t}
  -\E[\locweight[\tilde{x}]{X}\indiacc{S(X,Y)\le t}]
 \right)
 :t\in\mathbb{Q},\ \xi\in\{-1,1\}
\right\},
\]
whose elements are centred and bounded above by one.  Their variances satisfy
\[
\frac{
 \E\bigl[\locweight[\tilde{x}]{X}^2\indiacc{S(X,Y)\le t}\bigr]
}{\supnorm{\KDE}^2}
\le \frac{\locnorm[\tilde{x}]}{\supnorm{\KDE}}.
\]
The inequality gives, with
probability at least $1-\delta/4$,
\begin{multline*}
n\sup_{t\in\RR}|A_n(t)-A(t)|
\le
n\,\E\sup_{t\in\RR}|A_n(t)-A(t)|
\\
+\sqrt{2\log(4/\delta)\Bigl(n\supnorm{\KDE}\locnorm[\tilde{x}]
+2\supnorm{\KDE}\,n\,\E\sup_{t\in\RR}|A_n(t)-A(t)|\Bigr)}
+\frac{\supnorm{\KDE}\log(4/\delta)}{3}.
\end{multline*}
Using $\sqrt{a+b}\le\sqrt a+\sqrt b$, then
$2\bigl\{\log(4/\delta)\supnorm{\KDE}\,
n\,\E\sup_{t\in\RR}|A_n(t)-A(t)|\bigr\}^{1/2}
\le n\,\E\sup_{t\in\RR}|A_n(t)-A(t)|+\supnorm{\KDE}\log(4/\delta)$,
and finally \eqref{eq:expectation-bound-cal},
\[
n\sup_{t\in\RR}|A_n(t)-A(t)|
\le
8\sqrt{n\supnorm{\KDE}\locnorm[\tilde{x}]}
+\sqrt{2n\supnorm{\KDE}\locnorm[\tilde{x}]\log(4/\delta)}
+\frac{4\supnorm{\KDE}\log(4/\delta)}{3}.
\]
Dividing by $n$ shows that the second event in \eqref{eq:supp-definition-Omega}
fails with probability at most $\delta/4$.  A union bound gives
$\PP(\SetOmega[{n,\tilde{x},h,\delta}])\ge1-3\delta/4\ge1-\delta$.

\emph{Step 3: the weight bound \eqref{eq:supp-cal-weight-bound}.}
Assume $n\ge n_0(h,\delta)$ and work on
$\SetOmega[{n,\tilde{x},h,\delta}]$.  Since $h\le1$,
\Cref{lem:gamma_bound_exact} gives
$\locnorm[\tilde{x}]\ge c_0\pmin h^d>0$, so \eqref{eq:definition_n0} yields
\[
\frac{\supnorm{\KDE}\log(4/\delta)}{n\locnorm[\tilde{x}]}
\le\frac{\supnorm{\KDE}\log(4/\delta)}{c_0\,n\,\pmin\,h^d}
\le\frac{1}{64}.
\]
Dividing the inequality of the first event by $\locnorm[\tilde{x}]$,
\[
\frac{|B_n-\locnorm[\tilde{x}]|}{\locnorm[\tilde{x}]}
\le\Bigl(\frac{2\supnorm{\KDE}\log(4/\delta)}
{n\locnorm[\tilde{x}]}\Bigr)^{1/2}
+\frac{2\supnorm{\KDE}\log(4/\delta)}{3n\locnorm[\tilde{x}]}
\le\frac{1}{\sqrt{32}}+\frac{1}{96}<\frac12 ,
\]
so $B_n\ge\locnorm[\tilde{x}]/2>0$; in particular the denominator
$\sum_{j\in[n]}\locweight[\tilde{x}]{X_j}=nB_n$ is positive.  For every
$x\in\Xset$, since $\locweight[\tilde{x}]{x}\le\supnorm{\KDE}$ and the
weights are nonnegative,
\[
\nlocweightnorm[x]{\tilde{x}}{n+1}
=\frac{\locweight[\tilde{x}]{x}}
      {\sum_{j\in[n]}\locweight[\tilde{x}]{X_j}+\locweight[\tilde{x}]{x}}
\le\frac{\supnorm{\KDE}}{nB_n}
\le\frac{2\supnorm{\KDE}}{n\locnorm[\tilde{x}]}
\le\frac{2\supnorm{\KDE}}{c_0\,n\,\pmin\,h^d}
\le\rhom ,
\]
the last inequality by \eqref{eq:definition-rhom}; this proves
\eqref{eq:supp-cal-weight-bound}.

\emph{Step 4: the Kolmogorov bound \eqref{eq:supp-cal-cdf-bound}.}
By the tower property and \eqref{eq:lambda},
$A(t)=\locnorm[\tilde{x}]\,\CDFloc{\tilde{x}}(t)$ for every $t\in\RR$.  On
the event, using $B_n\ge\locnorm[\tilde{x}]/2$ and
$A(t)/\locnorm[\tilde{x}]=\CDFloc{\tilde{x}}(t)\le1$,
\[
\bigl|\nBarCDFloc{\tilde{x}}(t)-\CDFloc{\tilde{x}}(t)\bigr|
=\left|\frac{A_n(t)-A(t)}{B_n}
+\frac{A(t)}{\locnorm[\tilde{x}]}\,
\frac{\locnorm[\tilde{x}]-B_n}{B_n}\right|
\le\frac{2}{\locnorm[\tilde{x}]}
\Bigl(\sup_{s\in\RR}|A_n(s)-A(s)|+|B_n-\locnorm[\tilde{x}]|\Bigr).
\]
Combining the deviation levels of the two events,
\[
\sup_{t\in\RR}\bigl|\nBarCDFloc{\tilde{x}}(t)-\CDFloc{\tilde{x}}(t)\bigr|
\le
16\sqrt{\frac{\supnorm{\KDE}}{n\locnorm[\tilde{x}]}}
+4\sqrt{\frac{2\supnorm{\KDE}\log(4/\delta)}{n\locnorm[\tilde{x}]}}
+\frac{4\supnorm{\KDE}\log(4/\delta)}{n\locnorm[\tilde{x}]} ,
\]
and $\locnorm[\tilde{x}]\ge c_0\pmin h^d$ bounds each term by the
corresponding term of \eqref{eq:definition-varepsilon}, using
$16\le16\sqrt2$, $4\sqrt2\le8$ and $4\le8$; this proves
\eqref{eq:supp-cal-cdf-bound}.
\end{proof}

\section{Intermediate Results for Fixed-Score Theorems}
\label{supp:sec:fixed-intermediate}
\begin{replemma}{lem:T1}
Let $\alpha\in\ooint{0,1}$ and
$\alphalow\in\ooint{0,\alpha}$, $\alphaup\in\ooint{\alpha,1}$, and
$\boldsymbol\alpha=(\alpha,\alphalow,\alphaup)$.
Assume \Cref{assumP:marginal-density,assumK:kernel_assum}
and \Aref{assum:S_up}{S}. Let
\(\delta\in(0,1)\), \(\tilde{x}\in\Xset\), \(0<h\le1\), and suppose
that \(\lows{\tilde x}>0\) and
\(n\ge n_1(h,\delta,\boldsymbol\alpha)\), with \(n_1\) as in
\eqref{eq:definition-n1}. Then, on the event
\(\SetOmega[n,\tilde{x},h,\delta]\) of \eqref{eq:supp-definition-Omega}, the
following holds simultaneously for every \(x\in\Xset\):
\begin{equation}
\label{eq:supp-rev-T1}
\abs{\rlcpq-\locoraclequant{\tilde{x}}}
\le
\frac{\varepsilon_{n,h}(\delta)+2(1-\alpha)\rhom}
     {\lows{\tilde{x}}},
\end{equation}
where \(\lows{\tilde{x}}\) is defined in \eqref{eq:pdfloc_bound}. Moreover,
\(\rlcpq\in\ccint{\min_{i\in[n]}S_i,\,\max_{i\in[n]}S_i}\) is finite.
\end{replemma}

\begin{proof}
Throughout, work on $\SetOmega[n,\tilde{x},h,\delta]$ and fix
$x\in\Xset$. Note that $\varepsilon_{n,h}(\delta)>0$ and
$\rhom>0$ by inspection of \eqref{eq:definition-rhom} and
\eqref{eq:definition-varepsilon}.

\emph{Step 1.} Since $h\le1$ and $n\ge n_0(h,\delta)$ by
\eqref{eq:rev-n2-consequences}, \Cref{prop:calibration_event} applies:
the denominator $\sum_{j\in[n]}\locweight[\tilde{x}]{X_j}$ is
positive, so $\nBarCDFloc{\tilde{x}}$ of \eqref{eq:bar-empirical-cdf}
is a well-defined CDF supported on $\{S_1,\dots,S_n\}$ and the weights
\eqref{eq:rlcp-weights} are well defined; moreover,
\eqref{eq:supp-cal-weight-bound} gives
$\nlocweightnorm[x]{\tilde{x}}{n+1}\le\rhom$ uniformly in
$x\in\Xset$, and \eqref{eq:supp-cal-cdf-bound} gives
$\supnorm{\nBarCDFloc{\tilde{x}}-\CDFloc{\tilde{x}}}
\le\varepsilon_{n,h}(\delta)$.

\emph{Step 2.} By \eqref{eq:rev-n2-consequences} and
\eqref{eq:supp-cal-weight-bound},

\begin{equation}
\label{eq:rev-weight-small}
\nlocweightnorm[x]{\tilde{x}}{n+1}
\le\rhom \leq \frac{\alpha}{2} .
\end{equation}
Hence the effective level
$p_{n,h}(x;\tilde{x})$ of \eqref{eq:effective-level} is well defined
and, since $\nlocweightnorm[x]{\tilde{x}}{n+1}<\alpha$,
\begin{equation}
\label{eq:rev-level-range}
1-\alpha
\le p_{n,h}(x;\tilde{x})
=\frac{1-\alpha}{1-\nlocweightnorm[x]{\tilde{x}}{n+1}}
<1 .
\end{equation}
By \eqref{eq:bar-emp-relation}, for every finite $t$,
$\nEmpCDFloc{\tilde{x}}(t)\ge1-\alpha$ if and only if
$\nBarCDFloc{\tilde{x}}(t)\ge p_{n,h}(x;\tilde{x})$. Both level sets
are nonempty subsets of $\R$: indeed
$\nBarCDFloc{\tilde{x}}(\max_{i\in[n]}S_i)=1$. Taking infima,
\begin{equation}
\label{eq:rev-quantile-identity}
\rlcpq
=\quant{p_{n,h}(x;\tilde{x})}{\nBarCDFloc{\tilde{x}}}
\in\ccint{\min_{i\in[n]}S_i,\,\max_{i\in[n]}S_i},
\end{equation}
because $\nBarCDFloc{\tilde{x}}$ is supported on
$\{S_1,\dots,S_n\}$ and $p_{n,h}(x;\tilde{x})\in(0,1)$ by
\eqref{eq:rev-level-range}. This proves
the finiteness claim. Moreover, by \eqref{eq:rev-weight-small}
$\nlocweightnorm[x]{\tilde{x}}{n+1}<1/2$ so
$(1-\nlocweightnorm[x]{\tilde{x}}{n+1})^{-1}\le2$, and
\begin{equation}
\label{eq:rev-levelshift}
0\le p_{n,h}(x;\tilde{x})-(1-\alpha)
=\frac{(1-\alpha)\,\nlocweightnorm[x]{\tilde{x}}{n+1}}
      {1-\nlocweightnorm[x]{\tilde{x}}{n+1}}
\le 2(1-\alpha)\,\nlocweightnorm[x]{\tilde{x}}{n+1}
\le 2(1-\alpha)\,\rhom .
\end{equation}

\emph{Step 3.} We bound $\varepsilon_{n,h}(\delta)$ by $\min(p_{n,h}(x;\tilde{x}),1- p_{n,h}(x;\tilde{x}))$. First, since
$1-\nlocweightnorm[x]{\tilde{x}}{n+1}\le1$,
\[
1-p_{n,h}(x;\tilde{x})
=\frac{\alpha-\nlocweightnorm[x]{\tilde{x}}{n+1}}
      {1-\nlocweightnorm[x]{\tilde{x}}{n+1}}
\ge\alpha-\nlocweightnorm[x]{\tilde{x}}{n+1} ,
\]
so that, combining with \eqref{eq:rev-weight-small} and $\alpha>0$,
\begin{equation}
\label{eq:bounds-varepsilon}
0<\varepsilon_{n,h}(\delta)
\le\frac{\alpha}{2}-\nlocweightnorm[x]{\tilde{x}}{n+1}
<\alpha-\nlocweightnorm[x]{\tilde{x}}{n+1}
\le 1-p_{n,h}(x;\tilde{x}) .
\end{equation}
Second, by \eqref{eq:rev-n1-consequences}, the definition
\eqref{eq:definition-I} of the level margin $m_{\boldsymbol\alpha}$, and
$\alphaup<1$,
\begin{equation}
\label{eq:bounds-varepsilon-p}
\varepsilon_{n,h}(\delta)
\le\varepsilon_{n,h}(\delta)+2(1-\alpha)\rhom
\le m_{\boldsymbol\alpha}
\le\alphaup-\alpha
<1-\alpha
\le p_{n,h}(x;\tilde{x}) ,
\end{equation}
the last inequality being \eqref{eq:rev-level-range}. Combining
\eqref{eq:bounds-varepsilon} and \eqref{eq:bounds-varepsilon-p},
\begin{equation}
\label{eq:bounds-varepsilon-min}
0<\varepsilon_{n,h}(\delta)
<\min\{p_{n,h}(x;\tilde{x}),\,1-p_{n,h}(x;\tilde{x})\} .
\end{equation}

\emph{Step 4.} Under \Aref{assum:S_up}{S}, Tonelli's theorem applied to
\eqref{eq:localized-cdf} shows that $\CDFloc{\tilde{x}}$ is absolutely
continuous with density $\PDFloc{\tilde{x}}$ of \eqref{eq:localized-pdf};
the assumed minorization $\lows{\tilde x}>0$ states that this density is
bounded below by $\lows{\tilde{x}}$ Lebesgue-a.e.\ on
$\I[\tilde{x}]
=\ccint{\locoraclequant{\tilde{x}}[1-\alphaup],\,
        \locoraclequant{\tilde{x}}[1-\alphalow]}$, and we
recall that
$\locoraclequant{\tilde{x}}=\quant{1-\alpha}{\CDFloc{\tilde{x}}}$.

By \eqref{eq:rev-levelshift},
\eqref{eq:rev-n1-consequences} and the definition
\eqref{eq:definition-I} of the level margin
$m_{\boldsymbol\alpha}=\min\{\alpha-\alphalow,\,\alphaup-\alpha\}$,
\begin{align}
\label{eq:rev-lambda-plus}
p_{n,h}(x;\tilde{x})+\varepsilon_{n,h}(\delta)
&\le 1-\alpha+2(1-\alpha)\rhom
   +\varepsilon_{n,h}(\delta)
\le 1-\alpha+m_{\boldsymbol\alpha}
\le 1-\alphalow,
\\
\label{eq:rev-lambda-minus}
p_{n,h}(x;\tilde{x})-\varepsilon_{n,h}(\delta)
&\ge 1-\alpha-\varepsilon_{n,h}(\delta)
\ge 1-\alpha-m_{\boldsymbol\alpha}
\ge 1-\alphaup,
\end{align}
where \eqref{eq:rev-lambda-plus} uses
$m_{\boldsymbol\alpha}\le\alpha-\alphalow$ and
\eqref{eq:rev-lambda-minus} uses
$\varepsilon_{n,h}(\delta)
\le\varepsilon_{n,h}(\delta)+2(1-\alpha)\rhom
\le m_{\boldsymbol\alpha}\le\alphaup-\alpha$.
Since the endpoints of $\I[\tilde{x}]$ are the quantiles of
$\CDFloc{\tilde{x}}$ at the levels $1-\alphaup$ and
$1-\alphalow$, monotonicity of the quantile function
applied to \eqref{eq:rev-lambda-plus}--\eqref{eq:rev-lambda-minus}
directly yields
\begin{equation}
\label{eq:rev-bracket}
\ccint{\quant{p_{n,h}(x;\tilde{x})
-\varepsilon_{n,h}(\delta)}{\CDFloc{\tilde{x}}},\,
       \quant{p_{n,h}(x;\tilde{x})
+\varepsilon_{n,h}(\delta)}{\CDFloc{\tilde{x}}}}
\subseteq\I[\tilde{x}].
\end{equation}

\emph{Step 5: conclusion.} All hypotheses of
\Cref{supp:lem:quantile_stability_rlcp}(B) hold with
$F\leftarrow\nBarCDFloc{\tilde{x}}$,
$G\leftarrow\CDFloc{\tilde{x}}$, $p\leftarrow p_{n,h}(x;\tilde{x})$,
$\zeta\leftarrow\varepsilon_{n,h}(\delta)$,
$I\leftarrow\I[\tilde{x}]$, $\underline g\leftarrow\lows{\tilde{x}}$: the
Kolmogorov bound is \eqref{eq:supp-cal-cdf-bound}, the level condition is
\eqref{eq:bounds-varepsilon-min}, and the quantile bracket is
\eqref{eq:rev-bracket}. Hence
\[
\abs{\quant{p_{n,h}(x;\tilde{x})}{\nBarCDFloc{\tilde{x}}}
     -\quant{p_{n,h}(x;\tilde{x})}{\CDFloc{\tilde{x}}}}
\le\frac{\varepsilon_{n,h}(\delta)}{\lows{\tilde{x}}} .
\]
Furthermore, since
\[
p_{n,h}(x;\tilde{x})-\varepsilon_{n,h}(\delta)
\le p_{n,h}(x;\tilde{x})
\le p_{n,h}(x;\tilde{x})+\varepsilon_{n,h}(\delta),
\]
monotonicity of the
quantile function and \eqref{eq:rev-bracket} give
$\quant{p_{n,h}(x;\tilde{x})}{\CDFloc{\tilde{x}}}\in\I[\tilde{x}]$;
also $\quant{1-\alpha}{\CDFloc{\tilde{x}}}
=\locoraclequant{\tilde{x}}\in\I[\tilde{x}]$, since
$1-\alpha\in[1-\alphaup,\,1-\alphalow]$ and the
endpoints of $\I[\tilde{x}]$ are the quantiles at the extreme
levels. \Cref{supp:lem:quantile_lipschitz_levels_rlcp}
(applied to $\CDFloc{\tilde{x}}$ on $\I[\tilde{x}]$) and
\eqref{eq:rev-levelshift} then give
\[
\abs{\quant{p_{n,h}(x;\tilde{x})}{\CDFloc{\tilde{x}}}
     -\quant{1-\alpha}{\CDFloc{\tilde{x}}}}
\le\frac{p_{n,h}(x;\tilde{x})-(1-\alpha)}{\lows{\tilde{x}}}
\le\frac{2(1-\alpha)\rhom}{\lows{\tilde{x}}} .
\]
The triangle inequality, together with
\eqref{eq:rev-quantile-identity}, concludes the proof of
\eqref{eq:supp-rev-T1}. Finally, the event
$\SetOmega[n,\tilde{x},h,\delta]$ does not depend on $x$, and the
dependence on $x$ of every quantity used above
(namely $\nlocweightnorm[x]{\tilde{x}}{n+1}$, the effective level
$p_{n,h}(x;\tilde{x})$, and the two levels
$p_{n,h}(x;\tilde{x})\pm\varepsilon_{n,h}(\delta)$) enters
only through $\nlocweightnorm[x]{\tilde{x}}{n+1}$, which the proof
controls exclusively via the bounds
$0\le\nlocweightnorm[x]{\tilde{x}}{n+1}\le\rhom$ of
\Cref{prop:calibration_event}, valid uniformly in $x\in\Xset$; hence the
conclusion holds simultaneously for every $x\in\Xset$.
\end{proof}

\begin{repproposition}[RLCP quantile--oracle deviation]{prop:rlcp_quantile_oracle}
Let $\alpha\in\ooint{0,1}$ and
$\alphalow\in\ooint{0,\alpha}$, $\alphaup\in\ooint{\alpha,1}$, and
$\boldsymbol\alpha=(\alpha,\alphalow,\alphaup)$.
Assume \Cref{assumP:marginal-density,assumK:kernel_assum},
\Aref{assum:S_up}{S}, and \Aref{assum:lip_condcdf}{S}.
Let $\delta\in(0,1)$, $\tilde{x}\in\Xset$ and $h>0$. Suppose that
\[
  0<h\le1\wedge\mathrm h^{\mathrm{bias}}_{\boldsymbol\alpha},
  \qquad \lows{\tilde x}>0,
  \qquad
  n\ge n_1(h,\delta,\boldsymbol\alpha),
\]
with $n_1$ as in \eqref{eq:definition-n1}. Then
$\PP(\SetOmega[n,\tilde{x},h,\delta])\ge1-\delta$ and, on
$\SetOmega[n,\tilde{x},h,\delta]$, the following holds simultaneously
for every $x\in\Xset$ with $\norm{x-\tilde{x}}\le h$: the thresholds
$\rlcpq$ and $\oraclequant{x}$ are finite, and
\begin{equation}
\label{eq:supp-rlcp-quantile-oracle}
  \abs{\rlcpq-\oraclequant{x}}
  \le \Delta^{\mathrm R}_{n,h}(\tilde{x};\delta),
\end{equation}
with $\Delta^{\mathrm R}_{n,h}$ as in \eqref{eq:definition-rate}.
\end{repproposition}

\begin{proof}
The probability bound is the first part of \Cref{prop:calibration_event},
valid since $h\le1$.
Work on
$\SetOmega[n,\tilde{x},h,\delta]$ and fix $x\in\Xset$ with
$\norm{x-\tilde{x}}\le h$. Since $S(x,Y)$ is real-valued under the
regular conditional law given $X=x$, the CDF $\CondCDF(\cdot\mid x)$
is proper; hence, for $1-\alpha\in(0,1)$,
$\oraclequant{x}=\quant{1-\alpha}{\CondCDF(\cdot\mid x)}$ is finite.
The finiteness of $\rlcpq$ follows from \Cref{lem:T1}. We decompose
\begin{equation}
\label{eq:rev-decomposition}
\abs{\rlcpq-\oraclequant{x}}
\le
\abs{\rlcpq-\locoraclequant{\tilde{x}}}
+
\abs{\locoraclequant{\tilde{x}}-\oraclequant{x}} .
\end{equation}
\emph{First term.} \Cref{lem:T1} gives
\begin{equation}
\label{eq:rev-first-term}
\abs{\rlcpq-\locoraclequant{\tilde{x}}}
\le
\frac{\varepsilon_{n,h}(\delta)
      +2(1-\alpha)\rhom}{\lows{\tilde{x}}} .
\end{equation}
\emph{Second term.}  Set $\zeta_h:=L_S(2h)^\beta$. Since
$h\le\mathrm h^{\mathrm{bias}}_{\boldsymbol\alpha}$,
$\zeta_h\le m_{\boldsymbol\alpha}$. In particular
\[
  0<\zeta_h<\min\{\alpha,1-\alpha\},
\]
because $m_{\boldsymbol\alpha}\le\alpha-\alphalow<\alpha$ and
$m_{\boldsymbol\alpha}\le\alphaup-\alpha<1-\alpha$. By
\Cref{assumK:kernel_assum}, $\locmeasure[{\tilde{x},h}]$ is supported
on $B(\tilde{x},h)\cap\Xset$, so every $u$ in its support satisfies
$\norm{u-x}\le2h$. Hence \Cref{assum:lip_condcdf} gives, for every
$t\in\R$,
\begin{equation}
\label{eq:rev-kolmo}
\abs{\CDFloc{\tilde{x}}(t)-\CondCDF(t\mid x)}
\le
\int_\Xset \abs{\CondCDF(t\mid u)-\CondCDF(t\mid x)}
\,\locmeasure[{\tilde{x},h}](\rmd u)
\le \zeta_h .
\end{equation}
Moreover,
\[
  1-\alphaup
  \le 1-\alpha-\zeta_h
  \le 1-\alpha+\zeta_h
  \le 1-\alphalow,
\]
because $\zeta_h\le m_{\boldsymbol\alpha}\le\alphaup-\alpha$ and
$\zeta_h\le m_{\boldsymbol\alpha}\le\alpha-\alphalow$. Therefore, by monotonicity of
the quantile function and the definition of $\I[\tilde{x}]$ (see \eqref{eq:definition-I}),
\begin{equation}
\label{eq:rev-oracle-bracket}
\ccint{
\quant{1-\alpha-\zeta_h}{\CDFloc{\tilde{x}}},\,
\quant{1-\alpha+\zeta_h}{\CDFloc{\tilde{x}}}}
\subseteq \I[\tilde{x}] .
\end{equation}
By the assumed minorization, the localized CDF $\CDFloc{\tilde{x}}$
has density at least $\lows{\tilde{x}}$
Lebesgue-a.e. on $\I[\tilde{x}]$. Thus all hypotheses of
\Cref{supp:lem:quantile_stability_rlcp}(B) hold with
$F\leftarrow\CondCDF(\cdot\mid x)$,
$G\leftarrow\CDFloc{\tilde{x}}$, $p\leftarrow1-\alpha$,
$\zeta\leftarrow\zeta_h$, $I\leftarrow\I[\tilde{x}]$, and
$\underline g\leftarrow\lows{\tilde{x}}$. Consequently,
\begin{equation}
\label{eq:rev-second-term}
\abs{\locoraclequant{\tilde{x}}-\oraclequant{x}}
\le
\frac{\zeta_h}{\lows{\tilde{x}}}
=
\frac{2^\beta L_S h^\beta}{\lows{\tilde{x}}} .
\end{equation}

Combining \eqref{eq:rev-decomposition}, \eqref{eq:rev-first-term} and
\eqref{eq:rev-second-term} with \eqref{eq:definition-rate} proves
\eqref{eq:supp-rlcp-quantile-oracle}. Since
$\SetOmega[n,\tilde{x},h,\delta]$ does not depend on $x$, and the above
bounds hold for each $x\in\Xset$ satisfying $\norm{x-\tilde{x}}\le h$,
the conclusion holds simultaneously for all such $x$.
\end{proof}

\section{Proof of \Cref{thm:rlcp_coverage}}

This section contains the complete proof of \Cref{thm:rlcp_coverage}.

\begin{proof}[Proof of \Cref{thm:rlcp_coverage}]
Fix $x\in B(\tilde x,h)\cap\Xset$. By the definition~\eqref{eq:rlcp-prediction-set}
of the prediction set,
\begin{equation}
\label{eq:rev-cov-identity}
\int_{\Yset} \indi{\rlcpC}(y)\, P_{Y\mid X}(\rmd y\mid x)
=
\CondCDF\bigl(\rlcpq\,\big|\, x\bigr),
\end{equation}
where $\CondCDF(\cdot\mid x)$ denotes the regular conditional CDF of the
score fixed in \Cref{sec:setting}. The conditional coverage error
therefore satisfies
\[
\abs{\int_{\Yset} \indi{\rlcpC}(y)\, P_{Y\mid X}(\rmd y\mid x) - (1-\alpha)}
=
\abs{\CondCDF\bigl(\rlcpq\,\big|\, x\bigr) - (1-\alpha)}.
\]
Recall also the effective calibration level
\[
p_{n,h}(x;\tilde{x})
=\frac{1-\alpha}{1-\nlocweightnorm[x]{\tilde{x}}{n+1}}
\]
introduced in \eqref{eq:effective-level}, which is well defined on the
event $\SetOmega[n,\tilde{x},h,\delta]$: there,
$\nlocweightnorm[x]{\tilde{x}}{n+1}\le\rhom\le\alpha/2$ by
\eqref{eq:supp-cal-weight-bound} and \eqref{eq:rev-n2-consequences}.
The triangle inequality then yields the decomposition
\begin{align*}
\abs{\CondCDF\bigl(\rlcpq\,\big|\, x\bigr) - (1-\alpha)}
&\le
\abs{\CondCDF\bigl(\rlcpq\,\big|\, x\bigr) - \CDFloc{\tilde{x}}(\rlcpq)} \\
&\quad
+ \abs{\CDFloc{\tilde{x}}(\rlcpq) - p_{n,h}(x;\tilde{x})}
+ \abs{p_{n,h}(x;\tilde{x}) - (1-\alpha)}.
\end{align*}
We bound the three terms in turn.

\emph{First term.} By
\Cref{assumK:kernel_assum}, $\locmeasure[{\tilde{x},h}]$ is supported
on $B(\tilde{x},h)\cap\Xset$, so every $u$ in its support satisfies
$\norm{u-x}\le2h$. Hence \Cref{assum:lip_condcdf} gives, for every
$t\in\R$,
\begin{equation}
\label{eq:cov-kolmo}
\abs{\CDFloc{\tilde{x}}(t)-\CondCDF(t\mid x)}
\le
\int_\Xset \abs{\CondCDF(t\mid u)-\CondCDF(t\mid x)}
\,\locmeasure[{\tilde{x},h}](\rmd u)
\le L_S(2h)^\beta ,
\end{equation}
and taking $t=\rlcpq$ yields
\begin{equation}
    \label{eq:cov-first-term}
    \abs{\CondCDF\bigl(\rlcpq\,\big|\, x\bigr)-\CDFloc{\tilde{x}}(\rlcpq)}
\le L_S(2h)^\beta.
\end{equation}

\emph{Second term.}
Since $\nlocweightnorm[x]{\tilde{x}}{n+1}<\alpha$, as established
above, relation \eqref{eq:bar-emp-relation} shows
that, for every finite $t$,
$\nEmpCDFloc{\tilde{x}}(t)\ge1-\alpha$ if and only if
$\nBarCDFloc{\tilde{x}}(t)\ge p_{n,h}(x;\tilde{x})$. As
$p_{n,h}(x;\tilde{x})\in(0,1)$ and
$\nBarCDFloc{\tilde{x}}(\max_{i\in[n]}S_i)=1$, taking infima in the
definition \eqref{eq:rlcp-prediction-set} of $\rlcpq$ yields
\[
\rlcpq=\quant{p_{n,h}(x;\tilde{x})}{\nBarCDFloc{\tilde{x}}}
\in\ccint{\min_{i\in[n]}S_i,\,\max_{i\in[n]}S_i}.
\]
Since $\nBarCDFloc{\tilde{x}}$ is a right-continuous CDF, this
generalized inverse satisfies
\begin{equation}
\label{eq:cov-quantile-props}
\nBarCDFloc{\tilde{x}}(\rlcpq)\ge p_{n,h}(x;\tilde{x}),
\qquad
\nBarCDFloc{\tilde{x}}(t)<p_{n,h}(x;\tilde{x})
\quad\text{for every } t<\rlcpq .
\end{equation}
Moreover, by \Cref{prop:calibration_event}, on
$\SetOmega[n,\tilde{x},h,\delta]$,
\begin{equation}
\label{eq:cov-cal-kolmo}
\sup_{t\in\R}
\abs{\nBarCDFloc{\tilde{x}}(t)-\CDFloc{\tilde{x}}(t)}
\le\varepsilon_{n,h}(\delta) .
\end{equation}

Combining \eqref{eq:cov-cal-kolmo} with the first inequality in
\eqref{eq:cov-quantile-props} yields the lower bound
\begin{equation}
\label{eq:cov-lower}
\CDFloc{\tilde{x}}(\rlcpq)
\ge \nBarCDFloc{\tilde{x}}(\rlcpq)-\varepsilon_{n,h}(\delta)
\ge p_{n,h}(x;\tilde{x})-\varepsilon_{n,h}(\delta),
\end{equation}
that is,
$\CDFloc{\tilde{x}}(\rlcpq)-p_{n,h}(x;\tilde{x})
\ge-\varepsilon_{n,h}(\delta)$.

For the matching upper bound, \eqref{eq:cov-cal-kolmo} and the second
inequality in \eqref{eq:cov-quantile-props} give, for every
$t<\rlcpq$,
\[
\CDFloc{\tilde{x}}(t)
\le \nBarCDFloc{\tilde{x}}(t)+\varepsilon_{n,h}(\delta)
< p_{n,h}(x;\tilde{x})+\varepsilon_{n,h}(\delta) .
\]
Letting $t\uparrow\rlcpq$ and using the continuity of
$\CDFloc{\tilde{x}}$  (absolutely continuous with density $\PDFloc{\tilde{x}}$, see \eqref{eq:localized-pdf}), we obtain
\begin{equation}
\label{eq:cov-upper}
\CDFloc{\tilde{x}}(\rlcpq)
\le p_{n,h}(x;\tilde{x})+\varepsilon_{n,h}(\delta) .
\end{equation}
Combining \eqref{eq:cov-lower} and \eqref{eq:cov-upper},
\begin{equation}
\label{eq:cov-second-term}
\abs{\CDFloc{\tilde{x}}(\rlcpq)-p_{n,h}(x;\tilde{x})}
\le\varepsilon_{n,h}(\delta).
\end{equation}

\emph{Third term.} By \eqref{eq:effective-level},
\begin{equation}
\label{eq:cov-third-term}
0\le p_{n,h}(x;\tilde{x})-(1-\alpha)
=\frac{(1-\alpha)\,\nlocweightnorm[x]{\tilde{x}}{n+1}}
      {1-\nlocweightnorm[x]{\tilde{x}}{n+1}}
\le 2(1-\alpha)\,\rhom,
\end{equation}
where the last inequality holds because
$\nlocweightnorm[x]{\tilde{x}}{n+1}\le\rhom\le\alpha/2\le1/2$, as
established above, so that
$(1-\nlocweightnorm[x]{\tilde{x}}{n+1})^{-1}\le2$.

\emph{Conclusion.} Substituting the bounds
\eqref{eq:cov-first-term}, \eqref{eq:cov-second-term} and
\eqref{eq:cov-third-term} into the decomposition and invoking the
identity \eqref{eq:rev-cov-identity},
\[
\abs{\int_{\Yset}\indi{\rlcpC}(y)\,P_{Y\mid X}(\rmd y\mid x)-(1-\alpha)}
\le
\varepsilon_{n,h}(\delta)+2(1-\alpha)\,\rhom+L_S(2h)^{\beta}
\lesssim
A_{n,h}^{\mathrm{cal}}(\delta)+h^{\beta},
\]
where the last comparison follows from the definitions
\eqref{eq:definition-rhom}--\eqref{eq:definition-varepsilon}, as in
\eqref{eq:definition-rate-order}. Finally, neither the event
$\SetOmega[n,\tilde{x},h,\delta]$ nor the right-hand side depends on
$x$, so the bound holds simultaneously for all
$x\in B(\tilde x,h)\cap\Xset$.
\end{proof}

The remainder of this section records a weaker version of
\Cref{thm:rlcp_coverage}, obtained by the same quantile-deviation route
as the length bound.  It is superseded by \Cref{thm:rlcp_coverage},
since $\lows{\tilde x}\le\ups$, but its proof is independent of the
direct CDF argument above, and the Lipschitz property of conditional
coverage on which it relies (\Cref{lem:cov_control}) is also used in
the learned-score analysis.

\begin{lemma}
\label{lem:cov_control}
Assume \Aref{assum:S_up}{S}. Then, for every $x\in\Xset$ and every
$q,q'\in\R$,
\begin{equation}
\label{eq:ec_cov_def}
  \covdif{x}(q,q')  =
\left|
\PP\bigl(Y\in\conformalset{x}{q}\,\big|\,X=x\bigr)
- \PP\bigl(Y\in\conformalset{x}{q'}\,\big|\,X=x\bigr)
\right|
\end{equation}
satisfies
\begin{equation}
\label{eq:ec_cov_bound}
    \covdif{x}(q,q') \le \ups\abs{q-q'} .
\end{equation}
\end{lemma}

\begin{proof}
By \eqref{eq:conf_set},
\[
  \PP\bigl(Y\in\conformalset{x}{q}\mid X=x\bigr)=\CondCDF(q\mid x),
  \qquad
  \PP\bigl(Y\in\conformalset{x}{q'}\mid X=x\bigr)=\CondCDF(q'\mid x).
\]
If $a=q\wedge q'$ and $b=q\vee q'$, then \Aref{assum:S_up}{S} gives
\[
  \covdif{x}(q,q')
  =\int_a^b\CondPDF(s\mid x)\,\rmd s
  \le\ups(b-a)
  =\ups|q-q'|.
\]
\end{proof}
With this Lipschitz control at hand, the following weak version of
\Cref{thm:rlcp_coverage} reduces to the quantile bound of
\Cref{prop:rlcp_quantile_oracle}.

\begin{proposition}[Coverage control via quantile deviation]
\label{prop:rlcp_coverage_weak}
Under the assumptions and notation of \Cref{thm:rlcp_length}, except
that \Aref{assum:S_len}{S} is not required, on the same event
$\SetOmega[n,\tilde x,h,\delta]$, simultaneously for every
$x\in B(\tilde x,h)\cap\Xset$,
\[
  \left|
  \int_{\Yset}\indi{\rlcpC}(y)\,P_{Y\mid X}(\rmd y\mid x)
  -(1-\alpha)
  \right|
  \lesssim
  \bigl\{A_{n,h}^{\mathrm{cal}}(\delta)+h^\beta\bigr\}/\lows{\tilde x}.
\]
\end{proposition}

\begin{proof}[Proof of \Cref{prop:rlcp_coverage_weak}]
If $\lows{\tilde x}=0$, the right-hand side of the claimed bound is
vacuous under the convention preceding \Cref{thm:rlcp_length}.  We may
therefore assume throughout that $\lows{\tilde x}>0$.

The bandwidth assumption gives
  $0<h\le1$.
The probability bound and the inequality
\begin{equation}
\label{eq:rev-cov-quantile}
\abs{\rlcpq-\oraclequant{x}}
\le
\Delta^{\mathrm R}_{n,h}(\tilde{x};\delta)
\end{equation}
on $\SetOmega[n,\tilde{x},h,\delta]$ are the content of
\Cref{prop:rlcp_quantile_oracle}, whose assumptions are satisfied; in
particular $\rlcpq$ and $\oraclequant{x}$ are finite.

By the definition \eqref{eq:conf_set} of the prediction set, we get 
\begin{equation}
\label{eq:rev-cov-identity-bis}
\int \indi{\rlcpC}(y) P_{Y|X}(\rmd y| x)
=
\CondCDF\bigl(\rlcpq\,\big|\, x\bigr) ,
\end{equation}
$\CondCDF(\cdot\mid x)$ being the regular conditional score CDF fixed
in \Cref{sec:setting}.

Under \Cref{assum:S_up}, $\CondCDF(\cdot\mid x)$ is continuous on
$\R$. By right continuity of a CDF at its left quantile,
$\CondCDF(\oraclequant{x}\mid x)\ge1-\alpha$; by definition of the
left quantile, $\CondCDF(t\mid x)<1-\alpha$ for every
$t<\oraclequant{x}$, so the left limit of $\CondCDF(\cdot\mid x)$ at
$\oraclequant{x}$ is at most $1-\alpha$. Continuity then yields
\begin{equation}
\label{eq:rev-oracle-coverage}
\CondCDF(\oraclequant{x}\mid x)=1-\alpha .
\end{equation}
Hence, by \eqref{eq:rev-cov-quantile}, \eqref{eq:rev-cov-identity-bis},
\eqref{eq:rev-oracle-coverage}, and \Cref{lem:cov_control} with \Cref{assum:S_up}, we get
\begin{align*}
\bigl|\CondCDF(\rlcpq\mid x)-\CondCDF(\oraclequant{x}\mid x)\bigr|
&\le
\ups\abs{\rlcpq-\oraclequant{x}} \\
&\le
\ups\,
\Delta^{\mathrm R}_{n,h}(\tilde{x};\delta)
\lesssim
\bigl\{A_{n,h}^{\mathrm{cal}}(\delta)+h^\beta\bigr\}/\lows{\tilde x} .
\end{align*}
Since \Cref{prop:rlcp_quantile_oracle} holds simultaneously for every
$x\in B(\tilde x,h)\cap\Xset$ on the same event, so does the coverage bound.
\end{proof}

\section{Lemmas for Learned-Score Analysis}
\label{supp:sec:learned}

\begin{replemma}{lem:quant_pivo}
Let $\alpha\in\ooint{0,1}$, $\alphalow\in\ooint{0,\alpha}$,
$\alphaup\in\ooint{\alpha,1}$,
$\boldsymbol{\alpha}=(\alpha,\alphalow,\alphaup)$,
$x\in\Xset$, and $0<h\le1$.
Assume \Cref{assumP:marginal-density,assumK:kernel_assum},
\Aref{assum:S_up}{\Sstar} and
\Aref{assum:S_pivot_quantile}{\alpha}, and suppose
$\lows[\Sstar]{x}>0$.  Then the
localized score CDF $\CDFloc[\Sstar]{x}$ satisfies
\[
    \CDFloc[\Sstar]{x}(\taustar)=1-\alpha ,
\]
and the pivot is its localized $(1-\alpha)$-quantile:
\[
    \locoraclequant[\Sstar]{x}
    =\quant{1-\alpha}{\CDFloc[\Sstar]{x}}
    =\taustar
    \in\I[x][\balpha][\Sstar] .
\]
Moreover, for $\DC[X]$-almost every $z\in\Xset$,
$\CondCDF[\Sstar](\taustar\mid z)=1-\alpha$; consequently
$\oraclequant[\Sstar]{z}\le\taustar$, and both oracle sets
$\conformalset[\Sstar]{z}{\oraclequant[\Sstar]{z}}$ and
$\conformalset[\Sstar]{z}{\taustar}$ have conditional coverage
$1-\alpha$.
\end{replemma}

\begin{proof}
\emph{Step 1: the pivot is a localized level point.}
Since $0<h\le1$, \Cref{lem:gamma_bound_exact} gives $\locnorm>0$.  Hence
$\locmeasure[x,h]$ in \eqref{eq:lambda} is a well-defined probability
measure with $\locmeasure[x,h]\ll\DC[X]$, its density with respect to
$\DC[X]$ being $\locweight{\cdot}/\locnorm$.  The pivotality identity
\eqref{eq:pivotality} holds for $\DC[X]$-almost every covariate, hence
for $\locmeasure[x,h]$-almost every covariate, and integrating it against
$\locmeasure[x,h]$ in \eqref{eq:localized-cdf} yields
$\CDFloc[\Sstar]{x}(\taustar)
=\int_{\Xset}\CondCDF[\Sstar](\taustar\mid v)\,\locmeasure[x,h](\rmd v)
=1-\alpha$, which is \eqref{eq:quant_pivo_loc_cdf}.

\emph{Step 2: identification of the localized quantile.}
We first show that
\begin{equation}
\label{supp:eq:quant-pivo-bracket}
\locoraclequant[\Sstar]{x}[1-\alphaup]
\le
\quant{1-\alpha}{\CDFloc[\Sstar]{x}}
\le
\taustar
\le
\locoraclequant[\Sstar]{x}[1-\alphalow].
\end{equation}
The first inequality is monotonicity of the left-quantile map in the
level, since $1-\alphaup<1-\alpha$.  The second holds because
$\CDFloc[\Sstar]{x}(\taustar)=1-\alpha\ge1-\alpha$ places $\taustar$ in
the set $\{t:\CDFloc[\Sstar]{x}(t)\ge1-\alpha\}$ whose infimum is
$\quant{1-\alpha}{\CDFloc[\Sstar]{x}}$.  The third holds because, for
every $t\le\taustar$, monotonicity gives
$\CDFloc[\Sstar]{x}(t)\le\CDFloc[\Sstar]{x}(\taustar)
=1-\alpha<1-\alphalow$, so no such $t$ belongs to
$\{t:\CDFloc[\Sstar]{x}(t)\ge1-\alphalow\}$, whence
$\taustar\le\locoraclequant[\Sstar]{x}[1-\alphalow]$.  Since the two
extreme members of \eqref{supp:eq:quant-pivo-bracket} are precisely the
endpoints of $\I[x][\balpha][\Sstar]$ (see \eqref{eq:definition-I}),
both $\quant{1-\alpha}{\CDFloc[\Sstar]{x}}$ and $\taustar$ belong to
$\I[x][\balpha][\Sstar]$.

Under \Aref{assum:S_up}{\Sstar}, Tonelli's theorem applied to
\eqref{eq:localized-cdf} shows that $\CDFloc[\Sstar]{x}$ is absolutely
continuous with density $\PDFloc[\Sstar]{x}$ of
\eqref{eq:localized-pdf}, and the assumed minorization gives
$\PDFloc[\Sstar]{x}\ge\lows[\Sstar]{x}>0$ Lebesgue-almost everywhere on
$\I[x][\balpha][\Sstar]$.  Suppose, for contradiction, that
$\quant{1-\alpha}{\CDFloc[\Sstar]{x}}<\taustar$.  Both points belong to
$\I[x][\balpha][\Sstar]$, so absolute continuity and the density lower
bound give
\[
\CDFloc[\Sstar]{x}(\taustar)
-\CDFloc[\Sstar]{x}\bigl(\quant{1-\alpha}{\CDFloc[\Sstar]{x}}\bigr)
=\int_{\quant{1-\alpha}{\CDFloc[\Sstar]{x}}}^{\taustar}
\PDFloc[\Sstar]{x}(s)\,\rmd s
\ge\lows[\Sstar]{x}
\Bigl(\taustar-\quant{1-\alpha}{\CDFloc[\Sstar]{x}}\Bigr)
>0 ,
\]
so that
$\CDFloc[\Sstar]{x}\bigl(\quant{1-\alpha}{\CDFloc[\Sstar]{x}}\bigr)
<\CDFloc[\Sstar]{x}(\taustar)=1-\alpha$, contradicting the right
continuity of a distribution function at its left quantile, which gives
$\CDFloc[\Sstar]{x}\bigl(\quant{1-\alpha}{\CDFloc[\Sstar]{x}}\bigr)
\ge1-\alpha$.  Therefore
$\quant{1-\alpha}{\CDFloc[\Sstar]{x}}=\taustar$, which is
\eqref{eq:quant_pivo_loc_quant}.

\emph{Step 3: pointwise statements.}
Fix $z$ in the $\DC[X]$-full set on which \eqref{eq:pivotality} holds, so
that $\CondCDF[\Sstar](\taustar\mid z)=1-\alpha$; in particular
$\taustar$ belongs to the set whose infimum defines
$\oraclequant[\Sstar]{z}$, whence
$\oraclequant[\Sstar]{z}\le\taustar$.  Under
\Aref{assum:S_up}{\Sstar}, $\CondCDF[\Sstar](\cdot\mid z)$ is
continuous; since $\CondCDF[\Sstar](t\mid z)<1-\alpha$ for every
$t<\oraclequant[\Sstar]{z}$, while right continuity at the left quantile
gives $\CondCDF[\Sstar](\oraclequant[\Sstar]{z}\mid z)\ge1-\alpha$,
continuity forces
$\CondCDF[\Sstar](\oraclequant[\Sstar]{z}\mid z)=1-\alpha$.  Finally, by
\eqref{eq:conf_set} and the fixed regular conditional versions of
Section~\ref{sec:setting} of the main paper,
$\PP\{Y\in\conformalset[\Sstar]{z}{q}\mid X=z\}
=\CondCDF[\Sstar](q\mid z)$ for every $q\in\R$; taking
$q=\oraclequant[\Sstar]{z}$ and $q=\taustar$ shows that both oracle sets
have conditional coverage $1-\alpha$.
\end{proof}

\begin{replemma}[Localized Kolmogorov perturbation by score estimation]{lem:kolmo_s_sstar}
Let \(x\in\Xset\) and \(h>0\) be such that \(\locnorm>0\).  Assume
\Aref{assum:S_up}{\Sstar} and \(\DeltaS\in L^1(\locmeasure[x,h])\).  Then
\[
    \sup_{t\in\R}
    \abs{
    \CDFloc[\Shat]{x}(t)
    -
    \CDFloc[\Sstar]{x}(t)
    }
    \le
    \ups\,\Loneloc{\DeltaS}{x,h}.
\]
\end{replemma}

\begin{proof}
Since \(\DeltaS\in L^1(\locmeasure[x,h])\), \(\DeltaS(u)<\infty\) for
\(\locmeasure[x,h]\)-almost every \(u\).  For every such \(u\), the
definition \eqref{eq:score-error-pointwise} gives
\(\abs{\Shat(u,y)-\Sstar(u,y)}\le\DeltaS(u)\) for all \(y\in\Yset\),
whence, for every \(t\in\R\),
\[
\{y\in\Yset:\Sstar(u,y)\le t-\DeltaS(u)\}
\subseteq
\{y\in\Yset:\Shat(u,y)\le t\}
\subseteq
\{y\in\Yset:\Sstar(u,y)\le t+\DeltaS(u)\}.
\]
Taking \(P_{Y\mid X=u}\)-probabilities,
\begin{equation}
\label{supp:eq:deltas-control}
\CondCDF[\Sstar](t-\DeltaS(u)\mid u)
\le
\CondCDF[\Shat](t\mid u)
\le
\CondCDF[\Sstar](t+\DeltaS(u)\mid u).
\end{equation}
By \Aref{assum:S_up}{\Sstar}, \(\CondCDF[\Sstar](\cdot\mid u)\) is
absolutely continuous with density bounded by \(\ups\), so
\[
0\le
\CondCDF[\Sstar](t+\DeltaS(u)\mid u)-\CondCDF[\Sstar](t\mid u)
=\int_{t}^{t+\DeltaS(u)}\CondPDF[\Sstar](s\mid u)\,\rmd s
\le\ups\,\DeltaS(u) ,
\]
and likewise
\(0\le\CondCDF[\Sstar](t\mid u)-\CondCDF[\Sstar](t-\DeltaS(u)\mid u)
\le\ups\,\DeltaS(u)\).  Hence both outer members of
\eqref{supp:eq:deltas-control} lie within \(\ups\,\DeltaS(u)\) of
\(\CondCDF[\Sstar](t\mid u)\), and
\[
\abs{\CondCDF[\Shat](t\mid u)-\CondCDF[\Sstar](t\mid u)}
\le \ups\,\DeltaS(u)
\qquad
\text{for \(\locmeasure[x,h]\)-almost every }u .
\]
Integrating against \(\locmeasure[x,h]\) in \eqref{eq:localized-cdf} and
taking the supremum over \(t\in\R\) proves \eqref{eq:kolmo_s_sstar}.
\end{proof}

\begin{replemma}[Length and coverage comparison under score and threshold deviations]{lem:sandwich_length}
Let $x\in\Xset$, and let $r,q\in\R$ and $\Delta\ge0$ be such that
$\abs{q-r}\le\Delta$.  If \Aref{assum:S_up}{\Sstar} holds, then
\[
\abs{
\int_{\Yset} \indi{\conformalset[\Shat]{x}{q}}(y)\, P_{Y\mid X}(\rmd y\mid x)
-
\int_{\Yset} \indi{\conformalset[\Sstar]{x}{r}}(y)\, P_{Y\mid X}(\rmd y\mid x)
}
\le
\ups
\bigl[\Delta+\DeltaS(x)\bigr],
\]
and if in addition \Aref{assum:S_len}{\Sstar} holds, then
\[
\abs{
\abs{\conformalset[\Shat]{x}{q}}
-
\abs{\conformalset[\Sstar]{x}{r}}
}
\le
\Llen
\bigl[\Delta+\DeltaS(x)\bigr].
\]
\end{replemma}

\begin{proof}
If $\DeltaS(x)=+\infty$, both bounds hold trivially; assume
$\DeltaS(x)<\infty$.
We first establish the inclusions
\begin{equation}
\label{supp:eq:sandwich-inclusions}
\conformalset[\Sstar]{x}{r-\Delta-\DeltaS(x)}
\subseteq
\conformalset[\Shat]{x}{q}
\subseteq
\conformalset[\Sstar]{x}{r+\Delta+\DeltaS(x)}.
\end{equation}
By \eqref{eq:score-error-pointwise},
$\abs{\Shat(x,y)-\Sstar(x,y)}\le\DeltaS(x)$ for every $y\in\Yset$.  If
$\Sstar(x,y)\le r-\Delta-\DeltaS(x)$, then
$\Shat(x,y)\le\Sstar(x,y)+\DeltaS(x)\le r-\Delta\le q$; conversely, if
$\Shat(x,y)\le q$, then
$\Sstar(x,y)\le\Shat(x,y)+\DeltaS(x)\le r+\Delta+\DeltaS(x)$.  This proves
\eqref{supp:eq:sandwich-inclusions}.

By \eqref{eq:conf_set} and the fixed regular conditional versions of
Section~\ref{sec:setting} of the main paper, for every $q'\in\R$,
\begin{equation}
\label{supp:eq:sandwich-monotone-maps}
\int_{\Yset}\indi{\conformalset[\Sstar]{x}{q'}}(y)\,
P_{Y\mid X}(\rmd y\mid x)
=\CondCDF[\Sstar](q'\mid x);
\end{equation}
both $q'\mapsto\abs{\conformalset[\Sstar]{x}{q'}}$ and
$q'\mapsto\CondCDF[\Sstar](q'\mid x)$ are nondecreasing, the first being
$\Llen$-Lipschitz by \Aref{assum:S_len}{\Sstar} and the second
$\ups$-Lipschitz by \Aref{assum:S_up}{\Sstar}
(\Cref{lem:cov_control} applied with $S=\Sstar$).

Applying Lebesgue measure to \eqref{supp:eq:sandwich-inclusions} and
subtracting $\abs{\conformalset[\Sstar]{x}{r}}$,
\begin{align*}
    \abs{\conformalset[\Sstar]{x}{r-\Delta-\DeltaS(x)}}
-\abs{\conformalset[\Sstar]{x}{r}}
&\le
\abs{\conformalset[\Shat]{x}{q}}
-\abs{\conformalset[\Sstar]{x}{r}} \\
&\le
\abs{\conformalset[\Sstar]{x}{r+\Delta+\DeltaS(x)}}
-\abs{\conformalset[\Sstar]{x}{r}},
\end{align*}

and the $\Llen$-Lipschitz property bounds the two outer members by
$\Llen\bigl[\Delta+\DeltaS(x)\bigr]$ in absolute value, which proves
\eqref{eq:sandwich_length}.  Applying $P_{Y\mid X=x}$ to
\eqref{supp:eq:sandwich-inclusions} instead and using
\eqref{supp:eq:sandwich-monotone-maps}, the coverage of
$\conformalset[\Shat]{x}{q}$ is sandwiched between
$\CondCDF[\Sstar](r-\Delta-\DeltaS(x)\mid x)$ and
$\CondCDF[\Sstar](r+\Delta+\DeltaS(x)\mid x)$, each within
$\ups\bigl[\Delta+\DeltaS(x)\bigr]$ of $\CondCDF[\Sstar](r\mid x)$ by the
$\ups$-Lipschitz property; this proves \eqref{eq:sandwich_coverage}.
\end{proof}

\section{Intermediate Results for Learned-Score Theorems}
\label{supp:sec:learned-intermediate}
\begin{repproposition}[Adaptive localized quantile control]{prop:adaptive_rlcp_quantile}
Let $\alpha\in\ooint{0,1}$, $\alphalow\in\ooint{0,\alpha}$,
$\alphaup\in\ooint{\alpha,1}$, $\boldsymbol{\alpha}=(\alpha,\alphalow,\alphaup)$,
$\delta\in\ooint{0,1}$, $\tilde{x}\in\Xset$, and $0<h\le1$.
Assume \Cref{assumP:marginal-density,assumK:kernel_assum},
\Aref{assum:S_up}{\Sstar},
\Aref{assum:S_pivot_quantile}{\alpha} and
\Aref{assum:score_estimation_rate}{S}.
Let $m_{\boldsymbol{\alpha}}$ be defined in \eqref{eq:definition-I}.
Suppose that
\[
n\ge n_1^\star(h,\delta/2,\boldsymbol{\alpha}),
\qquad
\ntrain \ge \ntrain^{\Delta}(h, \delta, \boldsymbol{\alpha}),
\]
with $n_1^\star$ and $\ntrain^{\Delta}$ defined in
\eqref{eq:definition-n1star} and \eqref{eq:definition-ntrain-delta},
respectively.  Then
$\PP(\SetAdapt)\ge1-\delta$ and, on
$\SetAdapt$, the following holds
simultaneously for every $x\in\Xset$:
\begin{equation}
\label{eq:supp-adaptive-quantile-control}
\abs{\rlcpq-\taustar}
\le
\frac{\varepsilon_{n,h}(\delta/2) + 2(1-\alpha)\,\rhom + \ups\,\epsilon^{\Delta}_{\infty}(\ntrain,\delta/2)}{\lows[\Sstar]{\tilde{x}}}.
\end{equation}
\end{repproposition}

\begin{proof}
Throughout, set 
\[
\zeta
:=
\varepsilon_{n,h}(\delta/2)
+\ups\,\epsilon^{\Delta}_{\infty}(\ntrain,\delta/2),
\]
and note that $\zeta>0$, since
$\varepsilon_{n,h}(\delta/2)>0$ by inspection of
\eqref{eq:definition-varepsilon}.
Since $h\le1$, \eqref{eq:omega_adapt-probability} gives
$\PP(\SetAdapt)\ge1-\delta$. From
now on, we work on $\SetAdapt$ and fix $x\in\Xset$. 
If $\lows = 0$, the result is immediate under the convention $1/0 = \infty$. In what follows, we therefore assume $\lows > 0$.

\emph{Step 1: quantile representation and effective level.}
By \eqref{eq:definition-n1star} at level $\delta/2$,
$n\ge n_0(h,\delta/2)$.  Hence \Cref{prop:calibration_event}, applied with
$S=\Shat$ on $\SetOmega[n,\tilde{x},h,\delta/2]$, gives a positive
calibration denominator.  Thus $\nBarCDFloc[\Shat]{\tilde{x}}$ of
\eqref{eq:bar-empirical-cdf} is a well-defined CDF supported on the
calibration scores $\{\Shat(X_i,Y_i)\}_{i\in[n]}$, and the weights
\eqref{eq:rlcp-weights} are well defined.  The same proposition and
\eqref{eq:definition-n1star} give, uniformly in
$x\in\Xset$,
\begin{equation}
\label{eq:adaptive-weight-small}
\nlocweightnorm{\tilde{x}}{n+1}
\le \rhom
\le \frac{\alpha}{2}
< \alpha ,
\end{equation}
where the second inequality holds because
$\varepsilon_{n,h}(\delta/2)+\rhom\le\alpha/2$ and
$\varepsilon_{n,h}(\delta/2)>0$. Hence the effective level
$p_{n,h}(x;\tilde x)$ of \eqref{eq:effective-level} is well defined and,
exactly as in \eqref{eq:rev-level-range} and
\eqref{eq:rev-levelshift}---whose derivations use only
\eqref{eq:adaptive-weight-small}---
\begin{equation}
\label{eq:adaptive-level-facts}
1-\alpha \le p_{n,h}(x;\tilde x) < 1,
\qquad
0\le p_{n,h}(x;\tilde x)-(1-\alpha)\le 2(1-\alpha)\,\rhom.
\end{equation}
By \eqref{eq:bar-emp-relation} with $S=\Shat$ and
$p_{n,h}(x;\tilde x)\in\ooint{0,1}$, the level-set argument leading to
\eqref{eq:rev-quantile-identity}---which uses no property of the
score---applies verbatim and yields
\[
\rlcpq
=
\quant{p_{n,h}(x;\tilde x)}{\nBarCDFloc[\Shat]{\tilde{x}}}.
\]
On the other hand, \Cref{lem:quant_pivo}, applicable at $\tilde x$, gives
\[
\taustar
=
\quant{1-\alpha}{\CDFloc[\Sstar]{\tilde x}}
\in\I[\tilde x][\balpha][\Sstar]. 
\]
The triangle inequality then yields the decomposition
\begin{multline}
\label{eq:adaptive-decomposition}
\abs{\rlcpq-\taustar}
    \le
    \underbrace{\bigl|
    \quant{p_{n,h}(x;\tilde x)}{\nBarCDFloc[\Shat]{\tilde{x}}}
    -
    \quant{p_{n,h}(x;\tilde x)}{\CDFloc[\Sstar]{\tilde x}}
    \bigr|}_{T_{\mathrm{est}}}
    \\
    +
    \underbrace{\bigl|
    \quant{p_{n,h}(x;\tilde x)}{\CDFloc[\Sstar]{\tilde x}}
    -
    \quant{1-\alpha}{\CDFloc[\Sstar]{\tilde x}}
    \bigr|}_{T_{\mathrm{lev}}} .
\end{multline}

\emph{Step 2: uniform CDF control.}
Since $\locmeasure[\tilde x,h]$ is a probability measure with
$\locmeasure[\tilde x,h]\ll\DC[X]$, the
definition~\eqref{eq:omega_train-definition} of $\SetTrain$ gives
\begin{equation}
    \label{eq:l1-lp}
\Loneloc{\DeltaS}{\tilde x,h}
\le
\esssup_{x\in\Xset}\DeltaS(x)
\le
\epsilon^{\Delta}_{\infty}(\ntrain,\delta/2);
\end{equation}
in particular $\DeltaS\in L^1(\locmeasure[\tilde x,h])$. For every
$t\in\R$,
\[
\bigl| \nBarCDFloc[\Shat]{\tilde{x}}(t) - \CDFloc[\Sstar]{\tilde x}(t) \bigr|
\le
\bigl| \nBarCDFloc[\Shat]{\tilde{x}}(t) - \CDFloc[\Shat]{\tilde x}(t) \bigr|
+
\bigl| \CDFloc[\Shat]{\tilde x}(t) - \CDFloc[\Sstar]{\tilde x}(t) \bigr|.
\]
For the first term, apply \Cref{prop:calibration_event} with $S=\Shat$ at
level $\delta/2$, valid on $\SetOmega[n,\tilde{x},h,\delta/2]$ since
$n\ge n_0(h,\delta/2)$; for the second, apply
\Cref{lem:kolmo_s_sstar}---its hypotheses hold by
\Cref{lem:gamma_bound_exact} and \eqref{eq:l1-lp}---and bound the
resulting $L^1$ norm by \eqref{eq:l1-lp}. This yields
\begin{equation}
\label{eq:adaptive-kolmogorov}
\sup_{t \in \R}
\bigl| \nBarCDFloc[\Shat]{\tilde{x}}(t) - \CDFloc[\Sstar]{\tilde x}(t) \bigr|
\le \zeta .
\end{equation}

\emph{Step 3: level condition.}
By \eqref{eq:definition-n1star},
$\varepsilon_{n,h}(\delta/2)\le\alpha/2-\rhom$, and by
\eqref{eq:definition-ntrain-delta},
$\ups\,\epsilon^{\Delta}_{\infty}(\ntrain,\delta/2)<\alpha/2$;
hence, using \eqref{eq:adaptive-weight-small} and
$1-\nlocweightnorm{\tilde{x}}{n+1}\le1$,
\begin{equation}
\label{eq:adaptive-level-condition}
\zeta
< \Bigl(\frac{\alpha}{2} - \rhom\Bigr) + \frac{\alpha}{2}
 = \alpha - \rhom
\le \alpha - \nlocweightnorm{\tilde{x}}{n+1}
\le \frac{\alpha - \nlocweightnorm{\tilde{x}}{n+1}}
         {1 - \nlocweightnorm{\tilde{x}}{n+1}}
 = 1 - p_{n,h}(x;\tilde x) .
\end{equation}
Moreover, since
$\varepsilon_{n,h}(\delta/2)\le m_{\boldsymbol{\alpha}}/2$
by \eqref{eq:definition-n1star} and
$\ups\,\epsilon^{\Delta}_{\infty}(\ntrain,\delta/2)<m_{\boldsymbol{\alpha}}/2$
by \eqref{eq:definition-ntrain-delta},
\begin{equation}
\label{eq:adaptive-level-condition-lower}
\zeta
< m_{\boldsymbol{\alpha}}
\le \alphaup-\alpha
< 1-\alpha
\le p_{n,h}(x;\tilde x) ,
\end{equation}
the last inequality by \eqref{eq:adaptive-level-facts}. Combining
\eqref{eq:adaptive-level-condition} and
\eqref{eq:adaptive-level-condition-lower},
\begin{equation}
\label{eq:adaptive-zeta-window}
0<\zeta
< \min\{p_{n,h}(x;\tilde x),\, 1-p_{n,h}(x;\tilde x)\} .
\end{equation}

\emph{Step 4: quantile bracket.}
By the level shift in \eqref{eq:adaptive-level-facts} and \eqref{eq:definition-n1star},
$\varepsilon_{n,h}(\delta/2)+2(1-\alpha)\rhom\le m_{\boldsymbol{\alpha}}/2$ and by \eqref{eq:definition-ntrain-delta},
$\ups\,\epsilon^{\Delta}_{\infty}(\ntrain,\delta/2)< m_{\boldsymbol{\alpha}}/2$,
\[
p_{n,h}(x;\tilde x)+\zeta
\le 1-\alpha + 2(1-\alpha)\rhom+\zeta
\le 1-\alpha + m_{\boldsymbol{\alpha}}
\le 1-\alphalow,
\]
while $\zeta<m_{\boldsymbol{\alpha}}\le\alphaup-\alpha$ from
\eqref{eq:adaptive-level-condition-lower} gives
\[
p_{n,h}(x;\tilde x)-\zeta
\ge 1-\alpha - m_{\boldsymbol{\alpha}}
\ge 1-\alphaup .
\]
Since the endpoints of $\I[\tilde x][\balpha][\Sstar]$ are the quantiles of
$\CDFloc[\Sstar]{\tilde x}$ at the levels $1-\alphaup$ and $1-\alphalow$
(see \eqref{eq:definition-I}), monotonicity of the quantile function yields
\begin{equation}
    \label{eq:adaptive-bracket}
\ccint{
\quant{p_{n,h}(x;\tilde x)-\zeta}{\CDFloc[\Sstar]{\tilde x}},\,
\quant{p_{n,h}(x;\tilde x)+\zeta}{\CDFloc[\Sstar]{\tilde x}}
}
\subseteq \I[\tilde x][\balpha][\Sstar] .
\end{equation}

\emph{Step 5: bounding $T_{\mathrm{est}}$.}
Under \Aref{assum:S_up}{\Sstar}, Tonelli's theorem applied to
\eqref{eq:localized-cdf} shows that $\CDFloc[\Sstar]{\tilde x}$ is
absolutely continuous with density $\PDFloc[\Sstar]{\tilde x}$ of
\eqref{eq:localized-pdf}, and the assumed minorization at $\tilde x$ gives
$\PDFloc[\Sstar]{\tilde x}\ge\lows[\Sstar]{\tilde x}>0$ Lebesgue-a.e.\ on
$\I[\tilde x][\balpha][\Sstar]$. Thus all hypotheses of
\Cref{supp:lem:quantile_stability_rlcp}(B) hold with
$F \leftarrow \nBarCDFloc[\Shat]{\tilde{x}}$,
$G \leftarrow \CDFloc[\Sstar]{\tilde x}$,
$p \leftarrow p_{n,h}(x;\tilde{x})$,
$\zeta\leftarrow\zeta$,
$I \leftarrow \I[\tilde x][\balpha][\Sstar]$, and
$\underline{g} \leftarrow \lows[\Sstar]{\tilde{x}}$: the Kolmogorov bound
is \eqref{eq:adaptive-kolmogorov}, the level condition is
\eqref{eq:adaptive-zeta-window}, and the quantile bracket is
\eqref{eq:adaptive-bracket}. Hence
\begin{equation}
\label{eq:Test-final}
T_{\mathrm{est}} \le \frac{\zeta}{\lows[\Sstar]{\tilde{x}}}
=
\frac{\varepsilon_{n,h}(\delta/2) + \ups\,\epsilon^{\Delta}_{\infty}(\ntrain,\delta/2)}{\lows[\Sstar]{\tilde{x}}}.
\end{equation}

\emph{Step 6: bounding $T_{\mathrm{lev}}$.}
Since $p_{n,h}(x;\tilde x)$ lies between the two levels defining
\eqref{eq:adaptive-bracket}, monotonicity of the quantile function gives
$\quant{p_{n,h}(x;\tilde x)}{\CDFloc[\Sstar]{\tilde x}} \in \I[\tilde x][\balpha][\Sstar]$;
moreover
$\quant{1-\alpha}{\CDFloc[\Sstar]{\tilde x}} = \taustar \in \I[\tilde x][\balpha][\Sstar]$
by \Cref{lem:quant_pivo}. Therefore
\Cref{supp:lem:quantile_lipschitz_levels_rlcp}, applied to
$\CDFloc[\Sstar]{\tilde x}$ on $\I[\tilde x][\balpha][\Sstar]$ with the
density floor of Step 5, together with the level shift in
\eqref{eq:adaptive-level-facts}, gives
\begin{equation}
\label{eq:Tlev-final}
T_{\mathrm{lev}}
\le \frac{p_{n,h}(x;\tilde x) - (1-\alpha)}{\lows[\Sstar]{\tilde{x}}}
\le \frac{2(1-\alpha)\,\rhom}{\lows[\Sstar]{\tilde{x}}}.
\end{equation}

\emph{Step 7: synthesis.}
Combining \eqref{eq:adaptive-decomposition}, \eqref{eq:Test-final} and
\eqref{eq:Tlev-final} proves \eqref{eq:supp-adaptive-quantile-control}.
Finally, the event $\SetAdapt$ does not depend on $x$, and $x$ enters the
argument only through $\nlocweightnorm{\tilde{x}}{n+1}$, which is
controlled exclusively via the bound \eqref{eq:adaptive-weight-small},
valid uniformly in $x\in\Xset$; hence the conclusion holds simultaneously
for every $x\in\Xset$.
\end{proof}

\section{Deterministic Quantile Lemmas}
\label{supp:sec:quantile}

For the underlying order and endpoint properties of left quantiles, see
\citet[Proposition~1]{embrechts2013generalized}; the inverse-CDF Lipschitz
criterion is discussed in
\citet[Appendix~A, Proposition~A.24]{bobkovledoux2019}, and part~(A) of
\Cref{supp:lem:quantile_stability_rlcp} specializes the CDF-band inversion of
\citet[Section~2.2, Theorem~1]{chernozhukov2020generic}.

\begin{lemma}
\label{supp:lem:quantile_lipschitz_levels_rlcp}
Let $F$ be a CDF on $\rset$ and $I\subset\rset$ an interval such that $F$
is absolutely continuous on $I$ with density $p$ satisfying
\begin{equation}
\label{supp:eq:density_lb_I}
p(t)\ge\mu>0\qquad\text{for Lebesgue-a.e.\ }t\in I .
\end{equation}
Then, for any $u,v\in(0,1)$ with $\quant{u}{F}\in I$ and
$\quant{v}{F}\in I$,
\begin{equation}
\label{supp:eq:quantile_lipschitz}
\abs{\quant{u}{F}-\quant{v}{F}}\le\abs{u-v}/\mu .
\end{equation}
\end{lemma}

\begin{proof}
Without loss of generality $v\le u$, so
$\quant{v}{F}\le\quant{u}{F}$ by monotonicity of the left-quantile map in
the level; if the two quantiles are equal there is nothing to prove, so
assume $\quant{v}{F}<\quant{u}{F}$.  Since both points belong to the
interval $I$, so does the segment
$\ccint{\quant{v}{F},\quant{u}{F}}$.

By right continuity of $F$ at its left quantile,
$F(\quant{v}{F})\ge v$.  By the definition of the left quantile,
$F(t)<u$ for every $t<\quant{u}{F}$; letting $t\uparrow\quant{u}{F}$
along $I$---possible because $\quant{v}{F}\in I$ lies strictly below
$\quant{u}{F}$---and using the continuity of $F$ on $I$ (absolute
continuity on $I$ implies continuity there, including one-sided
continuity at the endpoints relative to $I$), we get
$F(\quant{u}{F})\le u$.  Finally, absolute continuity of $F$ on $I$ and
the density lower bound \eqref{supp:eq:density_lb_I} give
\[
\mu\bigl(\quant{u}{F}-\quant{v}{F}\bigr)
\le\int_{\quant{v}{F}}^{\quant{u}{F}}p(t)\,\rmd t
=F(\quant{u}{F})-F(\quant{v}{F})
\le u-v ,
\]
which is \eqref{supp:eq:quantile_lipschitz}.
\end{proof}

\begin{lemma}
\label{supp:lem:quantile_stability_rlcp}
Let $F$ and $G$ be CDFs on $\mathbb R$, fix $p\in(0,1)$, and define
$Q_F(p):=\inf\{t:\ F(t)\ge p\}$ and $Q_G(p):=\inf\{t:\ G(t)\ge p\}$.
Assume that
\begin{equation}
\label{supp:eq:sup_cdf_close_any}
\|F-G\|_\infty=\sup_{t\in\mathbb R}|F(t)-G(t)|\le\zeta
\qquad\text{for some }\zeta>0 .
\end{equation}
\runinhead{(A) One-sided bracketing.}
\begin{enumerate}
\item If $\zeta<p$, then $Q_G(p-\zeta)\le Q_F(p)$.
\item If $\zeta<1-p$, then $Q_F(p)\le Q_G(p+\zeta)$.
\end{enumerate}
In particular, if $\zeta<\min\{p,1-p\}$, then
$Q_G(p-\zeta)\le Q_F(p)\le Q_G(p+\zeta)$.
\runinhead{(B) One-/two-sided Lipschitz bounds under a density lower
bound.}
Assume additionally that \(G\) is absolutely continuous on an interval
\(I\) with density \(g\) satisfying $g(t)\ge\underline g>0$ for
Lebesgue-a.e.\ $t\in I$.  For any $\zeta<\min\{p,1-p\}$ satisfying
$[Q_G(p-\zeta),\,Q_G(p+\zeta)]\subseteq I$, we get that
\begin{equation}
\label{supp:eq:two_sided_lip_any_sharp}
|Q_F(p)-Q_G(p)|\le\zeta/\underline g .
\end{equation}
\end{lemma}

\begin{proof}
\emph{(A)(1).}  Assume $\zeta<p$, so $p-\zeta\in(0,1)$.  If
$\{t:F(t)\ge p\}=\emptyset$, then $Q_F(p)=+\infty$ and the claim is
trivial.  Otherwise, for every $t$ with $F(t)\ge p$, the Kolmogorov
proximity assumption \eqref{supp:eq:sup_cdf_close_any} gives
$G(t)\ge F(t)-\zeta\ge p-\zeta$, hence $Q_G(p-\zeta)\le t$; taking the
infimum over such $t$ yields $Q_G(p-\zeta)\le Q_F(p)$.

\emph{(A)(2).}  Assume $\zeta<1-p$, so $p+\zeta\in(0,1)$.  The set
$\{t:G(t)\ge p+\zeta\}$ is nonempty, since $G(t)\to1$ as $t\to\infty$;
and $Q_G(p+\zeta)>-\infty$, since $G(t)\to0$ as $t\to-\infty$ while
$p+\zeta>0$.  By right continuity of $G$ at its left quantile,
$G\bigl(Q_G(p+\zeta)\bigr)\ge p+\zeta$, hence, by
\eqref{supp:eq:sup_cdf_close_any},
$F\bigl(Q_G(p+\zeta)\bigr)\ge G\bigl(Q_G(p+\zeta)\bigr)-\zeta\ge p$, and
therefore $Q_F(p)\le Q_G(p+\zeta)$.

\emph{(B).}  Since $\zeta<\min\{p,1-p\}$, both parts of (A) apply:
\begin{equation}
\label{supp:eq:quantile-sandwich}
Q_G(p-\zeta)\;\le\;Q_F(p)\;\le\;Q_G(p+\zeta);
\end{equation}
in particular $Q_F(p)\in\rset$.  Monotonicity of the left-quantile map in
the level also places $Q_G(p)$ in
$\ccint{Q_G(p-\zeta),Q_G(p+\zeta)}$, which is contained in $I$ by
hypothesis; thus the four points
$Q_G(p-\zeta)\le Q_G(p)\le Q_G(p+\zeta)$ and $Q_F(p)$ all lie in $I$, and
the three levels $p-\zeta,p,p+\zeta$ belong to $(0,1)$.  If
$Q_F(p)\ge Q_G(p)$, then by \eqref{supp:eq:quantile-sandwich} and
\Cref{supp:lem:quantile_lipschitz_levels_rlcp}, applied to $G$ on $I$ with the
levels $p+\zeta$ and $p$,
\[
0\le Q_F(p)-Q_G(p)\le Q_G(p+\zeta)-Q_G(p)\le\frac{\zeta}{\underline g};
\]
if $Q_F(p)<Q_G(p)$, then symmetrically, by
\Cref{supp:lem:quantile_lipschitz_levels_rlcp} applied to $G$ on $I$ with the
levels $p$ and $p-\zeta$,
\[
0< Q_G(p)-Q_F(p)\le Q_G(p)-Q_G(p-\zeta)\le\frac{\zeta}{\underline g}.
\]
In either case $|Q_F(p)-Q_G(p)|\le\zeta/\underline g$, which is
\eqref{supp:eq:two_sided_lip_any_sharp}.
\end{proof}

\section{Experimental Details and Additional Results}
\label{supp:sec:experiments}
\runinhead{Reproducibility and fixed-score simulation conventions.}
The fixed-score studies in \Cref{fig:exp-bandwidth,fig:exp-coverage,fig:exp-tracking}
use a single global seed, $20260710$, from which independent per-replication
seeds are generated deterministically.  The computations were performed with
Python~3.13.2, NumPy~2.2.3, SciPy~1.15.2, Matplotlib~3.10.6, PyTorch~2.9.1, and
pandas~2.2.3.  Quantiles use the left-continuous convention
$\inf\{t:F(t)\ge q\}$ throughout.  As an implementation check, the RLCP
threshold was computed both as the direct weighted conformal quantile and in
an algebraically equivalent renormalized form; the two values agreed to
machine precision whenever the threshold was finite.  The Epanechnikov
auxiliary variable was sampled exactly.  All experiments are reproducible from
the supplementary computational archive; the seeds,
bandwidth grids, kernels, and package versions are recorded in its
\texttt{PARAMS.md} and \texttt{requirements-lock.txt}.

Both fixed-score studies use the model in \eqref{eq:exp-dgp-bj}, with the fixed
response envelope $\Yset=[-4,4]$, $d=1$, $x_0=1/2$, $B=3$, $\alpha=0.1$, and
$\beta\in\{0.5,0.75,1\}$.  The localizer is
$K(u)=\tfrac34(1-u^2)_+$.  In each replication we draw the auxiliary centre
$\tilde x=x_0+hV$ with $V\sim K$, evaluate the conditional error at $x=x_0$ for
the realized centre, and average over both the calibration sample and the
auxiliary-centre randomization.  This is the interior geometry of
\Cref{thm:rlcp_length,thm:rlcp_coverage}; the averages are rate diagnostics,
not additional theorem-level guarantees.  For each realized
threshold, conditional coverage is evaluated exactly from the distribution
function of the uniform noise.  No auxiliary test sample or Monte Carlo draw
from $Y\mid X=x_0$ is used.  Both studies are formal-RLCP experiments on the
shared response envelope $[-M,M]$ with $M=B+1=4$: a $+\infty$ localized
threshold (an empty window, or a nonempty window with
$W_{\mathrm{cal}}<\tfrac{1-\alpha}{\alpha}W_{\mathrm{test}}$) maps the one-sided
set to the full envelope $[-4,4]$, yielding a finite set-length error
$2M-|C_{\mathrm{oracle}}|=1.6$ and coverage error $0.1$.  Such thresholds are
kept throughout; they are never discarded, conditioned out, capped at the
support edge, or replaced by a fallback.

\runinhead{Bandwidth bias--variance study.}
\Cref{fig:exp-bandwidth,fig:exp-coverage} report the fixed-score bandwidth
diagnostic across $\beta\in\{0.5,0.75,1\}$ and all seven calibration sizes
$n\in\{500,10^3,2\times10^3,5\times10^3,10^4,2\times10^4,5\times10^4\}$, each
over $R=1000$ independent calibration samples.  The primary metric is the
\emph{formal set-length error}
$\bigl|\,|C_{\mathrm{RLCP}}(x_0)|-|C_{\mathrm{oracle}}(x_0)|\,\bigr|$ of the
one-sided set $[-M,\min(\rlcpq[x_0],M)]$ (equal to $[-M,M]$ when the
threshold is $+\infty$), measured against the oracle set $[-4,2.4]$ of length
$6.4$.  At each pair $(\beta,n)$ the error is evaluated over sixteen
logarithmically spaced bandwidths spanning $[2\times10^{-3},0.25]$.  Small
bandwidths reduce the effective calibration size, so the formal $+\infty$ event
becomes common---its empirical frequency $\widehat{\Pr}(q^{\mathrm{R}}=+\infty)$
reaches $0.99$ at $n=500$, $h=2\times10^{-3}$---which is precisely why the
finite-under-$+\infty$ set-length metric is used; large bandwidths increase the
localization bias.  The bottom row of \Cref{fig:exp-bandwidth} therefore reports
$\widehat{\Pr}(q^{\mathrm{R}}=+\infty)$ against $h$ alongside the length error.
The minimizer $h^\star(n)$ is the argument of the minimum of the
\emph{unconditional} mean set-length error, obtained as the vertex of a local
quadratic fit in $(\log h,\log\mathrm{error})$ coordinates; the slope of
$\log h^\star(n)$ against $\log n$ is fitted by least squares, with $90\%$
confidence intervals from $2000$ bootstrap resamples of the calibration
replications.  The fitted slopes are
$-0.50\,[-0.53,-0.49]$, $-0.40\,[-0.41,-0.38]$ and $-0.37\,[-0.39,-0.34]$ for
$\beta=0.5,0.75,1$, compared with the predicted balance $-1/(2\beta+1)$, i.e.
$-0.500$, $-0.400$ and $-0.333$.  The first two predicted values lie inside the
corresponding intervals; the third lies just outside its interval on the
less-negative side (\(-0.333\) versus the upper endpoint \(-0.34\)).  Thus the
finite-grid calculations reproduce the predicted balance for
\(\beta=0.5,0.75\) and are close, but not interval-consistent, for \(\beta=1\);
they are not reported as a verification of the theorem.
\Cref{fig:exp-coverage} is the companion diagnostic in which the metric is the
exact conditional-coverage error $|\widehat{\mathrm{cov}}-(1-\alpha)|$ (a
$+\infty$ threshold gives coverage $1$ and hence error $0.1$); it exhibits the
same U-shape and the same $+\infty$ frequencies.

\begin{figure}[t]
  \centering
  \includegraphics[width=\linewidth]{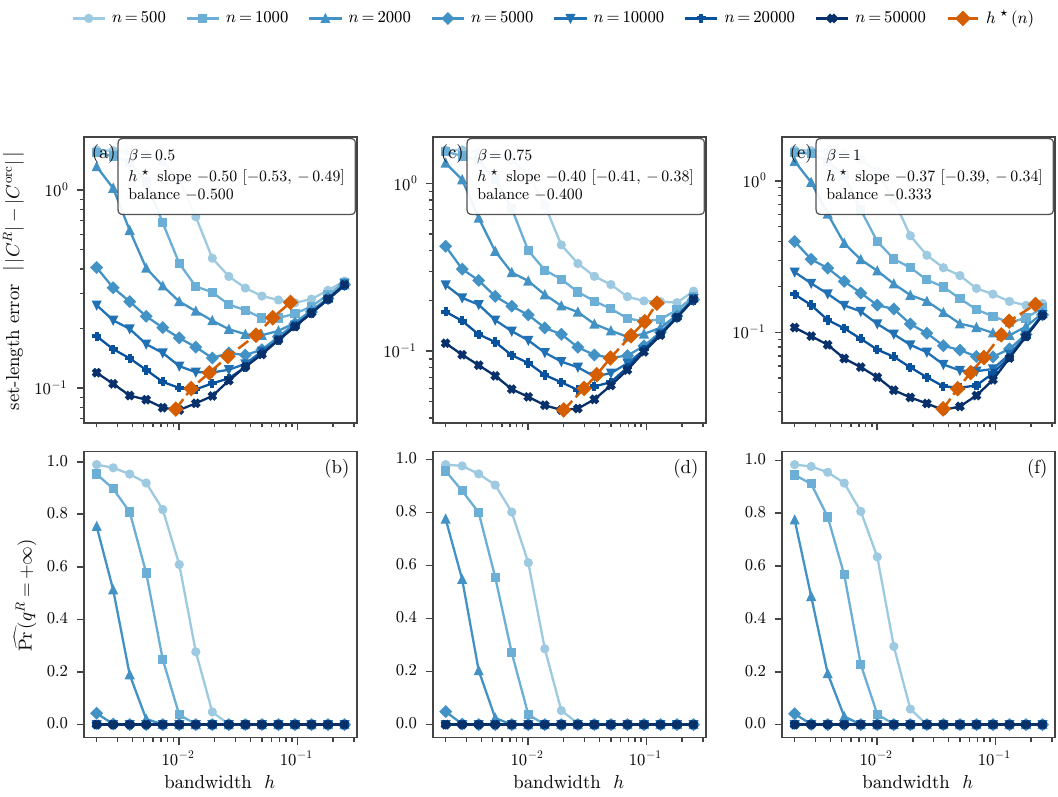}
  \caption{Bias--variance trade-off of the formal RLCP \emph{set-length} error in
  the fixed-score cusp model ($d=1$, $\alpha=0.1$, $R=1000$).  Top row: mean
  set-length error
  $\bigl|\,|C_{\mathrm{RLCP}}(x_0)|-|C_{\mathrm{oracle}}(x_0)|\bigr|$ against the
  bandwidth $h$, for all seven $n\in\{500,\dots,5\times10^4\}$ and
  $\beta=0.5,0.75,1$ (columns a,c,e).  Bottom row: the formal-threshold
  frequency $\widehat{\Pr}(q^{\mathrm{R}}=+\infty)$ against $h$, which reaches
  $0.99$ at the smallest $(n,h)$ and vanishes for moderate $h$.  Orange markers
  trace the fitted optimum $h^\star(n)$; insets report the least-squares slope
  of $h^\star(n)$ with a $90\%$ bootstrap interval and the theoretical exponent
  $-1/(2\beta+1)$.  Infinite thresholds are kept and mapped to the shared
  envelope $M=4$, so the length error remains finite throughout.}
  \label{fig:exp-bandwidth}
\end{figure}

\begin{figure}[t]
  \centering
  \includegraphics[width=\linewidth]{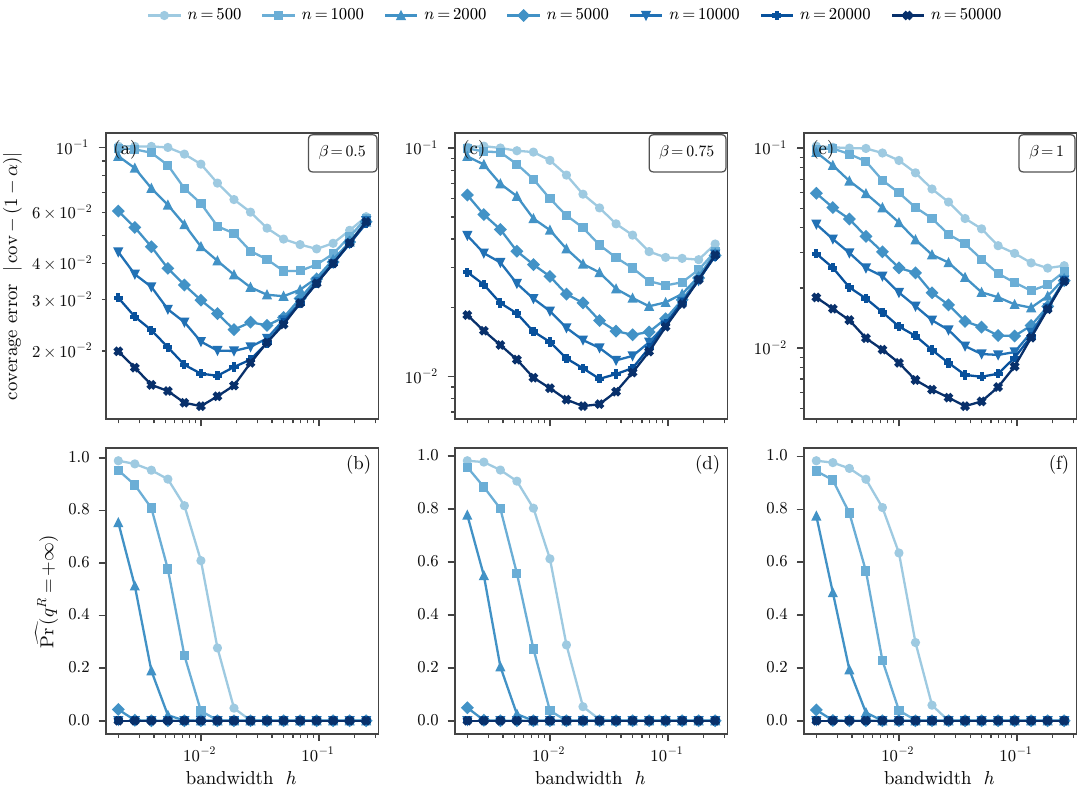}
  \caption{Companion conditional-coverage diagnostic for the bandwidth study of
  \Cref{fig:exp-bandwidth}, same design and grid.  Top row: exact
  conditional-coverage error $|\widehat{\mathrm{cov}}-(1-\alpha)|$ against $h$
  (a $+\infty$ threshold yields coverage $1$ and error $0.1$), for
  $\beta=0.5,0.75,1$ (columns a,c,e).  Bottom row: the formal-threshold
  frequency $\widehat{\Pr}(q^{\mathrm{R}}=+\infty)$.  The trade-off mirrors the
  set-length error, confirming that the two errors are driven by the same
  localization mechanism.}
  \label{fig:exp-coverage}
\end{figure}

\runinhead{Oracle-tracking protocol.}
For \Cref{fig:exp-tracking} of the main text, the bandwidth is
$h=n^{-1/(2\beta+d)}$ with unit multiplicative constant, and
\[
 n\in\{500,1000,2000,5000,10^4,2\times10^4,5\times10^4\}.
\]
The study uses the same formal $+\infty$ handling and shared envelope $M=4$ as
the bandwidth study above; at the balanced bandwidth the $+\infty$ frequency is
$0$ for every $(\beta,n)$, so the shared envelope is a disclosure and a parity
with \Cref{fig:exp-bandwidth}, not a change to the data, and the reported slopes
are unaffected.  The reported exponents are least-squares slopes in log--log
coordinates.  Markers are empirical means over the $300$ calibration
replications, error bars are twice the estimated standard error, and shaded
regions are empirical interquartile ranges.  In this model, while the threshold
remains in the uniform-noise support, the absolute conditional-coverage error is
exactly $(2B)^{-1}$ times the absolute threshold error and the absolute
one-sided set-length error equals the absolute threshold error, so the two
panels use the same simulated thresholds and differ only by this deterministic
transformation.

\runinhead{Learned-score diagnostics: worked estimators and uniform error.}
The learned-score results,
\Cref{thm:adaptive_rlcp_uniform_length,thm:adaptive_rlcp_uniform_coverage},
control the oracle errors through the \emph{uniform} score-estimation error
\begin{equation}
  \label{eq:exp-uniform-assum}
  \esssup_{x\in\Xset}\DeltaS(x)\;\le\;\epsilon^{\Delta}_{\infty}(m,\delta),
\end{equation}
so any experiment that speaks to those bounds must measure a sup-norm, not a
localized $\mathrm{L}^2(\Lambda)$ average.  We instantiate the two worked
examples of the paper: the local-linear quantile pair of
\cref{prop:ex-cqr} (a \CQR{} score) and the Nadaraya--Watson conditional-CDF
estimator of \cref{prop:ex-pit-standard} (a distributional, PIT score).  The
data-generating process has $X\sim\mathrm{Unif}[0,1]$ and is the
heteroscedastic cusp model
\begin{equation}
  \label{eq:exp-dgp-learned}
  Y=|X-x_0|^{\beta}+\sigma(X)\,\eps,\qquad
  \sigma(x)=0.75+0.25\cos(2\pi x),\qquad
  \eps\sim\mathrm{Unif}[-1,1],
\end{equation}
with $d=1$, $x_0=1/2$, level $\alpha=0.1$ (so the \CQR{} target levels are
$0.05$ and $0.95$), and $\beta\in\{0.5,1\}$.  Neither figure uses a neural-network
estimator: a ReLU quantile network plateaus in sup-norm and therefore cannot
speak to \eqref{eq:exp-uniform-assum}, so it is not reported.

The conditional support of \eqref{eq:exp-dgp-learned} is
$[f(x)-\sigma(x),f(x)+\sigma(x)]$, which varies with $x$; the common-support
assumption \Cref{assumP:density} (and its constant-$\mu_{\mathrm{low}}$
variant \Cref{assumP:density-unif} used by \cref{prop:ex-cqr}) therefore holds
only \emph{locally} on that support, where the conditional density is bounded
below by $\mu_{\mathrm{low}}=(2\sigma_{\max})^{-1}=1/2$, and not globally on the
fixed envelope $[-M,M]$.  Accordingly the learned-score study illustrates the
predicted \emph{rate exponent}, not the multiplicative constants of the bounds.

\runinhead{Uniform score-estimation rate.}
\Cref{fig:supp-c1} reports computable uniform proxies for the essential supremum
in \eqref{eq:exp-uniform-assum}.  Against the training size $m$ the two
estimators use the natural sup-norm metric for their respective scores,
\begin{equation}
  \label{eq:exp-c1-metric}
  \widehat{\epsilon}^{\Delta,\mathrm{CQR}}_{\infty}(m)
   =\max_{x\in G}\;\max_{\tau\in\{0.05,0.95\}}
     \bigl|\widehat q_\tau(x)-q_\tau(x)\bigr|,
  \qquad
  \widehat{\epsilon}^{\Delta,\mathrm{PIT}}_{\infty}(m)
   =\max_{x\in G}\;\sup_{y\in\Yset}
     \bigl|\widehat F_{Y\mid X}(y\mid x)-F_{Y\mid X}(y\mid x)\bigr|.
\end{equation}
Here $G$ is a \emph{fixed deterministic} grid over the whole covariate domain,
with $450$ distinct nodes: $401$ uniform nodes at mesh $2.5\times10^{-3}$ over
$[0,1]$ and $50$ graded nodes clustered near the cusp $x_0=0.5$, which give
$451$ named nodes minus the single coincidence at $x=0.3$ (already a uniform
node).  For the PIT metric the supremum over $y$ is computed \emph{exactly}, at
the jumps of the local empirical CDF and using both one-sided limits, so it is
not degraded by a $y$-discretization.  The estimator bandwidth is
$b_m=(\log m/m)^{1/(2\beta+d)}$, and the $m$-grid runs over
$\{10^3,2\times10^3,5\times10^3,10^4,2\times10^4,5\times10^4,10^5\}$.  For each
$m$ we draw $100$ independent training seeds and report the mean, the standard
error, and the empirical $0.95$ quantile of the corresponding metric across
seeds, since \eqref{eq:exp-uniform-assum} is a high-probability statement.
Log--log slopes are fitted on $m\ge2000$ against $m/\log m$, whose theoretical
exponent is $-\beta/(2\beta+d)$.  Refining the covariate grid by factors of two
and four (to $900$ and $1800$ nodes) leaves the fitted slopes unchanged to
within $5\times10^{-4}$, so the estimated exponents are not artifacts of the
grid.

\begin{table}[t]
\centering
\footnotesize
\caption{Fitted uniform-error exponents of \eqref{eq:exp-c1-metric} versus
$m/\log m$ (least-squares slope of the six mean errors with $m\ge2000$; the
displayed $\pm$ is a bootstrap standard error over the $100$ training seeds),
compared with the leading theoretical exponent $-\beta/(2\beta+d)$.  The
local-linear \CQR{} pair runs slightly steeper than the leading exponent, an
extreme-quantile pre-asymptotic effect; the Nadaraya--Watson estimator is
closest to theory.  DGP-B has $x$-dependent support, so these slopes illustrate
the rate exponent, not the constants, and are not an asymptotic test.}
\label{tab:c1-slopes}
\begin{tabular}{ll cc}
\toprule
Estimator & $\beta$ & Fitted slope & Theory $-\beta/(2\beta+d)$ \\
\midrule
\multirow{2}{*}{\cref{prop:ex-cqr} (local-linear \CQR)} & $0.5$ & $-0.293\pm0.008$ & $-0.250$ \\
 & $1$ & $-0.403\pm0.007$ & $-0.333$ \\
\midrule
\multirow{2}{*}{\cref{prop:ex-pit-standard} (Nadaraya--Watson CDF)} & $0.5$ & $-0.261\pm0.005$ & $-0.250$ \\
 & $1$ & $-0.370\pm0.005$ & $-0.333$ \\
\bottomrule
\end{tabular}
\end{table}

\begin{figure}[t]
  \centering
  \includegraphics[width=\linewidth]{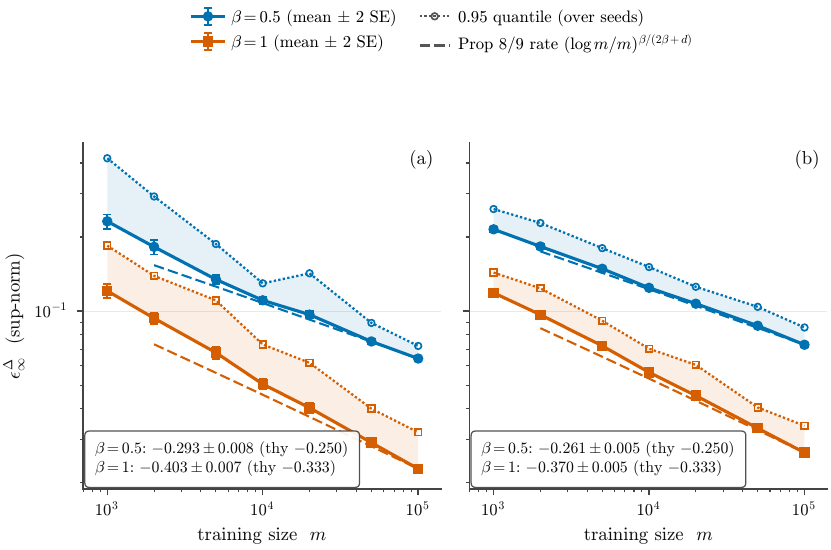}
  \caption{Uniform score-estimation error metrics \eqref{eq:exp-c1-metric}
  against the training size $m$, for the local-linear \CQR{} pair
  (\cref{prop:ex-cqr}) and the Nadaraya--Watson conditional-CDF estimator
  (\cref{prop:ex-pit-standard}), at $\beta=0.5,1$.  Each metric is a numerical
  proxy for the essential supremum in \eqref{eq:exp-uniform-assum}, evaluated on
  a fixed $450$-node grid over the covariate domain (with the PIT supremum over
  $y$ taken exactly).  Curves show the mean over $100$ training seeds; the upper
  edges of the shaded bands are the empirical $0.95$ quantiles.  Fitted slopes
  and theoretical exponents are reported in \cref{tab:c1-slopes}.}
  \label{fig:supp-c1}
\end{figure}

\runinhead{Empirical decomposition of the learned-score oracle errors.}
\Cref{fig:supp-c2} plots the measured RLCP oracle errors against an
\emph{empirical decomposition proxy} for the two terms that
\Cref{thm:adaptive_rlcp_uniform_length,thm:adaptive_rlcp_uniform_coverage}
isolate.  The solid curves are the measured RLCP length and conditional-coverage
errors, each with a hierarchical-bootstrap $95\%$ band resampling over the
$T=40$ training seeds and the $R=8$ calibration replications.  The proxy combines a
threshold-scale calibration term and the uniform score-estimation term into a
single quantity
\begin{equation}
  \label{eq:exp-c2-bracket}
  \widehat B_{t,r}
   =\widehat\Delta^{\mathrm{cal}}_{n,h;t,r}
     +\Bigl(1+\frac{1}{\lows[\Sstar]{\tilde x}}\Bigr)\,
      \widehat\epsilon^{\Delta}_{\infty;t}(m),
   \qquad
   \widehat\Delta^{\mathrm{cal}}_{n,h;t,r}
    =\bigl|\,\widehat q^{\,\mathrm{R},\star}_{n,h;t,r}-\tau^\star\,\bigr|,
\end{equation}
in which $\widehat\Delta^{\mathrm{cal}}_{n,h;t,r}$ is the perfect-score RLCP
threshold deviation---the deviation of the oracle-score ($\Sstar$) RLCP
threshold $\widehat q^{\,\mathrm{R},\star}_{n,h;t,r}$ from the pivotal target
$\tau^\star$, computed at the same $n,h$, calibration sample, test covariate,
and realized centre as the learned run and \emph{not} divided by
$\lows[\Sstar]{\tilde x}$---and
$\widehat\epsilon^{\Delta}_{\infty;t}(m)$ is the C1 estimate of
\eqref{eq:exp-c1-metric}.  Writing $\overline B$ for the mean of
$\widehat B_{t,r}$ over the $T=40$ seeds and $R=8$ replications, the two dashed
curves overlaid on each row are the \emph{same} $\widehat B$ scaled to the
panel: $\Llen\cdot\overline B$ on the length rows and $\ups\cdot\overline B$ on
the coverage rows, with no per-metric recomputation.  The pivotal target, the
density-minorization constant, and the factors used to scale the diagnostic are
analytic (at $\beta=1$ and $h=0.15$): $\tau^\star=0$,
$\lows[\Sstar]{\tilde x}\approx1.53$, $\Llen=2$ and $\ups=2$ for the \CQR{}
example (\cref{prop:ex-cqr}); and $\tau^\star=(1-\alpha)/2=0.45$,
$\lows[\Sstar]{\tilde x}=2$, the in-support length factor
$\Llen=2/\mu_{\mathrm{low}}=4\sigma_{\max}=4$, and $\ups=2$ for the PIT example
(\cref{prop:ex-pit-standard}).  Because the conditional support varies with
$x$, the PIT set-length map jumps from the conditional support to the full
response envelope at threshold $1/2$ and therefore does not satisfy the global
Lipschitz requirement in \Aref{assum:S_len}{S}.  Across the simulated grid the
PIT set never reaches this jump (support-escape frequency $0$; maximum set
length $\approx1.47\ll 2M=5.0$).  Its length panel is thus an in-support
diagnostic of the contraction associated with the theorem's error terms, not
an application of \cref{thm:adaptive_rlcp_uniform_length}.  This overlay is an
empirical decomposition \emph{diagnostic}---\emph{not}
a numerical bound or confidence set: it uses the \emph{measured}
$\widehat\Delta^{\mathrm{cal}}$ and $\widehat\epsilon^{\Delta}_{\infty}$ and omits
the confidence level $\delta$, the hidden multiplicative constants, and the
sample-size conditions of the theorems.  The figure has a $2\times2$ layout: the
rows are the absolute length difference from the pivotal oracle
(\cref{thm:adaptive_rlcp_uniform_length}) and the conditional-coverage error
(\cref{thm:adaptive_rlcp_uniform_coverage}); the columns vary $m$ at fixed
$n=500$ and vary $n$ at fixed $m=10^5$.  Empty windows use the formal $+\infty$
threshold---returning the full response envelope---rather than the global-\CQR{}
fallback of the real-data implementation; the perfect-score
$+\infty$-threshold frequency is $0.000$ for both examples and both sweeps, so
no replication is excluded.  The experiment compares the observed contraction
with the two sources of error isolated by the theorem: for the \CQR{} example
the varying-$m$ length error falls from $0.087$ to $0.033$ while the scaled
proxy decreases from $0.591$ to $0.100$, and for the PIT example the
corresponding quantities decrease from $0.100$ to $0.037$ and from $0.971$ to
$0.211$, respectively.  These paired decreases illustrate contraction along
the two terms isolated by the theorem; their relative magnitudes are not
interpreted as a bound.

\begin{figure}[t]
  \centering
  \includegraphics[width=0.96\linewidth]{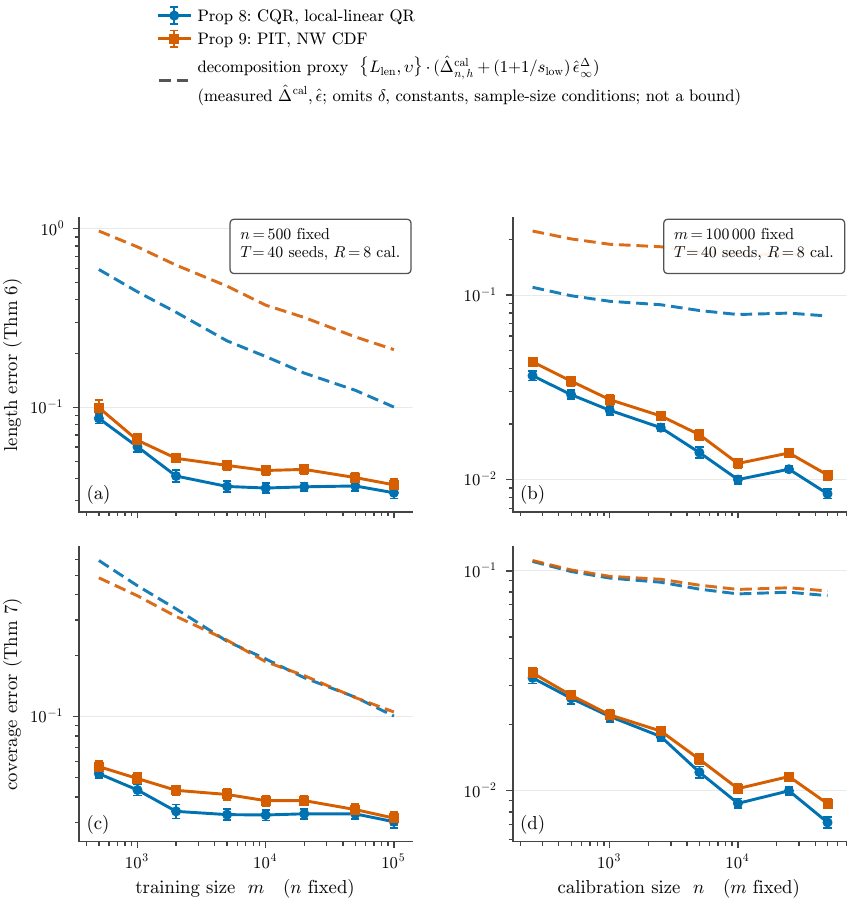}
  \caption{Measured RLCP oracle errors versus the empirical decomposition proxy
  \eqref{eq:exp-c2-bracket} of the
  \Cref{thm:adaptive_rlcp_uniform_length,thm:adaptive_rlcp_uniform_coverage}
  terms, at $\beta=1$ and localization bandwidth $h=0.15$.  Rows: absolute
  length difference from the pivotal oracle (top,
  \cref{thm:adaptive_rlcp_uniform_length}) and conditional-coverage error
  (bottom, \cref{thm:adaptive_rlcp_uniform_coverage}).  Columns: varying $m$ at
  fixed $n=500$ (left) and varying $n$ at fixed $m=10^5$ (right).  Solid curves
  are the measured errors with a hierarchical-bootstrap $95\%$ band over training
  seeds and calibration replications.  The dashed curves are the scaled
  decomposition proxy $\{\Llen,\ups\}\cdot(\widehat\Delta^{\mathrm{cal}}_{n,h}
  +(1+1/\lows[\Sstar]{\tilde x})\,\widehat\epsilon^{\Delta}_{\infty})$ formed from
  the \emph{measured} $\widehat\Delta^{\mathrm{cal}}$ and
  $\widehat\epsilon^{\Delta}_{\infty}$; it omits $\delta$, the hidden constants,
  and the sample-size conditions, and is a decomposition diagnostic, \emph{not} a
  bound.  Empty windows use the formal $+\infty$ threshold; its perfect-score
  frequency is $0$.  The PIT factor $\Llen=4$ is an in-support diagnostic scaling
  only: varying support violates \Aref{assum:S_len}{S} at threshold $1/2$, but
  no simulated threshold reaches that jump.}
  \label{fig:supp-c2}
\end{figure}

\runinhead{Real-data comparison with global conformal baselines}
The real-data experiments compare RLCP with two global baselines only:
\textsf{Split-CP}, using a residual score, and \textsf{Split-CQR}, using the
same fitted \CQR{} score as RLCP but calibrated with a single global threshold.
Thus the contrast between \textsf{Split-CQR} and RLCP isolates the calibration
rule: global calibration versus randomized localized calibration.  The target is
$1-\alpha=0.90$.  The datasets are the UCI \texttt{concrete} (id~165),
\texttt{bike} (hourly, id~275), \texttt{protein} (CASP, id~265, loaded from
the UCI direct CSV), and \texttt{airfoil} (id~291) regression problems
\citep{yeh1998concrete,fanaee2013bike,rana2013protein,brooks1989airfoil},
together with a
semi-synthetic \texttt{bike-synth} problem built from the real \texttt{bike}
covariates with a synthetic heteroskedastic target whose noise scale grows in
low-density regions
($\sigma(x)=\sigma_0\exp\{k(\mathrm{median}\log\widehat p-\log\widehat p(x))\}$
clipped to $[0.4\sigma_0,5\sigma_0]$, $\sigma_0=1$, $k=0.6$, and mean signal
$f(x)=1.5x_1+x_2$), so that global constant-width methods have a conditional
coverage gap by construction.  For \texttt{bike}, the response is
\texttt{cnt} and the retained $12$ covariates, in their data-file order, are
\texttt{season}, \texttt{yr}, \texttt{mnth}, \texttt{hr}, \texttt{holiday},
\texttt{weekday}, \texttt{workingday}, \texttt{weathersit}, \texttt{temp},
\texttt{atemp}, \texttt{hum}, and \texttt{windspeed}.  The date
\texttt{dteday} and record identifier \texttt{instant} are excluded, as are
\texttt{casual} and \texttt{registered}, since these are components of the
target \texttt{cnt} and retaining them would leak the response.
\texttt{bike-synth} uses exactly the same $12$-dimensional covariate vector.
The integer-coded categorical \texttt{bike} fields are retained as supplied
and are not one-hot encoded; for the other datasets, only numeric predictor
fields are retained.  Each run uses $20$ random $50/25/25$
train/calibration/test splits, with the split, base-learner, bandwidth,
density-decile, and auxiliary-centre randomness all driven by the master seed
$1000\cdot(\text{dataset index})+s$ for split $s$; sets larger than the cap
$n_{\max}=8000$ are subsampled without replacement using a fixed seed $0$
(\texttt{concrete} and \texttt{airfoil} are used whole).  Covariates are
standardized on the \emph{training} fold, using the per-coordinate mean and
standard deviation of the training split only.  All methods use
histogram-gradient-boosting base learners with $200$ boosting iterations,
learning rate $0.1$, at most $31$ leaves, and at least $20$ observations per
leaf.  The conditional mean is fitted for the residual score of
\textsf{Split-CP}; the conditional $0.05$ and $0.95$ quantiles are fitted for
\textsf{Split-CQR} and RLCP.  The global conformal levels use the finite-sample
correction $\lceil(n_{\rm cal}+1)(1-\alpha)\rceil/n_{\rm cal}$.
The localizer is the compactly supported Epanechnikov kernel of
\cref{assumK:kernel_assum}, $K(u)\propto(1-\|u\|^2)_+$, applied to standardized
covariates, with bandwidth
\[
  h \;=\; 2\,\mathrm{median}_{i,j\in\mathcal{S}}\|X_i-X_j\|\;
  n_{\rm cal}^{-1/(2+d)},
\]
where the median pairwise-distance scale is computed on $\mathcal{S}$, a
deterministic subsample of at most $500$ calibration covariates drawn without
replacement from $\Dcal[n_{\rm cal}]$ using seed $1000j+s+777$ (for dataset
index $j$ and split $s$).  This corresponds to the heuristic $\beta=1$; the
constant $2$ is fixed once across
all datasets so that the compact localization windows remain populated in
moderate dimension.  The RLCP auxiliary centre is drawn from the same forward
kernel, $\tilde x = x + hV$ with $V\sim K$ (sampled by an exact Devroye scheme in
$d=1$ and by rejection from a uniform-on-ball proposal in $d\ge2$), as in the
analysed procedure.  To summarize conditional behaviour, we fit a Gaussian KDE
to at most $2000$ training covariates (the estimate defines the evaluation bins
only, and enters no conformal step), bin the test points into deciles of the
estimated density $p_X$, and report coverage and length within each decile.
Decile $1$ is the sparsest region and decile $10$ the densest.  The exact
sources, identifiers, covariate lists, and preprocessing are collected in the
supplementary computational archive posted with the article.

These experiments depart from the formal setting, so the reported RLCP is a
disclosed \emph{hybrid} rather than formal RLCP.  First, the bandwidth is a
data-driven heuristic---a median-distance scale with a fixed constant---rather
than the theoretical bandwidth of the bounds.  Second, the base learners are
gradient-boosted trees rather than the estimator class of the learned-score
analysis.  Third, the decile bins are estimated summaries of conditional
behaviour, not conditioning events covered by the theorems.  Fourth, and most
importantly, the formal RLCP threshold is $+\infty$ on the complete event
$W_{\mathrm{cal}}<\tfrac{1-\alpha}{\alpha}W_{\mathrm{test}}$ ($<9\,W_{\mathrm{test}}$),
which the hybrid splits into two sub-events with different handling: an
\emph{empty} localization window falls back to the global \CQR{} threshold,
while a \emph{nonempty but insufficient-weight} window caps the half-width at the
calibration $y$-range.  The resulting hybrid procedure is therefore evaluated
empirically; the plots should be read as diagnostics of the localization
mechanism, not as a verification of the finite-sample guarantee or of the
constants in the bounds.

\runinhead{Coverage and length across density deciles.}
\Cref{fig:supp-e7-deciles} shows the most relevant qualitative pattern.  Global
methods can under-cover in the sparse deciles and over-cover in dense regions,
especially in the semi-synthetic \texttt{bike-synth} example.  RLCP tends to
lift the lowest-density coverage.  The effect is most visible on
\texttt{bike-synth}, where the first decile is far below the target for the
global baselines and substantially closer to the target under RLCP; on the four
real datasets the worst-slice gains are smaller but consistently positive ---
most precisely estimated on \texttt{bike} and \texttt{protein}, noisier on
\texttt{concrete} and \texttt{airfoil}.  In middle and dense deciles the three
methods are much closer.  Thus the figure supports a localized-calibration
mechanism, not a uniform dominance statement.

The length panel explains the cost.  RLCP is often longer than the global
baselines in the sparse deciles, precisely where the effective local calibration
size is smallest.  In dense deciles the length gap is usually smaller and can
nearly disappear.  This is the empirical counterpart of the calibration rate
$A_{n,h}^{\mathrm{cal}}$ defined in \eqref{eq:th-bound-definition}: local calibration
can correct a density-local
coverage deficit only by using a threshold estimated from a smaller effective
sample.

\begin{figure}[t]
  \centering
  \includegraphics[width=\linewidth]{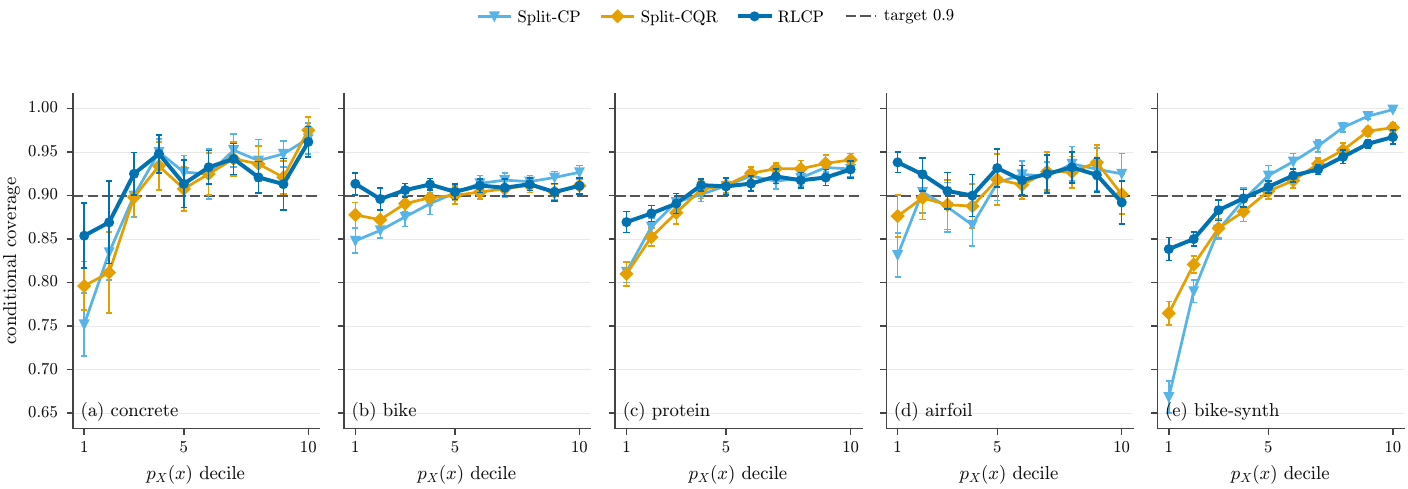}
  \vspace{0.35em}
  \includegraphics[width=\linewidth]{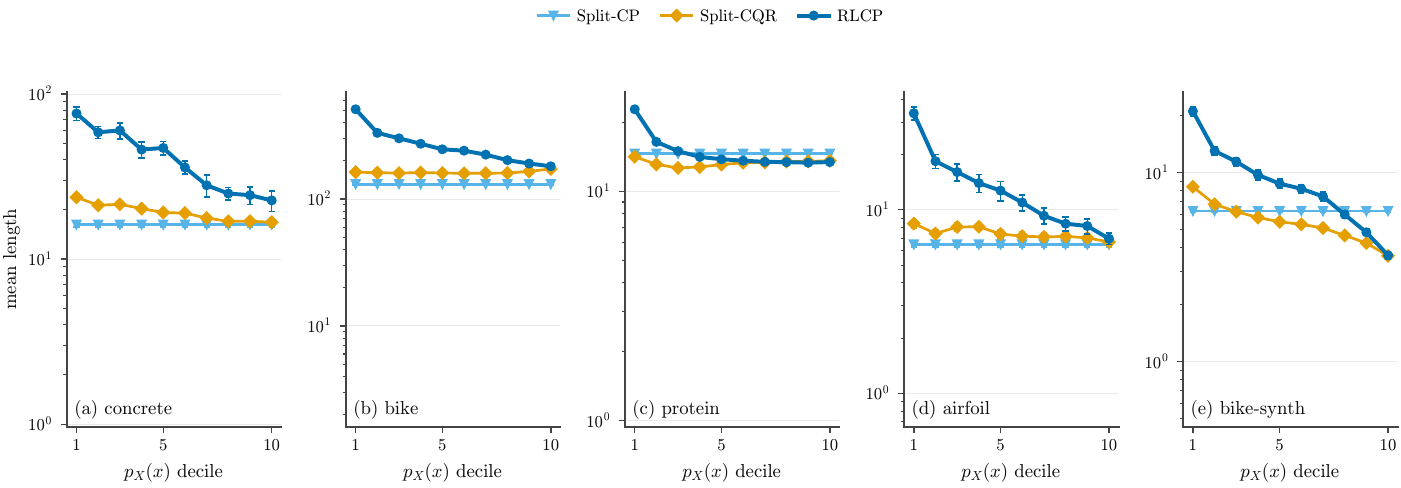}
  \caption{Real-data decile diagnostics.  Top: conditional coverage by decile of
  KDE-estimated $p_X(x)$, with decile $1$ the sparsest region and the dashed line
  the target $0.90$.  Bottom: mean interval length in the same deciles.  RLCP can
  reduce the most visible sparse-region coverage deficits, but the length plot
  shows the corresponding price of localized calibration.}
  \label{fig:supp-e7-deciles}
\end{figure}

\runinhead{Numerical summary.}
\Cref{tab:e7-realdata} complements the decile plots with a paired aggregate
summary.  The marginal coverage column is close to the target for all three
procedures --- the hybrid RLCP attains $0.918$, $0.908$, $0.907$, $0.919$, and
$0.910$ on \texttt{concrete}, \texttt{bike}, \texttt{protein}, \texttt{airfoil},
and \texttt{bike-synth} --- so the relevant comparison is how coverage is
distributed across estimated density deciles and how much length is paid for that
redistribution.  RLCP improves the worst decile relative to \textsf{Split-CP} on
all five datasets: by $+0.039$ to $+0.054$ on the four real datasets and by
$+0.163$ on \texttt{bike-synth}.  The \emph{Fallback\,\%} column reports the
\emph{total} formal-RLCP $+\infty$ frequency (empty $\cup$ insufficient-weight
windows), and the starred $\Delta$WS entries mark a Holm-adjusted paired
\emph{randomization} $p<0.05$; both quantities are descriptive at $20$ splits.
The length cost is heterogeneous: it is nearly neutral on \texttt{protein}
(mean-length ratio $1.02\pm0.01$), moderate on \texttt{bike-synth}
($1.51\pm0.06$), and large on \texttt{concrete}, \texttt{bike}, and
\texttt{airfoil} (ratios $2.63$, $2.07$, and $2.15$).  \textsf{Split-CQR} remains
a strong global baseline; for instance, it is shorter than \textsf{Split-CP} on
\texttt{protein} and \texttt{bike-synth}, and on \texttt{bike-synth} already
recovers a substantial part of the low-density coverage deficit.  Thus the table
reinforces the same message as the figures: localized calibration is most useful
where global calibration visibly undercovers, but its cost must be measured in
interval length.

\begin{table}[t]
\centering
\footnotesize
\setlength{\tabcolsep}{3.5pt}
\caption{Real-data conditional coverage at $\alpha=0.10$ (target $0.90$). Mean $\pm$ 2\,SE over 20 random 50/25/25 splits. \emph{WS} = worst-slice (min) coverage across KDE-$p_X$ deciles; \emph{Len.} = mean length (dataset units). \emph{$\Delta$WS} and \emph{Len.\ ratio} are paired (per-split) vs Split-CP ($^{*}$: Holm-adjusted paired \emph{randomization} $p<0.05$ for $\Delta$WS, across the dataset$\times$comparison family; paired-bootstrap CIs and per-comparison $p$-values in \texttt{results/E7\_inference.csv}; these remain descriptive at 20 splits). \emph{Fallback} is the fraction of test points on the complete formal-RLCP $+\infty$ event ($W_{\mathrm{cal}}<9W_{\mathrm{test}}$): an empty window ($\to$ global CQR) or a nonempty but insufficient-weight window ($\to$ $y$-range cap). On those points the implementation is a \emph{hybrid} (validity assessed empirically); the empty/insufficient split and a calibration-range sensitivity comparison are in \texttt{results/E7\_fallback\_frequency.csv}. \texttt{airfoil} and \texttt{bike-synth} (real $X$, synthetic $\sigma(x)\!\sim\!1/p_X$) are the heteroskedastic settings; \texttt{concrete} is small-$n$/noisy.}
\label{tab:e7-realdata}
\begin{tabular}{ll ccc cc c}
\toprule
Dataset & Method & Coverage & WS & Len. & $\Delta$WS vs SCP & Len.\ ratio & Fallback\ \% \\
\midrule
\multirow{3}{*}{\texttt{concrete}} & Split-CP & $0.909\pm0.010$ & $0.735\pm0.030$ & $16.2\pm0.51$ & -- & -- & -- \\
 & Split-CQR & $0.904\pm0.008$ & $0.754\pm0.034$ & $19.3\pm0.49$ & $+0.019\pm0.035$ & $1.20\pm0.05$ & -- \\
 & RLCP &$0.918\pm0.008$ & $0.788\pm0.037$ & $42.5\pm1.8$ & $+0.054\pm0.039$ & $2.63\pm0.14$ & $16.26\pm1.23$ \\
\midrule
\multirow{3}{*}{\texttt{bike}} & Split-CP & $0.897\pm0.004$ & $0.835\pm0.010$ & $131\pm1.9$ & -- & -- & -- \\
 & Split-CQR & $0.898\pm0.004$ & $0.854\pm0.009$ & $162\pm1.9$ & $+0.019\pm0.011^{*}$ & $1.24\pm0.02$ & -- \\
 & RLCP &$0.908\pm0.004$ & $0.874\pm0.008$ & $270\pm4.1$ & $+0.039\pm0.012^{*}$ & $2.07\pm0.04$ & $5.78\pm0.26$ \\
\midrule
\multirow{3}{*}{\texttt{protein}} & Split-CP & $0.900\pm0.005$ & $0.812\pm0.012$ & $14.6\pm0.14$ & -- & -- & -- \\
 & Split-CQR & $0.903\pm0.004$ & $0.809\pm0.013$ & $13.3\pm0.079$ & $-0.004\pm0.013$ & $0.91\pm0.01$ & -- \\
 & RLCP &$0.907\pm0.004$ & $0.857\pm0.009$ & $15\pm0.11$ & $+0.045\pm0.012^{*}$ & $1.02\pm0.01$ & $3.91\pm0.24$ \\
\midrule
\multirow{3}{*}{\texttt{airfoil}} & Split-CP & $0.904\pm0.007$ & $0.796\pm0.023$ & $6.47\pm0.21$ & -- & -- & -- \\
 & Split-CQR & $0.908\pm0.010$ & $0.817\pm0.020$ & $7.45\pm0.25$ & $+0.021\pm0.029$ & $1.16\pm0.04$ & -- \\
 & RLCP &$0.919\pm0.007$ & $0.837\pm0.016$ & $13.8\pm0.6$ & $+0.041\pm0.026^{*}$ & $2.15\pm0.10$ & $8.95\pm0.86$ \\
\midrule
\multirow{3}{*}{\texttt{bike-synth}} & Split-CP & $0.900\pm0.004$ & $0.668\pm0.018$ & $6.26\pm0.11$ & -- & -- & -- \\
 & Split-CQR & $0.899\pm0.004$ & $0.765\pm0.013$ & $5.57\pm0.08$ & $+0.096\pm0.012^{*}$ & $0.89\pm0.01$ & -- \\
 & RLCP &$0.910\pm0.003$ & $0.832\pm0.011$ & $9.42\pm0.29$ & $+0.163\pm0.014^{*}$ & $1.51\pm0.06$ & $6.25\pm0.28$ \\
\bottomrule
\end{tabular}
\end{table}

\runinhead{Fallback diagnostics: empty versus insufficient windows.}
The empty-window count materially undercounts the departure from formal RLCP,
which is why the table reports the total.  On \texttt{concrete}, empty windows
occur for only $\approx1.0\%$ of test points (routed to the global-\CQR{}
fallback), whereas the complete formal $+\infty$ event---empty
\emph{or} nonempty-but-insufficient (routed to the $y$-range cap)---occurs for
$\approx16.3\%$.  The other datasets show the same ordering, with total
frequencies $5.78\%$ (\texttt{bike}), $3.91\%$ (\texttt{protein}), $8.95\%$
(\texttt{airfoil}), and $6.25\%$ (\texttt{bike-synth}) against empty-window
frequencies below $0.6\%$ in every case.  Reporting only the empty-window
fraction would therefore have hidden the bulk of the hybrid behaviour.  For
comparison, we also ran a \emph{calibration-range sensitivity variant} that
maps both sub-events producing the formal $+\infty$ threshold to the per-split
interval $[Y_{\min}^{\mathrm{cal}},Y_{\max}^{\mathrm{cal}}]$.  This
data-dependent interval is not the fixed response envelope of formal RLCP and
need not contain a new test response.  Consequently, this variant is evaluated
only as an empirical sensitivity analysis and carries no formal finite-sample
coverage guarantee.  Its marginal coverage is essentially identical to the
hybrid's (e.g. $0.916$ versus $0.918$ on \texttt{concrete}, $0.908$ versus
$0.908$ on \texttt{bike}), while its intervals are uniformly shorter (mean
lengths $29.1$ versus $42.5$ on \texttt{concrete}, $214$ versus $270$ on
\texttt{bike}, $10.2$ versus $13.8$ on \texttt{airfoil}).  These values are
reported solely to quantify the effect of replacing the hybrid's two fallback
rules by one common finite-range rule.

\runinhead{Stronger inference.}
Because the theorems make no claim about the hybrid, the aggregate comparison is
reported descriptively at $20$ splits.  For each dataset we compute the paired
per-split worst-slice-coverage gain and mean-length ratio of RLCP against
\textsf{Split-CP}, with paired-bootstrap $95\%$ confidence intervals (resampling
splits), a two-sided paired sign-flip randomization $p$-value, and a
Holm adjustment across the whole dataset\,$\times$\,comparison family.  The RLCP
worst-slice gain is Holm-significant on \texttt{bike}
($+0.039$, $95\%$ CI $[0.028,0.051]$), \texttt{protein}
($+0.045$, $[0.033,0.056]$), \texttt{airfoil} ($+0.041$, $[0.017,0.067]$), and
\texttt{bike-synth} ($+0.164$, $[0.150,0.178]$), but \emph{not} on
\texttt{concrete} ($+0.054$, $[0.015,0.090]$; Holm $p\approx0.07$), where the
small, noisy sample widens the interval.  These results are consistent with the
decile figures and, being descriptive, are reported as evidence of a stable
worst-slice improvement rather than as a formal guarantee.

\runinhead{Bandwidth sensitivity.}
\Cref{fig:supp-e7-bandwidth} recomputes the real-data study at bandwidth
multipliers $c\in\{0.5,1,2,4\}$ around the heuristic scale (the reported study
uses $c=2$).  The multiplier controls the effective local calibration size and
hence the fallback frequency: at $c=0.5$ the windows are nearly empty (effective
local sizes of $1$--$13$ points and total fallback frequencies of $65$--$98\%$,
so the method is dominated by its fallback), while at $c=4$ the windows are dense
(hundreds to over a thousand points, fallback below $3.3\%$) and the intervals
are shorter but less locally adaptive.  Worst-slice coverage peaks near $c=1$ for
most datasets and then declines as the localization coarsens.  The conclusions
are correspondingly modest: RLCP buys an improvement in worst-slice coverage at a
cost in interval length, and the reported $c=2$ setting is a compromise between a
populated window and a genuinely local one, not a tuned optimum.

\begin{figure}[t]
  \centering
  \includegraphics[width=\linewidth]{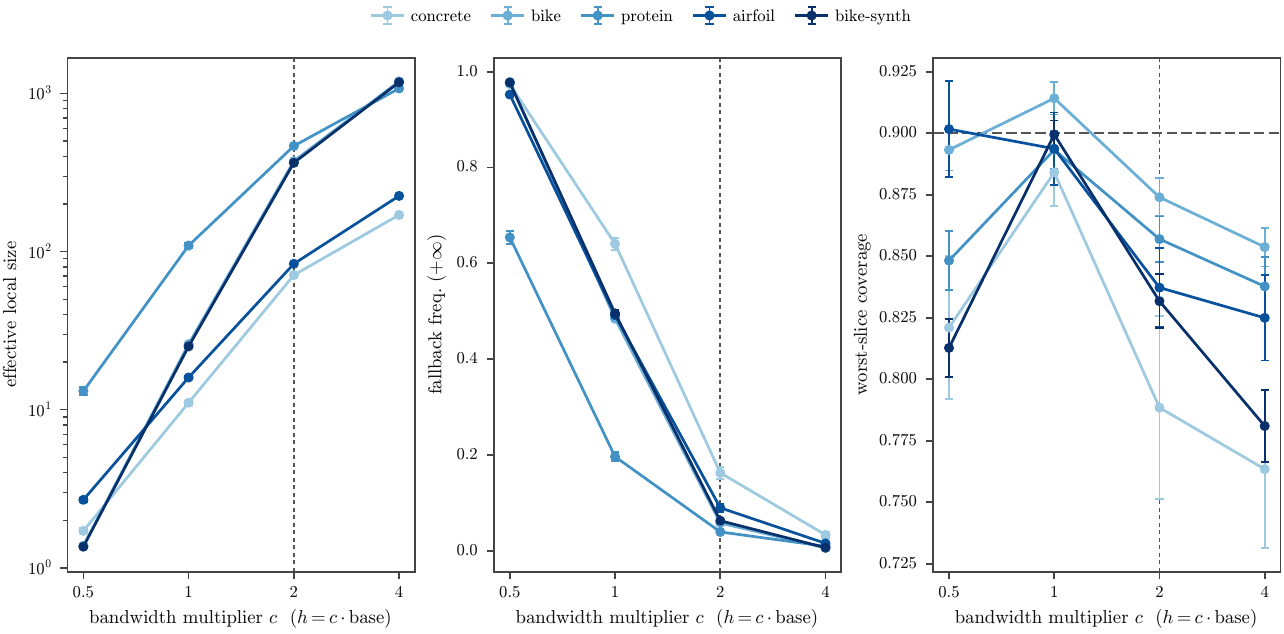}
  \caption{Real-data bandwidth sensitivity.  Each quantity is recomputed at
  bandwidth multipliers $c\in\{0.5,1,2,4\}$ (with $h=c\cdot$ the median-distance
  base scale; the reported study uses $c=2$), over the $20$ splits.  Panels show
  the effective local calibration size, the total formal $+\infty$ fallback
  frequency, and worst-slice coverage as
  functions of $c$.  Small $c$ yields near-empty windows dominated by the
  fallback; large $c$ yields dense windows with shorter but less adaptive
  intervals; worst-slice coverage peaks near $c=1$ for most datasets.}
  \label{fig:supp-e7-bandwidth}
\end{figure}

\runinhead{Aggregate coverage--length trade-off.}
\Cref{fig:supp-e7-tradeoff} compresses the decile plots and
\cref{tab:e7-realdata} into a paired comparison with \textsf{Split-CP}.  The
vertical coordinate is the gain in worst-slice coverage, where the worst slice is
the minimum over the ten density deciles; the horizontal coordinate is the
mean-length ratio relative to \textsf{Split-CP}.  The desirable quadrant is
upper-left: better worst-slice coverage at shorter length.  RLCP points lie above
the horizontal axis on all five datasets, but on or to the right of one in length
(\texttt{protein} essentially at parity).  Thus RLCP consistently improves the
worst density slice, usually by paying a length premium.  \textsf{Split-CQR} is
closer to \textsf{Split-CP} in length and sometimes yields a smaller or more
variable worst-slice gain.  The semi-synthetic \texttt{bike-synth} case shows the
clearest benefit of localization, while the real datasets exhibit smaller gains
and non-negligible error bars.

The fair conclusion is therefore deliberately modest.  The hybrid RLCP
implementation has empirical marginal coverage at or above the target on all five
datasets and uses local thresholds in regions where global calibration visibly
struggles.  The same localization can make intervals longer, especially in sparse
parts of the covariate space.  This is the coverage--length trade-off suggested
by the theory; it is not evidence that RLCP uniformly outperforms global
conformal procedures.

\begin{figure}[t]
  \centering
  \includegraphics[width=0.72\linewidth]{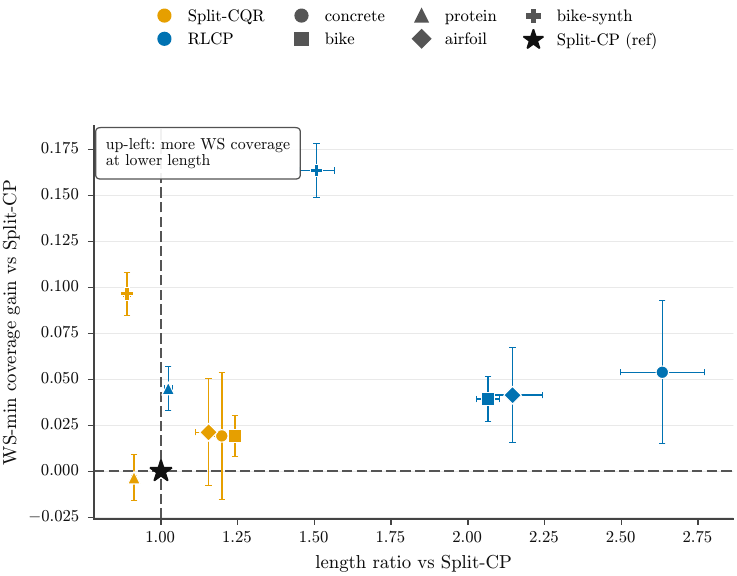}
  \caption{Worst-slice coverage gain versus mean-length ratio, paired against
  \textsf{Split-CP}.  Points above the horizontal line improve the worst density
  slice; points left of the vertical line are shorter than \textsf{Split-CP}.
  The real-data comparison shows a trade-off: RLCP improves the worst slice on
  all five datasets, usually at a length premium.}
  \label{fig:supp-e7-tradeoff}
\end{figure}

\bibliography{refs}

\end{document}